%% file: reg_tf.tex
\newcommand*{\ICML}{}

\newcommand*{\CAMREADY}{}

\ifdefined\NEURIPS
	\PassOptionsToPackage{numbers}{natbib}
	\documentclass{article}
	\ifdefined\CAMREADY
		\usepackage[final]{neurips_2020}
	\else
		\usepackage{neurips_2020}
	\fi
	\usepackage{footmisc}
	\usepackage[utf8]{inputenc} 
	\usepackage[T1]{fontenc}    
	\usepackage{hyperref}       	
	\usepackage{url}            	
	\usepackage{booktabs}       
	\usepackage{amsfonts}       
	\usepackage{nicefrac}       	
	\usepackage{microtype}      
\fi

\ifdefined\CVPR
	\documentclass[10pt,twocolumn,letterpaper]{article}
	\usepackage{footmisc}
	\usepackage{cvpr}
	\usepackage{times}
	\usepackage{epsfig}
	\usepackage{graphicx}
	\usepackage{amsmath}
	\usepackage{amssymb}
	\usepackage[square,comma,numbers]{natbib}
\fi
\ifdefined\AISTATS
	\documentclass[twoside]{article}
	\ifdefined\CAMREADY
		\usepackage[accepted]{aistats2015}
	\else
		\usepackage{aistats2015}
\fi
	\usepackage{url}
	\usepackage[round,authoryear]{natbib}
\fi
\ifdefined\ICML
	\documentclass{article}
	\usepackage{footmisc}
	\usepackage{microtype}
	\usepackage{graphicx}
	\usepackage{booktabs}
	\usepackage{hyperref}
	
	\ifdefined\CAMREADY
		\usepackage[accepted]{icml2021}
	\else
		\usepackage{icml2021}
	\fi
\fi
\ifdefined\ICLR
	\documentclass{article}
	\usepackage{iclr2019_conference,times}
	\usepackage{footmisc}
	\usepackage{hyperref}
	\usepackage{url}

	\ifdefined\CAMREADY
		\iclrfinalcopy
	\fi
\fi
\ifdefined\COLT
	\ifdefined\CAMREADY
		\documentclass{colt2017}
	\else
		\documentclass[anon,12pt]{colt2017}
	\fi
\	usepackage{times}
\fi

\usepackage{color}
\usepackage{amsfonts}
\usepackage{algorithm}
\usepackage{algorithmic}
\usepackage{graphicx}
\usepackage{nicefrac}
\usepackage{bbm}
\usepackage{wrapfig}
\usepackage{tikz}
\usepackage[titletoc,title]{appendix}
\usepackage{stmaryrd}
\usepackage{comment}
\usepackage{endnotes}
\usepackage{hyperendnotes}
\usepackage{enumerate}
\usepackage{enumitem}
\usepackage[normalem]{ulem}

\usepackage{titlesec}
\ifdefined\COLT
\else
\usepackage{amsmath}
\usepackage{amsthm}
\fi
\usepackage[small]{caption}
\usepackage[caption = false]{subfig}

\ifdefined\COLT
	\newtheorem{claim}[theorem]{Claim}
	\newtheorem{fact}[theorem]{Fact}
	\newtheorem{procedure}{Procedure}
	\newtheorem{conjecture}{Conjecture}	
	\newtheorem{hypothesis}{Hypothesis}	
	\newcommand{\qed}{\hfill\ensuremath{\blacksquare}}
\else
	\newtheorem{lemma}{Lemma}
	\newtheorem{corollary}{Corollary}
	\newtheorem{theorem}{Theorem}
	\newtheorem{proposition}{Proposition}
	\newtheorem{assumption}{Assumption}

	\theoremstyle{definition}
	\newtheorem{definition}{Definition}
\fi

\def\be{\begin{equation}}
\def\ee{\end{equation}}
\def\beas{\begin{eqnarray*}}
\def\eeas{\end{eqnarray*}}
\def\bea{\begin{eqnarray}}
\def\eea{\end{eqnarray}}

\newcommand{\x}{{\mathbf x}}

\newcommand{\w}{{\mathbf w}}

\newcommand{\aaa}{{\mathbf a}}
\newcommand{\bb}{{\mathbf b}}

\newcommand{\A}{{\mathcal A}}

\newcommand{\D}{{\mathcal D}}

\renewcommand{\S}{{\mathcal S}}

\newcommand{\W}{{\mathcal W}}
\newcommand{\X}{{\mathcal X}}
\newcommand{\Y}{{\mathcal Y}}
\newcommand{\OO}{{\mathcal O}}

\renewcommand{\L}{\mathcal{L}}

\newcommand{\R}{{\mathbb R}}
\newcommand{\N}{{\mathbb N}}

\newcommand{\abs}[1]{\left\lvert #1 \right\rvert}
\newcommand{\norm}[1]{\left\| #1 \right\|}
\newcommand{\normnoflex}[1]{\| #1 \|}

\newcommand{\inprod}[2]  {\left\langle{#1},{#2}\right\rangle}
\newcommand{\inprodnoflex}[2]{\langle{#1},{#2}\rangle}

\newcommand{\mat}[1]{\llbracket#1\rrbracket}
\newcommand{\matflex}[1]{\left\llbracket#1\right\rrbracket}

\newcommand{\tenp}{\otimes}
\newcommand{\kronp}{\odot}

\newcommand{\sign}{\mathrm{sign}}

\definecolor{xcolor-gray}{gray}{0.95}

\DeclareFontFamily{U}{mathx}{\hyphenchar\font45}
\DeclareFontShape{U}{mathx}{m}{n}{<-> mathx10}{}
\DeclareSymbolFont{mathx}{U}{mathx}{m}{n}
\DeclareMathAccent{\widebar}{0}{mathx}{"73}

\definecolor{darkspringgreen}{rgb}{0.09, 0.45, 0.27}
\usetikzlibrary{decorations.pathreplacing,calligraphy}
\usetikzlibrary{patterns}

\ifdefined\ENABLEENDNOTES
	\renewcommand{\theendnote}{\alph{endnote}} 
\else
	\renewcommand{\endnote}[1]{\null} 
\fi
\let\note\footnote

\ifdefined\CVPR
	\usepackage[breaklinks=true,bookmarks=false]{hyperref}
	\ifdefined\CAMREADY
		\cvprfinalcopy 
	\fi
	
\fi

\ifdefined\ICML
	\newcommand*{\ABBR}{}
\fi
\ifdefined\NEURIPS
	\newcommand*{\ABBR}{}
\fi
\ifdefined\ICLR
	\newcommand*{\ABBR}{}
\fi
\ifdefined\COLT
	\newcommand*{\ABBR}{}
\fi
\ifdefined\ABBR
	\newcommand{\eg}{{\it e.g.}}
	\newcommand{\ie}{{\it i.e.}}
	\newcommand{\cf}{{\it cf.}}

\fi

\begin{document}
	
	\ifdefined\NEURIPS
	\title{Paper Title}
		\author{
			Author 1 \\
			Author 1 Institution \\	
			\texttt{author1@email} \\
			\And
			Author 1 \\
			Author 1 Institution \\	
			\texttt{author1@email} \\
		}
		\maketitle
	\fi
	\ifdefined\CVPR
		\title{Paper Title}
		\author{
			Author 1 \\
			Author 1 Institution \\	
			\texttt{author1@email} \\
			\and
			Author 2 \\
			Author 2 Institution \\
			\texttt{author2@email} \\	
			\and
			Author 3 \\
			Author 3 Institution \\
			\texttt{author3@email} \\
		}
		\maketitle
	\fi
	\ifdefined\AISTATS
		\twocolumn[
		\aistatstitle{Paper Title}
		\ifdefined\CAMREADY
			\aistatsauthor{Author 1 \And Author 2 \And Author 3}
			\aistatsaddress{Author 1 Institution \And Author 2 Institution \And Author 3 Institution}
		\else
			\aistatsauthor{Anonymous Author 1 \And Anonymous Author 2 \And Anonymous Author 3}
			\aistatsaddress{Unknown Institution 1 \And Unknown Institution 2 \And Unknown Institution 3}
		\fi
		]	
	\fi
	\ifdefined\ICML
		\icmltitlerunning{Implicit Regularization in Tensor Factorization}
		\twocolumn[
		\icmltitle{Implicit Regularization in Tensor Factorization} 
		\icmlsetsymbol{equal}{*}
		\begin{icmlauthorlist}
			\icmlauthor{Noam Razin}{tau,equal} 
			\icmlauthor{Asaf Maman}{tau,equal}
			\icmlauthor{Nadav Cohen}{tau}
		\end{icmlauthorlist}
		\icmlaffiliation{tau}{Blavatnik School of Computer Science, Tel Aviv University, Israel}
		\icmlcorrespondingauthor{Noam Razin}{noam.razin@cs.tau.ac.il}
		\icmlcorrespondingauthor{Asaf Maman}{asafmaman@mail.tau.ac.il}
		\icmlkeywords{Implicit Regularization, Tensor Factorization, Deep Learning, Generalization}
		\vskip 0.3in
		]
		\printAffiliationsAndNotice{\icmlEqualContribution} 
	\fi
	\ifdefined\ICLR
		\title{Paper Title}
		\author{
			Author 1 \\
			Author 1 Institution \\
			\texttt{author1@email}
			\And
			Author 2 \\
			Author 2 Institution \\
			\texttt{author2@email}
			\And
			Author 3 \\ 
			Author 3 Institution \\
			\texttt{author3@email}
		}
		\maketitle
	\fi
	\ifdefined\COLT
		\title{Paper Title}
		\coltauthor{
			\Name{Author 1} \Email{author1@email} \\
			\addr Author 1 Institution
			\And
			\Name{Author 2} \Email{author2@email} \\
			\addr Author 2 Institution
			\And
			\Name{Author 3} \Email{author3@email} \\
			\addr Author 3 Institution}
		\maketitle
	\fi

	\input{abstract}

	\ifdefined\COLT
		\medskip
	\fi

	\input{sec_intro}

	\input{sec_tf}

	\input{sec_dynamic}

	\input{sec_rank}

	\input{sec_experiments}

	\input{sec_related}

	\input{sec_conclusion}

	\ifdefined\NEURIPS
		\begin{ack}
			\input{ack}
		\end{ack}
	\else
		\newcommand{\ack}
		{\input{ack}}
	\fi
	\ifdefined\COLT
		\acks{\input{ack}}
	\else
		\ifdefined\CAMREADY
			\ifdefined\ICLR
				\newcommand*{\subsuback}{}
			\fi
			\ifdefined\NEURIPS
			\else
				\section*{Acknowledgements}
				\input{ack}
			\fi
		\fi
	\fi

	\section*{References}
	{\small
		\ifdefined\ICML
			\bibliographystyle{icml2021}
		\else
			\bibliographystyle{plainnat}
		\fi
		\bibliography{refs}
	}

	\clearpage
	\appendix
	
	\onecolumn
	
	\ifdefined\ENABLEENDNOTES
		\theendnotes
	\fi

	\input{app_sensing}

	\input{app_experiments}

	\input{app_proofs}

\end{document}

%% file: abstract.tex
\begin{abstract}

Recent efforts to unravel the mystery of implicit regularization in deep learning have led to a theoretical focus on matrix factorization~---~matrix completion via linear neural network.
As a step further towards practical deep learning, we provide the first theoretical analysis of implicit regularization in tensor factorization~---~tensor completion via certain type of non-linear neural network.
We circumvent the notorious difficulty of tensor problems by adopting a dynamical systems perspective, and characterizing the evolution induced by gradient descent.
The characterization suggests a form of greedy low tensor rank search, which we rigorously prove under certain conditions, and empirically demonstrate under others.
Motivated by tensor rank capturing the implicit regularization of a non-linear neural network, we empirically explore it as a measure of complexity, and find that it captures the essence of datasets on which neural networks generalize.
This leads us to believe that tensor rank may pave way to explaining both implicit regularization in deep learning, and the properties of real-world data translating this implicit regularization to generalization.

\end{abstract}

%% file: sec_intro.tex
\section{Introduction} \label{sec:intro}

The ability of neural networks to generalize when having far more learnable parameters than training examples, even in the absence of any explicit regularization, is an enigma lying at the heart of deep learning theory.
Conventional wisdom is that this generalization stems from an \emph{implicit regularization}~---~a tendency of gradient-based optimization to fit training examples with predictors whose ``complexity'' is as low as possible.
The fact that ``natural'' data gives rise to generalization while other types of data (\eg~random) do not, is understood to result from the former being amenable to fitting by predictors of lower complexity.
A major challenge in formalizing this intuition is that we lack definitions for predictor complexity that are both quantitative (\ie~admit quantitative generalization bounds) and capture the essence of natural data (in the sense of it being fittable with low complexity).
Consequently, existing analyses typically focus on simplistic settings, where a notion of complexity is apparent. 
A prominent example of such a setting is \emph{matrix~completion}.

In matrix completion, we are given a randomly chosen subset of entries from an unknown matrix~$W^* \in \R^{d , d'}$, and our goal is to recover unseen entries.
This can be viewed as a prediction problem, where the set of possible inputs is $\X = \{ 1 , ...\, , d \} {\times} \{ 1 , ...\, , d' \}$, the possible labels are $\Y = \R$, and the label of $( i , j ) \in \X$ is~$[ W^* ]_{i , j}$.
Under this viewpoint, observed entries constitute the training set, and the average reconstruction error over unobserved entries is the test error, quantifying generalization.
A predictor, \ie~a function from $\X$ to~$\Y$, can then be seen as a matrix, and a natural notion of complexity is its rank.
It is known empirically (\cf~\citet{gunasekar2017implicit,arora2019implicit}) that this complexity measure is oftentimes implicitly minimized by \emph{matrix factorization}~---~\emph{linear} neural network\note{
That is, parameterization of learned predictor (matrix) as a product of matrices.
With such parameterization it is possible to explicitly constrain rank (by limiting shared dimensions of multiplied matrices), but the setting of interest is where rank is unconstrained, meaning all regularization is implicit.
} 
trained via gradient descent with small learning rate and near-zero initialization.
Mathematically characterizing the implicit regularization in matrix factorization is a highly active area of research.
Though initially conjectured to be equivalent to norm minimization (see~\citet{gunasekar2017implicit}), recent studies \cite{arora2019implicit,razin2020implicit,li2021towards} suggest that this is not the case, and instead adopt a dynamical view, ultimately establishing that (under certain conditions) the implicit regularization in matrix factorization is performing a greedy low rank search.

A central question that arises is the extent to which the study of implicit regularization in matrix factorization is relevant to more practical settings.
Recent experiments (see~\citet{razin2020implicit}) have shown that the tendency towards low rank extends from matrices (two-dimensional arrays) to \emph{tensors} (multi-dimensional arrays).
Namely, in the task of $N$-dimensional \emph{tensor completion}, which (analogously to matrix completion) can be viewed as a prediction problem over $N$~input variables, training a \emph{tensor factorization}\note{
The term ``tensor factorization'' refers throughout to the classic \emph{CP factorization}; other (more advanced) factorizations will be named differently (see \citet{kolda2009tensor,hackbusch2012tensor} for an introduction to various tensor factorizations).
\label{note:cp_decomp}
}
via gradient descent with small learning rate and near-zero initialization tends to produce tensors (predictors) with low \emph{tensor rank}.
Analogously to how matrix factorization may be viewed as a linear neural network, tensor factorization can be seen as a certain type of \emph{non-linear} neural network (two layer network with multiplicative non-linearity, \cf~\citet{cohen2016expressive}), and so it represents a setting much closer to practical deep learning.

In this paper we provide the first theoretical analysis of implicit regularization in tensor factorization.
We circumvent the notorious difficulty of tensor problems (see \citet{hillar2013most}) by adopting a dynamical systems perspective.
Characterizing the evolution that gradient descent with small learning rate and near-zero initialization induces on the components of a factorization, we show that their norms are subject to a momentum-like effect, in the sense that they move slower when small and faster when large.
This implies a form of greedy low tensor rank search, generalizing phenomena known for the case of matrices.
We employ the finding to prove that, with the classic Huber loss from robust statistics~\cite{huber1964robust}, arbitrarily small initialization leads tensor factorization to follow a trajectory of rank one tensors for an arbitrary amount of time or distance.
Experiments validate our analysis, demonstrating implicit regularization towards low tensor rank in a wide array of configurations.

Motivated by the fact that tensor rank captures the implicit regularization of a non-linear neural network, we empirically explore its potential to serve as a measure of complexity for multivariable predictors.
We find that it is possible to fit standard image recognition datasets~---~MNIST~\cite{lecun1998mnist} and Fashion-MNIST~\cite{xiao2017fashion}~---~with predictors of extremely low tensor rank, far beneath what is required for fitting random data.
This leads us to believe that tensor rank (or more advanced notions such as hierarchical tensor ranks) may pave way to explaining both implicit regularization of contemporary deep neural networks, and the properties of real-world data translating this implicit regularization to generalization.

\medskip

The remainder of the paper is organized as follows.
Section~\ref{sec:tf} presents the tensor factorization model, as well as its interpretation as a neural network.
Section~\ref{sec:dynamic} characterizes its dynamics, followed by Section~\ref{sec:rank} which employs the characterization to establish (under certain conditions) implicit tensor rank minimization.
Experiments, demonstrating both the dynamics of learning and the ability of tensor rank to capture the essence of standard datasets, are given in Section~\ref{sec:experiments}.
In Section~\ref{sec:related} we review related work.
Finally, Section~\ref{sec:conclusion} concludes.
Extension of our results to tensor sensing (more general setting than tensor completion) is discussed in Appendix~\ref{app:sensing}.

%% file: sec_tf.tex
\section{Tensor Factorization} \label{sec:tf}

Consider the task of completing an $N$-dimensional tensor ($N \geq 3$) with axis lengths $d_1, \ldots, d_N \in \N$, or, in standard tensor analysis terminology, an \emph{order}~$N$ tensor with \emph{modes} of \emph{dimensions} $d_1, \ldots, d_N$.
Given a set of observations $\{ y_{i_1, \ldots, i_N} \in \R \}_{(i_1, \ldots, i_N) \in \Omega }$, where $\Omega$ is a subset of all possible index tuples, a standard (undetermined) loss function for the task is:
\bea
\L : \R^{d_1, \ldots, d_N} \to \R_{\geq 0} 
\hspace{25mm}
\label{eq:tc_loss} \\
\L ( \W ) \,{=}\, \frac{1}{\abs{\Omega}} \hspace{-0.5mm} \sum\nolimits_{(i_1, \ldots, i_N) \in \Omega} \hspace{-0.5mm} \ell \left ( [ \W ]_{i_1, \ldots, i_N} \,{-}\, y_{i_1, \ldots, i_N} \right )
\text{\,,} \hspace{-1.5mm}
\nonumber
\eea
where $\ell : \R \to \R_{\geq 0}$ is differentiable and locally smooth.
A typical choice for~$\ell ( \cdot )$ is $\ell ( z ) = \frac{1}{2} z^2$, corresponding to $\ell_2$~loss.
Other options are also common, for example that given in Equation~\eqref{eq:huber_loss}, which corresponds to the Huber loss from robust statistics~\cite{huber1964robust}~---~a differentiable surrogate for $\ell_1$~loss.

Performing tensor completion with an $R$-component tensor factorization amounts to optimizing the following (non-convex) objective:
\be
\phi \left ( \{ \w_r^n \}_{r = 1}^R\hspace{0mm}_{n = 1}^N \right ) := \L \left ( \W_e \right )
\text{\,,}
\label{eq:cp_objective}
\ee
defined over \emph{weight vectors}  $\{ \w_r^n \in \R^{d_n} \}_{r = 1}^R\hspace{0mm}_{n = 1}^N$, where:
\be
\W_e := \sum\nolimits_{r = 1}^R \w_r^1 \tenp \cdots \tenp \w_r^N
\label{eq:end_tensor}
\ee
is referred to as the \emph{end tensor} of the factorization, with $\tenp$ representing outer product.\note{
For any $\{ \w^n \,{\in}\, \R^{d_n} \}_{n = 1}^N$, the outer product $\w^1 {\tenp} \cdots {\tenp} \w^N$, denoted also $\tenp_{n = 1}^N \w^n$, is the tensor in $\R^{d_1, \ldots, d_N}$ defined by $[ \tenp_{n = 1}^N \w^n ]_{i_1, \ldots, i_N} \,{=}\,  \prod_{n = 1}^N [ \w^n ]_{i_n}$.
}
The minimal number of components $R$ required in order for $\W_e$ to be able to express a given tensor $\W \in \R^{d_1, \ldots, d_N}$, is defined to be the \emph{tensor rank} of~$\W$.
One may explicitly restrict the tensor rank of solutions produced by the tensor factorization via limiting~$R$.
However, since our interest lies in the implicit regularization induced by gradient descent, \ie~in the type of end tensors (Equation~\eqref{eq:end_tensor}) it will find when applied to the objective $\phi (\cdot)$ (Equation~\eqref{eq:cp_objective}) with no explicit constraints, we treat the case where $R$ can be arbitrarily large.

In line with analyses of matrix factorization (\eg~\citet{gunasekar2017implicit,arora2018optimization,arora2019implicit,eftekhari2020implicit,li2021towards}), we model small learning rate for gradient descent through the infinitesimal limit, \ie~through \emph{gradient flow}:
\bea
&&\frac{d}{dt} \w_r^{n} (t) := - \frac{\partial}{\partial \w_r^{n}} \phi  \left ( \{ \w_{r'}^{n'} (t) \}_{r' = 1}^R\hspace{0mm}_{n' = 1}^N \right ) 
\label{eq:cp_gf} \\[1mm]
&&\hspace{15mm} , t \geq 0 ~ , ~ r = 1, \ldots, R ~ , ~ n = 1, \ldots, N 
\text{\,,}
\nonumber
\eea
where $\{ \w_r^n (t) \}_{r = 1}^R\hspace{0mm}_{n = 1}^N$ denote the weight vectors at time~$t$ of optimization.

Our aim is to theoretically investigate the prospect of implicit regularization towards low tensor rank, \ie~of gradient flow with near-zero initialization learning a solution that can be represented with a small number of components.

\subsection{Interpretation as Neural Network} \label{sec:tf:nn}

\begin{figure}
	\begin{center}
		\hspace{-2.2mm}
		\includegraphics[width=0.49\textwidth]{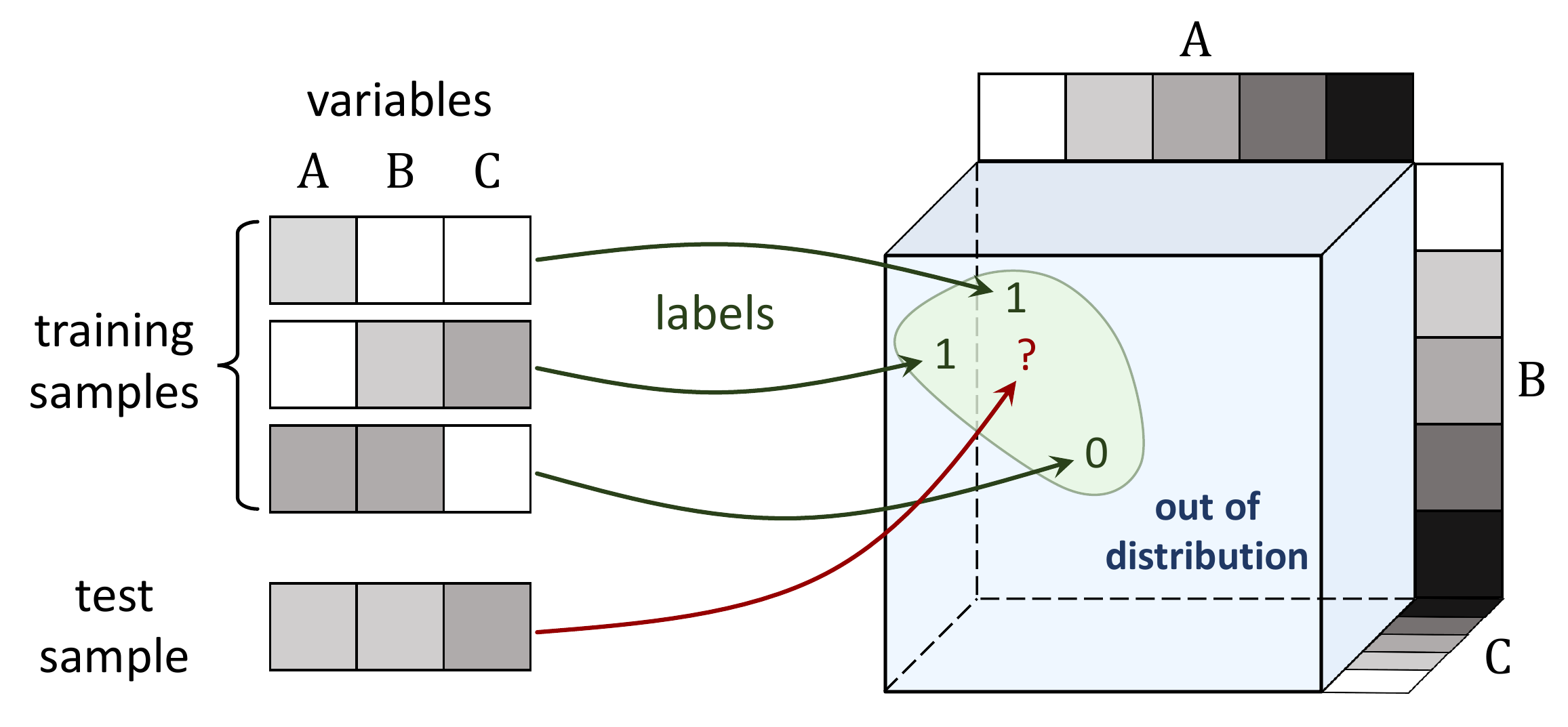}
	\end{center}
	\vspace{-3mm}
	\caption{
		Prediction tasks over discrete variables can be viewed as tensor completion problems.		
		Consider the task of learning a predictor from domain $\X = \{ 1 , \ldots , d_1 \} \times \cdots \times \{ 1 , \ldots , d_N \}$ to range $\Y = \R$ (figure assumes $N = 3$ and $d_1 = \cdots = d_N = 5$ for the sake of illustration).
		Each input sample is associated with a location in an order~$N$ tensor with mode (axis) dimensions $d_1, \ldots, d_N$, where the value of a variable (depicted as a shade of gray) determines the index of the corresponding mode (marked by ``A", ``B" or ``C").
		The associated location stores the label of the sample.
		Under this viewpoint, training samples are observed entries, drawn according to an unknown distribution from a ground truth tensor.
		Learning a predictor amounts to completing the unobserved entries, with test error measured by (weighted) average reconstruction error.
		In many standard prediction tasks (\eg~image recognition), only a small subset of the input domain has non-negligible probability.
		From the tensor completion perspective this means that observed entries reside in a restricted part of the tensor, and reconstruction error is weighted accordingly (entries outside the support of the distribution are neglected).
	}
	\label{fig:pred_prob_as_tc}
\end{figure}

\begin{figure*}
	\begin{center}
		\includegraphics[width=0.83\textwidth]{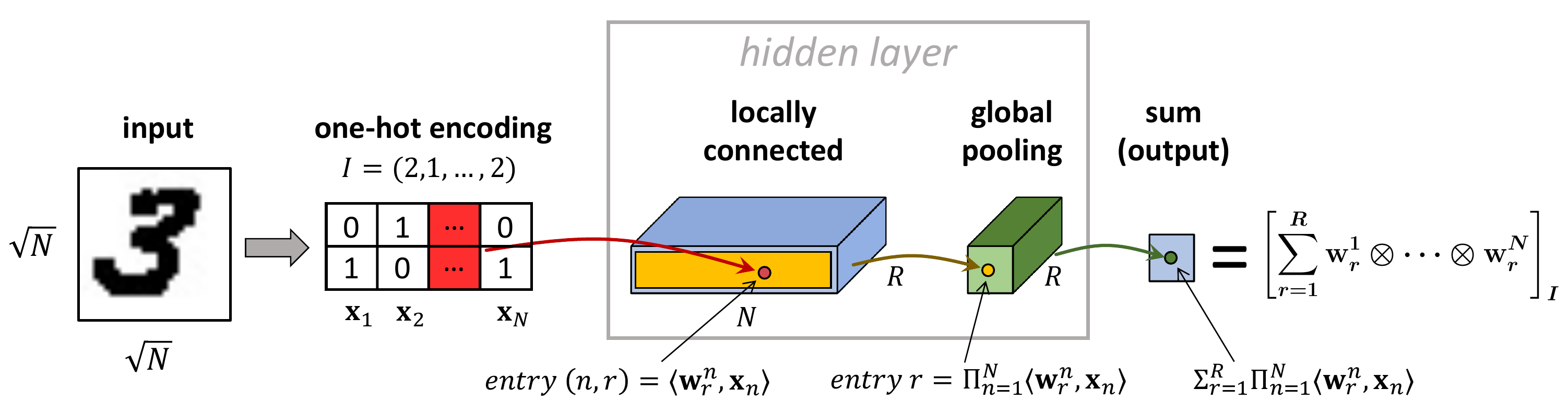}
	\end{center}
	\vspace{-3mm}
	\caption{
		Tensor factorizations correspond to a class of non-linear neural networks.
		Figure~\ref{fig:pred_prob_as_tc} illustrates how a prediction task can be viewed as a tensor completion problem.
		The current figure extends this correspondence, depicting an equivalence between solving tensor completion via tensor factorization, and learning a predictor using the non-linear neural network portrayed above.
		The input to the network is a tuple $I = (i_1, \ldots, i_N) \in \{ 1, \ldots, d_1 \} \times \cdots \times \{ 1, \ldots, d_N\}$, encoded via one-hot vectors $(\x_1, \ldots, \x_N) \in \R^{d_1} \times \cdots \times \R^{d_N}$.
		For example, in the diagram, $I$ stands for a binary image with $N$ pixels (in which case $d_1 = \cdots = d_N = 2$).
		The one-hot representations are passed through a hidden layer consisting of: \emph{(i)}~locally connected linear operator with $R$~channels, the $r$'th one computing $\inprodnoflex{\w_r^{\scriptscriptstyle 1}}{\x_1}, \ldots, \inprodnoflex{\w_r^{\scriptscriptstyle N}}{\x_N}$ with filters (learnable weights) $\{ \w_r^{\scriptscriptstyle n }\}_{n = 1}^N$; and \emph{(ii)} channel-wise global product pooling (multiplicative non-linearity).
		The resulting activations are then reduced through summation to a scalar --- the output of the network.
		All in all, given input tuple $I = ( i_1 , \ldots , i_N )$, the network outputs the $I$'th entry of $\sum_{r = 1}^R \w_r^{\scriptscriptstyle 1} \tenp \cdots \tenp \w_r^{\scriptscriptstyle N}$.
		Notice that the number of components $R$ and the weight vectors $\{ \w_r^{\scriptscriptstyle n} \}_{r, n}$ in the factorization correspond to the width and the learnable filters of the network, respectively.
	}
	\label{fig:tf_as_convac}
\end{figure*}

Tensor completion can be viewed as a prediction problem, where each mode corresponds to a discrete input variable.
For an unknown tensor $\W^* \in \R^{d_1, \ldots, d_N}$, inputs are index tuples of the form $(i_1, \ldots, i_N)$, and the label associated with such an input is $[ \W^*]_{i_1, \ldots, i_N}$.
Under this perspective, the training set consists of the observed entries, and the average reconstruction error over unseen entries measures test error.
The standard case, in which observations are drawn uniformly across the tensor and reconstruction error weighs all entries equally, corresponds to a data distribution that is uniform, but other distributions are also viable.

Consider for example the task of predicting a continuous label for a $100$-by-$100$ binary image.
This can be formulated as an order $10000$ tensor completion problem, where all modes are of dimension $2$.
Each input image corresponds to a location (entry) in the tensor~$\W^*$, holding its continuous label.
As image pixels are (typically) not distributed independently and uniformly, locations in the tensor are not drawn uniformly when observations are generated, and are not weighted equally when reconstruction error is computed.
See Figure~\ref{fig:pred_prob_as_tc} for further illustration of how a general prediction task (with discrete inputs and scalar output) can be formulated as a tensor completion problem.

Under the above formulation, tensor factorization can be viewed as a two layer neural network with multiplicative non-linearity.
Given an input, \ie~a location in the tensor, the network produces an output equal to the value that the factorization holds at the given location.
Figure~\ref{fig:tf_as_convac} illustrates this equivalence between solving tensor completion with a tensor factorization and solving a prediction problem with a non-linear neural network.
A major drawback of matrix factorization as a theoretical surrogate for modern deep learning is that it misses the critical aspect of non-linearity.
Tensor factorization goes beyond the realm of linear predictors~---~a significant step towards practical neural networks.

%% file: sec_dynamic.tex
\section{Dynamical Characterization} \label{sec:dynamic}

In this section we derive a dynamical characterization for the norms of individual components in the tensor factorization.
The characterization implies that with small learning rate and near-zero initialization, components tend to be learned incrementally, giving rise to a bias towards low tensor rank solutions.
This finding is used in Section~\ref{sec:rank} to prove (under certain conditions) implicit tensor rank minimization, and is demonstrated empirically in Section~\ref{sec:experiments}.\note{
We note that all results in this section apply even if the tensor completion loss~$\L (\cdot)$ (Equation~\eqref{eq:tc_loss}) is replaced by any differentiable and locally smooth function.
The proofs in Appendix~\ref{app:proofs} already account for this more general setting.
\label{note:dyn_holds_for_diff_local_smooth_loss}
}

Hereafter, unless specified otherwise, when referring to a norm we mean the standard Frobenius (Euclidean) norm, denoted by~$\norm{\cdot}$.

The following lemma establishes an invariant of the dynamics, showing that the differences between squared norms of vectors in the same component are constant through time.

\begin{lemma}
	\label{lem:balancedness_conservation_body}
	For all $r \in \{ 1, \ldots, R \}$ and $n, \bar{n} \in \{ 1 , \ldots, N \}$:
	\[
	\norm{ \w_r^{n} (t) }^2 - \norm{ \w_r^{\bar{n}} (t) }^2 = \norm{ \w_r^{n} (0) }^2 - \norm{ \w_r^{\bar{n}} (0) }^2 , ~t \geq 0
	\text{\,.}
	\]
\end{lemma}

\begin{proof}[Proof sketch (for proof see Lemma~\ref{lem:balancedness_conservation} in Subappendix~\ref{app:proofs:useful_lemmas:tf})]
	The claim readily follows by showing that under gradient flow $\frac{d}{dt}  \normnoflex{ \w_r^{n} (t) }^2 = \frac{d}{dt}  \normnoflex{ \w_r^{\bar{n}} (t) }^2$ for all $t \geq 0$.
\end{proof}

Lemma~\ref{lem:balancedness_conservation_body} naturally leads to the definition below.

\begin{definition}
	\label{def:unbalancedness_magnitude}
	The \emph{unbalancedness magnitude} of the weight vectors $\{ \w_r^n \in \R^{d_n} \}_{r = 1}^R\hspace{0mm}_{n = 1}^N$ is defined to be:
	\[
	\max\nolimits_{r \in \{ 1, \ldots, R\}, ~n, \bar{n} \in \{ 1, \ldots, N\}} \abs{ \norm{ \w_r^n }^2 - \norm{ \w_r^{\bar{n}} }^2 }
	\text{\,.}
	\]
\end{definition}

By Lemma~\ref{lem:balancedness_conservation_body}, the unbalancedness magnitude is constant during optimization, and thus, is determined at initialization.
When weight vectors are initialized near the origin~---~regime of interest~---~the unbalancedness magnitude is small, approaching zero as initialization scale decreases.

Theorem~\ref{thm:dyn_fac_comp_norm_unbal} below provides a dynamical characterization for norms of individual components in the tensor factorization.

\begin{theorem}
	\label{thm:dyn_fac_comp_norm_unbal}
	Assume unbalancedness magnitude $\epsilon \geq 0$ at initialization, and denote by~$\W_e (t)$ the end tensor (Equation~\eqref{eq:end_tensor}) at time $t \geq 0$ of optimization.
	Then, for any $r \in \{ 1, \ldots, R \}$ and time $t \geq 0$ at which $\normnoflex{ \tenp_{n = 1}^N \w_r^{n} (t) } > 0$:\footnote{
		When $\| \tenp_{n = 1}^N \w_r^{n} (t) \|$ is zero it may not be differentiable.
	}
	\begin{itemize}[leftmargin=3.5mm]
		\item If $\gamma_r (t) := \inprodnoflex{ - \nabla \L ( \W_e (t) ) }{ \tenp_{n = 1}^N \widehat{\w}_r^{n} (t) } \geq 0$, then:
		\be
		\begin{split}
			& \hspace{-3.2mm} \frac{d}{dt} \normnoflex{ \tenp_{n = 1}^N \w_r^{n} (t) } \leq N \gamma_r (t) ( \normnoflex{ \tenp_{n = 1}^N \w_r^{n} (t) }^{\frac{2}{N}} + \epsilon )^{N - 1} \\
			& \hspace{-3.2mm} \frac{d}{dt} \normnoflex{ \tenp_{n = 1}^N \w_r^{n} (t) } \geq N \gamma_r (t) \cdot \frac{ \normnoflex{ \tenp_{n = 1}^N \w_r^{n} (t) }^2 }{  \normnoflex{ \tenp_{n = 1}^N \w_r^{n} (t) }^{\frac{2}{N}} + \epsilon }
			\text{\,,}
		\end{split}
		\label{eq:dyn_fac_comp_norm_unbal_pos}
		\ee
		
		\item otherwise, if $\gamma_r (t)  < 0$, then:
		\be
		\begin{split}
			& \hspace{-3.2mm} \frac{d}{dt} \normnoflex{ \tenp_{n = 1}^N \w_r^{n} (t) } \geq N \gamma_r (t) ( \normnoflex{ \tenp_{n = 1}^N \w_r^{n} (t) }^{\frac{2}{N}} + \epsilon )^{N - 1} \\
			& \hspace{-3.2mm} \frac{d}{dt} \norm{ \tenp_{n = 1}^N \w_r^{n} (t) } \leq N \gamma_r (t) \cdot \frac{ \norm{ \tenp_{n = 1}^N \w_r^{n} (t) }^2 }{  \norm{ \tenp_{n = 1}^N \w_r^{n} (t) }^{\frac{2}{N}} + \epsilon }
			\text{\,,}
		\end{split}
		\label{eq:dyn_fac_comp_norm_unbal_neg}
		\ee
	\end{itemize}
	where $\widehat{\w}_r^{n} (t) := \w_r^{n} (t) / \normnoflex{ \w_r^{n} (t) }$ for $n = 1, \ldots, N$.
\end{theorem}

\begin{proof}[Proof sketch (for proof see Subappendix~\ref{app:proofs:dyn_fac_comp_norm_unbal})]
	Differentia- ting a component's norm with respect to time, we obtain $\frac{d}{dt} \normnoflex{ \tenp_{n = 1}^N \w_r^{n} (t) } = \gamma_r (t) \cdot \sum_{n = 1}^N \prod_{n' \neq n} \normnoflex{ \w_r^{n'} (t) }^2$.
	The desired bounds then follow from using conservation of unbalancedness magnitude (as implied by Lemma~\ref{lem:balancedness_conservation_body}), and showing that $\normnoflex{ \w_r^{n'} (t) }^2 \leq \normnoflex{ \tenp_{n = 1}^N \w_r^n (t) }^{ 2 / N } + \epsilon$ for all $t \geq 0$ and $n' \in  \{ 1, \ldots, N \}$.
\end{proof}

Theorem~\ref{thm:dyn_fac_comp_norm_unbal} shows that when unbalancedness magnitude at initialization (denoted~$\epsilon$) is small, the evolution rates of component norms are roughly proportional to their size exponentiated by $2 - 2 / N$, where $N$ is the order of the tensor factorization.
Consequently, component norms are subject to a momentum-like effect, by which they move slower when small and faster when large.
This suggests that when initialized near zero, components tend to remain close to the origin, and then, upon reaching a critical threshold, quickly grow until convergence, creating an incremental learning effect that yields implicit regularization towards low tensor rank.
This phenomenon is used in Section~\ref{sec:rank} to formally prove (under certain conditions) implicit tensor rank minimization, and is demonstrated empirically in Section~\ref{sec:experiments}.

When the unbalancedness magnitude at initialization is exactly zero, our dynamical characterization takes on a particularly lucid form.

\begin{corollary}
	\label{cor:dyn_fac_comp_norm_balanced}
	Assume unbalancedness magnitude zero at initialization.
	Then, with notations of Theorem~\ref{thm:dyn_fac_comp_norm_unbal}, for any $r \in \{ 1, \ldots, R \}$, the norm of the $r$'th component evolves by:
	\vspace{-3.5mm}
	\be
	\frac{d}{dt} \norm{ \tenp_{n = 1}^N \w_r^{n} (t) } = N \gamma_r (t) \cdot \norm{ \tenp_{n = 1}^N \w_r^{n} (t) }^{2 - \frac{2}{N}} 
	\text{\,,}
	\label{eq:dyn_fac_comp_norm}
	\ee
	where by convention $\widehat{\w}_r^{n} (t) = 0$ if $\w_r^{n} (t) = 0$.
\end{corollary}

\begin{proof}[Proof sketch (for proof see Subappendix~\ref{app:proofs:dyn_fac_comp_norm_balanced})]
	If the time~$t$ is such that $\normnoflex{ \tenp_{n = 1}^N \w_r^{n} (t) } > 0$, Equation~\eqref{eq:dyn_fac_comp_norm} readily follows from applying Theorem~\ref{thm:dyn_fac_comp_norm_unbal} with $\epsilon = 0$.
	For the case where $\normnoflex{ \tenp_{n = 1}^N \w_r^{n} (t) } = 0$, we show that the component $\tenp_{n = 1}^N \w_r^{n} (t)$ must be identically zero throughout, hence both sides of Equation~\eqref{eq:dyn_fac_comp_norm} are equal to zero.
\end{proof}

It is worthwhile highlighting the relation to matrix factorization.
There, an implicit bias towards low rank emerges from incremental learning dynamics similar to above, with singular values standing in place of component norms.
In fact, the dynamical characterization given in Corollary~\ref{cor:dyn_fac_comp_norm_balanced} is structurally identical to the one provided by Theorem~3 in \citet{arora2019implicit} for singular values of a matrix factorization.
We thus obtained a generalization from matrices to tensors, notwithstanding the notorious difficulty often associated with the latter (\cf~\citet{hillar2013most}),

%% file: sec_rank.tex
\section{Implicit Tensor Rank Minimization} \label{sec:rank}

In this section we employ the dynamical characterization derived in Section~\ref{sec:dynamic} to theoretically establish implicit regularization towards low tensor rank.
Specifically, we prove that under certain technical conditions, arbitrarily small initialization leads tensor factorization to follow a trajectory of rank one tensors for an arbitrary amount of time or distance.
As a corollary, we obtain that if the tensor completion problem admits a rank one solution, and all rank one trajectories uniformly converge to it, tensor factorization with infinitesimal initialization will converge to it as well.
Our analysis generalizes to tensor factorization recent results developed in \citet{li2021towards} for matrix factorization.
As typical in transitioning from matrices to tensors, this generalization entails significant challenges necessitating use of fundamentally different techniques.

For technical reasons, our focus in this section lies on the Huber loss from robust statistics~\cite{huber1964robust}, given by:
\be
\hspace{-0.125mm}
\ell_h : \R \,{\to}\, \R_{\geq 0}
~ , ~
\ell_h ( z ) \,{:=} \begin{cases}
	\frac{1}{2} z^2 & \hspace{-2.5mm} , \abs{z} < \delta_h \\
	\delta_h (\abs{z} - \frac{1}{2} \delta_h) & \hspace{-2.5mm} , \text{otherwise}
\end{cases}
\hspace{-0.5mm}\text{,}\hspace{-2mm}
\label{eq:huber_loss}
\ee
where $\delta_h > 0$, referred to as the transition point of the loss, is predetermined.
Huber loss is often used as a differentiable surrogate for $\ell_1$~loss, in which case~$\delta_h$ is chosen to be small.
We will assume it is smaller than observed tensor entries:\note{
Note that this entails assumption of non-zero observations.
}
\begin{assumption}
\label{assump:delta_h}
$\delta_h < | y_{i_1, \ldots, i_N} | ~ , \forall (i_1, \ldots, i_N) \in \Omega$.
\end{assumption}

We will consider an initialization $\{ \aaa_r^n \in \R^{d_n} \}_{r = 1}^R\hspace{0mm}_{n = 1}^N$ for the weight vectors of the tensor factorization, and will scale this initialization towards zero.
In line with infinitesimal initializations being captured by unbalancedness magnitude zero (\cf~Section~\ref{sec:dynamic}), we assume that this is the case:
\begin{assumption}
\label{assump:a_balance}
The initialization $\{ \aaa_r^n \}_{r = 1}^R\hspace{0mm}_{n = 1}^N$ has unbalancedness magnitude zero.
\end{assumption}
We further assume that within $\{ \aaa_r^n \}_{r , n}$ there exists a leading component (subset $\{ \aaa_{\bar{r}}^n \}_n$), in the sense that it is larger than others, while having positive projection on the attracting force at the origin, \ie~on minus the gradient of the loss~$\L ( \cdot )$ (Equation~\eqref{eq:tc_loss}) at zero:
\begin{assumption}
\label{assump:a_lead_comp}
There exists $\bar{r} \in \{ 1, \ldots, R \}$ such~that:
\be
\hspace{-6mm}
\begin{split}
	& \inprod{ - \nabla \L ( 0 ) }{ \tenp_{n = 1}^N \widehat{\aaa}_{\bar{r}}^{n} } > 0 \text{\,,} \\
	& \normnoflex{ \aaa_{\bar{r}}^n } > \normnoflex{ \aaa_{r}^n } {\cdot} \hspace{-1mm} \left ( \tfrac{ \norm{ \nabla \L ( 0 ) } }{  \inprod{ - \nabla \L ( 0 ) }{  \tenp_{n = 1}^N \widehat{\aaa}_{ \bar{r} }^{n} }} \right )^{ 1 / (N - 2) } , \forall r \neq \bar{r} \text{\,,}
\end{split}
\label{eq:assump_components_sep_at_init}
\ee
where $\widehat{\aaa}_{\bar{r}}^{n} := \aaa_{\bar{r}}^{n} / \normnoflex{ \aaa_{\bar{r}}^{n} }$ for $n = 1, \ldots, N$.
\end{assumption}

Let $\alpha > 0$, and suppose we run gradient flow on the tensor factorization (see Section~\ref{sec:tf}) starting from the initialization $\{ \aaa_r^n \}_{r , n}$ scaled by~$\alpha$.
That is, we set:
\[
\w_r^n (0) = \alpha \cdot \aaa_r^n \quad , ~r = 1, \ldots, R ~,n = 1, \ldots, N
\text{\,,}
\]
and let $\{ \w_r^n ( t ) \}_{r , n}$ evolve per Equation~\eqref{eq:cp_gf}.
Denote by $\W_e (t)$, $t \geq 0$, the trajectory induced on the end tensor (Equation~\eqref{eq:end_tensor}).
We will study the evolution of this trajectory through time.
A hurdle that immediately arises is that, by the dynamical characterization of Section~\ref{sec:dynamic}, when the initialization scale $\alpha$ tends to zero (regime of interest), the time it takes $\W_e (t)$ to escape the origin grows to infinity.\note{
	To see this, divide both sides of Equation~\eqref{eq:dyn_fac_comp_norm} from Corollary~\ref{cor:dyn_fac_comp_norm_balanced} by $\normnoflex{ \tenp_{n = 1}^N \w_r^{n} (t) }^{2 - 2 / N}$, and integrate with respect to~$t$.
	It follows that the norm of a component at any fixed time tends to zero as initialization scale $\alpha$ decreases.
	This implies that for any $D > 0$, when taking $\alpha \to 0$, the time required for a component to reach norm~$D$ grows to infinity.
}
We overcome this hurdle by considering a \emph{reference sphere}~---~a sphere around the origin with sufficiently small radius:
\be
\S := \{ \W \in \R^{d_1, \ldots, d_N} : \| \W \| = \rho \}
\text{\,,}
\label{eq:ref_sphere}
\ee
where $\rho \in ( 0 , \min_{(i_1, \ldots, i_N) \in \Omega} \abs{ y_{i_1, \ldots, i_N} } - \delta_h )$ can be chosen arbitrarily.
With the reference sphere $\S$ at hand, we define a time-shifted version of the trajectory $\W_e (t)$, aligning $t = 0$ with the moment at which $\S$ is reached:
\be
\overline{\W}_e ( t ) := \W_e \big( t + \inf \{ t' \geq 0 : \W_e ( t' ) \in \S \} \big)
\label{eq:end_tensor_time_shift}
\text{\,,}
\ee
where by definition $\inf \{ t' \geq 0 : \W_e ( t' ) \in \S \} = 0$ if $\W_e (t)$ does not reach~$\S$.
Unlike the original trajectory~$\W_e (t)$, the shifted one $\overline{\W}_e ( t )$ disregards the process of escaping the origin, and thus admits a concrete meaning to the time elapsing from optimization commencement.

We will establish proximity of $\overline{\W}_e ( t )$ to trajectories of rank one tensors.
We say that $\W_1 ( t ) \in \R^{d_1, \ldots, d_N}$, $t \geq 0$, is a \emph{rank one trajectory}, if it coincides with some trajectory of an end tensor in a one-component factorization, \ie~if there exists an initialization for gradient flow over a tensor factorization with $R = 1$ components, leading the induced end tensor to evolve by~$\W_1 ( t )$.
If the latter initialization has unbalancedness magnitude zero (\cf~Definition~\ref{def:unbalancedness_magnitude}), we further say that $\W_1 ( t )$ is a \emph{balanced rank one trajectory}.\note{
	Note that the definitions of rank one trajectory and balanced rank one trajectory allow for $\W_1 ( t )$ to have rank zero (\ie~to be equal to zero) at some or all times~$t \geq 0$.
}

We are now in a position to state our main result, by which arbitrarily small initialization leads tensor factorization to follow a (balanced) rank one trajectory for an arbitrary amount of time or distance.
\begin{theorem}
\label{thm:approx_rank_1}
Under Assumptions \ref{assump:delta_h}, \ref{assump:a_balance} and~\ref{assump:a_lead_comp}, for any distance from origin $D > 0$, time duration $T > 0$, and degree of approximation $\epsilon \in ( 0 , 1 )$, if initialization scale $\alpha$ is sufficiently small,\note{
Hiding problem-dependent constants, an initialization scale of $\epsilon D^{-1}  \exp ( - \OO (D^{2} T) )$ suffices.
Exact constants are specified at the beginning of the proof in Subappendix~\ref{app:proofs:approx_rank_1}.
}
then:
\emph{(i)}~$\W_e ( t )$ reaches the reference sphere~$\S$;
and
\emph{(ii)}~there exists a balanced rank one trajectory $\W_1 ( t )$ emanating from~$\S$, such that $\| \overline{\W}_e ( t ) - \W_1 ( t ) \| \leq \epsilon$ at least until $t \geq T$ or $\| \overline{\W}_e ( t ) \| \geq D$.
\end{theorem}
\begin{proof}[Proof sketch (for proof see Subappendix~\ref{app:proofs:approx_rank_1})]
Using the dynamical characterization from Section~\ref{sec:dynamic} (Lemma~\ref{lem:balancedness_conservation_body} and Corollary~\ref{cor:dyn_fac_comp_norm_balanced}), and the fact that $\nabla \L (\cdot)$ is locally constant around the origin, we establish that 
\emph{(i)} $\W_e (t)$ reaches the reference sphere $\S$; 
and 
\emph{(ii)} at that time, the norm of the $\bar{r}$'th component is of constant scale (independent of $\alpha$), while the norms of all other components are $\OO (\alpha^N)$.
Thus, taking $\alpha$ towards zero leads $\W_e (t)$ to arrive at~$\S$ while being arbitrarily close to the initialization of a balanced rank one trajectory~---~$\W_1 (t)$.
Since the objective is locally smooth, this ensures $\overline{\W}_e ( t )$ is within distance~$\epsilon$ from $\W_1 (t)$ for an arbitrary amount of time or distance.
That is, if $\alpha$ is sufficiently small, $\| \overline{\W}_e ( t ) - \W_1 ( t ) \| \leq \epsilon$ at least until $t \geq T$ or $\| \overline{\W}_e ( t ) \| \geq D$.
\end{proof}

As an immediate corollary of Theorem~\ref{thm:approx_rank_1}, we obtain that if all balanced rank one trajectories uniformly converge to a global minimum, tensor factorization with infinitesimal initialization will do so too.
In particular, its implicit regularization will direct it towards a solution with tensor~rank~one.
\begin{corollary}
\label{corollary:converge_rank_1}
Assume the conditions of Theorem~\ref{thm:approx_rank_1} (Assumptions \ref{assump:delta_h}, \ref{assump:a_balance} and~\ref{assump:a_lead_comp}), and in addition, that all balanced rank one trajectories emanating from~$\S$ converge to a tensor $\W^* \in \R^{d_1, \ldots, d_N}$ uniformly, in the sense that they are all confined to some bounded domain, and for any $\epsilon > 0$, there exists a time~$T$ after which they are all within distance $\epsilon$ from~$\W^*$.
Then, for any $\epsilon > 0$, if initialization scale $\alpha$ is sufficiently small, there exists a time~$T$ for which $\| \W_e ( T ) - \W^* \| \leq \epsilon$.
\end{corollary}
\begin{proof}[Proof sketch (for proof see Subappendix~\ref{app:proofs:converge_rank_1})]
Let $T' > 0$ be a time at which all balanced rank one trajectories that emanated from $\S$ are within distance~$\epsilon / 2$ from $\W^*$.
By Theorem~\ref{thm:approx_rank_1}, if $\alpha$ is sufficiently small, $\overline{\W}_e ( t )$~is guaranteed to be within distance~$\epsilon / 2$ from a balanced rank one trajectory that emanated from~$\S$, at least until time~$T'$.
Recalling that $\overline{\W}_e ( t )$ is a time-shifted version of~$\W_e ( t )$, the desired result follows from the triangle inequality.
\end{proof}

%% file: sec_experiments.tex
\section{Experiments} \label{sec:experiments}

In this section we present our experiments.
Subsection~\ref{sec:experiments:dyn} corroborates our theoretical analyses (Sections~\ref{sec:dynamic} and~\ref{sec:rank}), evaluating tensor factorization (Section~\ref{sec:tf}) on synthetic low (tensor) rank tensor completion problems.
Subsection~\ref{sec:experiments:tensor_rank_complexity} explores tensor rank as a measure of complexity, examining its ability to capture the essence of standard datasets.
For brevity, we defer a description of implementation details, as well as some experiments, to Appendix~\ref{app:experiments}.

\subsection{Dynamics of Learning}
\label{sec:experiments:dyn}

\begin{figure*}
	\vspace{-4mm}
	\begin{center}
		\subfloat{
			\includegraphics[width=0.23\textwidth]{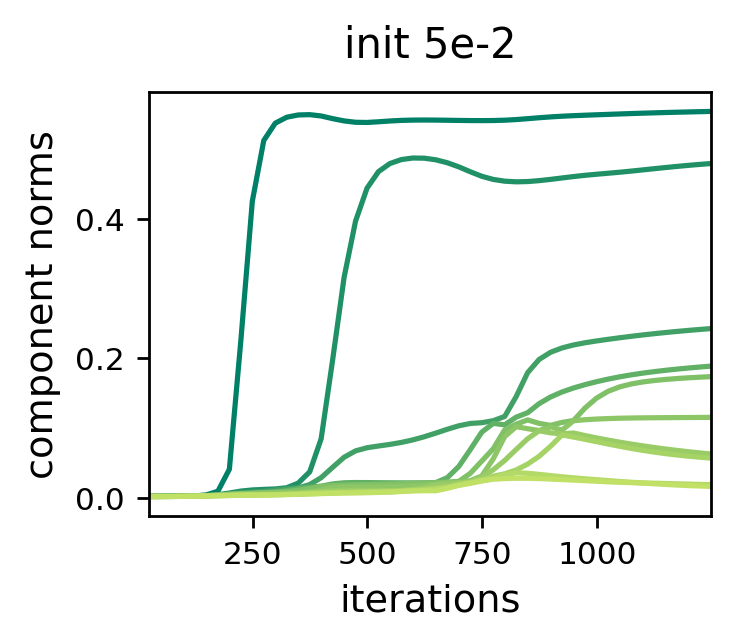}
		}
		\subfloat{
			\includegraphics[width=0.23\textwidth]{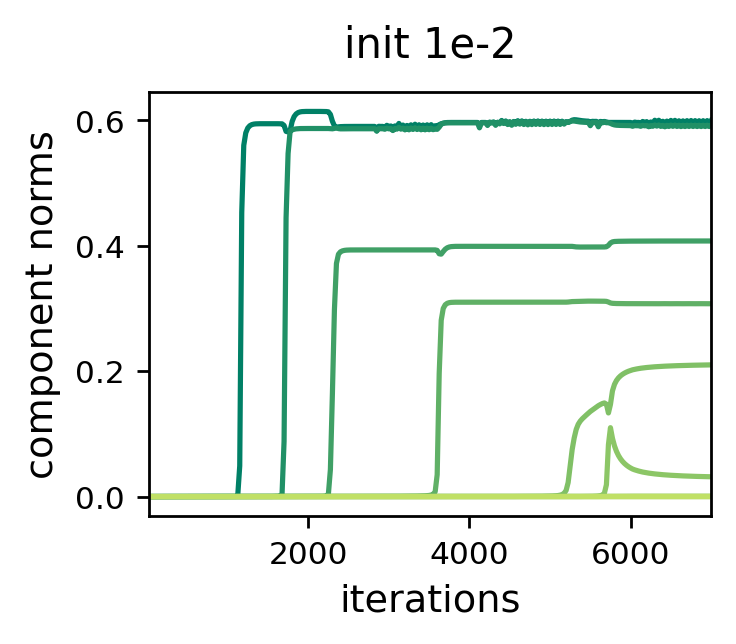}
		}
		\subfloat{
			\includegraphics[width=0.23\textwidth]{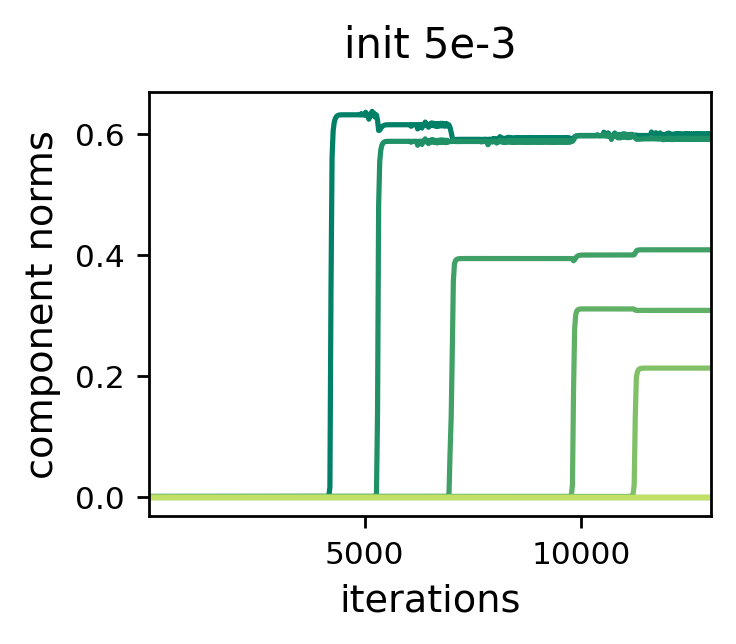}
		}
		\subfloat{
			\includegraphics[width=0.23\textwidth]{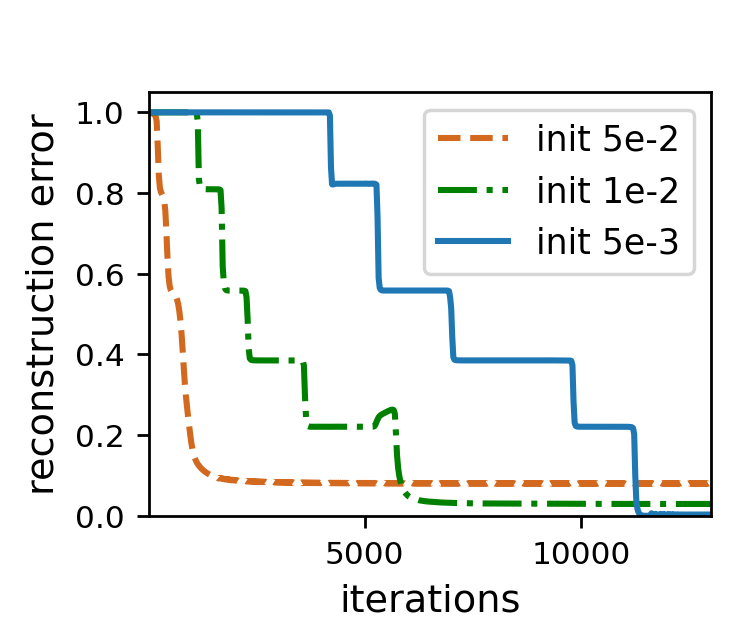}
		}
	\end{center}
	\vspace{-3.5mm}
	\caption{
		Dynamics of gradient descent over tensor factorization --- incremental learning of components yields low tensor rank solutions.
		Presented plots correspond to the task of completing a (tensor) rank $5$ ground truth tensor of size $10$-by-$10$-by-$10$-by-$10$ (order $4$) based on $2000$ observed entries chosen uniformly at random without repetition (smaller sample sizes led to solutions with tensor rank lower than that of the ground truth tensor).
		In each experiment, the $\ell_2$ loss (more precisely, Equation~\eqref{eq:tc_loss} with $\ell ( z ) := z^2$) was minimized via gradient descent over a tensor factorization with $R = 1000$ components (large enough to express any tensor), starting from (small) random initialization.
		First (left) three plots show (Frobenius) norms of the ten largest components under three standard deviations for initialization~---~$0.05, 0.01,$ and $0.005$.
		Further reduction of initialization scale yielded no noticeable change.
		The rightmost plot compares reconstruction errors (Frobenius distance from ground truth) from the three runs.
		To facilitate more efficient experimentation, we employed an adaptive learning rate scheme (see Subappendix~\ref{app:experiments:details} for details).
		Notice that, in accordance with the theoretical analysis of Section~\ref{sec:dynamic}, component norms move slower when small and faster when large, creating an incremental process in which components are learned one after the other.
		This effect is enhanced as initialization scale is decreased, producing low tensor rank solutions that accurately reconstruct the low (tensor) rank ground truth tensor.
		In particular, even though the factorization consists of $1000$ components, when initialization is sufficiently small, only five (tensor rank of the ground truth tensor) substantially depart from zero.
		Appendix~\ref{app:experiments} provides further implementation details, as well as similar experiments with: \emph{(i)} Huber loss (see Equation~\eqref{eq:huber_loss}) instead of $\ell_2$~loss; \emph{(ii)} ground truth tensors of different orders and (tensor) ranks; and \emph{(iii)} tensor sensing (see Appendix~\ref{app:sensing}).
		\vspace{-0.5mm}
	}
	\label{fig:tc_mse_ord4}
\end{figure*}

Recently,~\citet{razin2020implicit} empirically showed that, with small learning rate and near-zero initialization, gradient descent over tensor factorization exhibits an implicit regularization towards low tensor rank.
Our theory (Sections~\ref{sec:dynamic} and~\ref{sec:rank}) explains this implicit regularization through a dynamical analysis~---~we prove that the movement of component norms is attenuated when small and enhanced when large, thus creating an incremental learning effect which becomes more potent as initialization scale decreases.
Figure~\ref{fig:tc_mse_ord4} demonstrates this phenomenon empirically on synthetic low (tensor) rank tensor completion problems.
Figures~\ref{fig:tc_huber_ord4},~\ref{fig:tc_mse_ord3} and~\ref{fig:ts_mse_ord4} in Subappendix~\ref{app:experiments:further} extend the experiment, corroborating our analyses in a wide array of settings.

\subsection{Tensor Rank as Measure of Complexity}
\label{sec:experiments:tensor_rank_complexity}

Implicit regularization in deep learning is typically viewed as a tendency of gradient-based optimization to fit training examples with predictors whose ``complexity'' is as low as possible.
The fact that ``natural'' data gives rise to generalization while other types of data (\eg~random) do not, is understood to result from the former being amenable to fitting by predictors of lower complexity.
A major challenge in formalizing this intuition is that we lack definitions for predictor complexity that are both quantitative (\ie~admit quantitative generalization bounds) and capture the essence of natural data (types of data on which neural networks generalize in practice), in the sense of it being fittable~with~low~complexity.

As discussed in Subsection~\ref{sec:tf:nn}, learning a predictor with multiple discrete input variables and a continuous output can be viewed as a tensor completion problem.
Specifically, with $N \in \N$, $d_1 , \ldots , d_N \in \N$, learning a predictor from domain $\X = \{ 1 , \ldots , d_1 \} \times \cdots \times \{ 1 , \ldots , d_N \}$ to range $\Y = \R$ corresponds to completion of an order~$N$ tensor with mode (axis) dimensions $d_1 , \ldots , d_N$. 
Under this correspondence, any predictor can simply be thought of as a tensor, and vice versa.
We have shown that solving tensor completion via tensor factorization amounts to learning a predictor through a certain neural network (Subsection~\ref{sec:tf:nn}), whose implicit regularization favors solutions with low tensor rank (Sections \ref{sec:dynamic} and~\ref{sec:rank}).
Motivated by these connections, the current subsection empirically explores tensor rank as a measure of complexity for predictors, by evaluating the extent to which it captures natural data, \ie~allows the latter to be fit with low complexity predictors.

As representatives of natural data, we chose the classic MNIST dataset \cite{lecun1998mnist}~---~perhaps the most common benchmark for demonstrating ideas in deep learning~---~and its more modern counterpart Fashion-MNIST \cite{xiao2017fashion}.
A hurdle posed by these datasets is that they involve classification into multiple categories, whereas the equivalence to tensors applies to predictors whose output is a scalar.
It is possible to extend the equivalence by equating a multi-output predictor with multiple tensors, in which case the predictor is associated with multiple tensor ranks.
However, to facilitate a simple presentation, we avoid this extension and simply map each dataset into multiple one-vs-all binary classification problems.
For each problem, we associate the label~$1$ with the active category and $0$ with all the rest, and then attempt to fit training examples with predictors of low tensor rank, reporting the resulting mean squared error, \ie~the residual of the fit.
This is compared against residuals obtained when fitting two types of random data: one generated via shuffling labels, and the other by replacing inputs with noise.

\begin{figure*}
	\begin{center}
		\hspace*{-1mm}
		\includegraphics[width=1\textwidth]{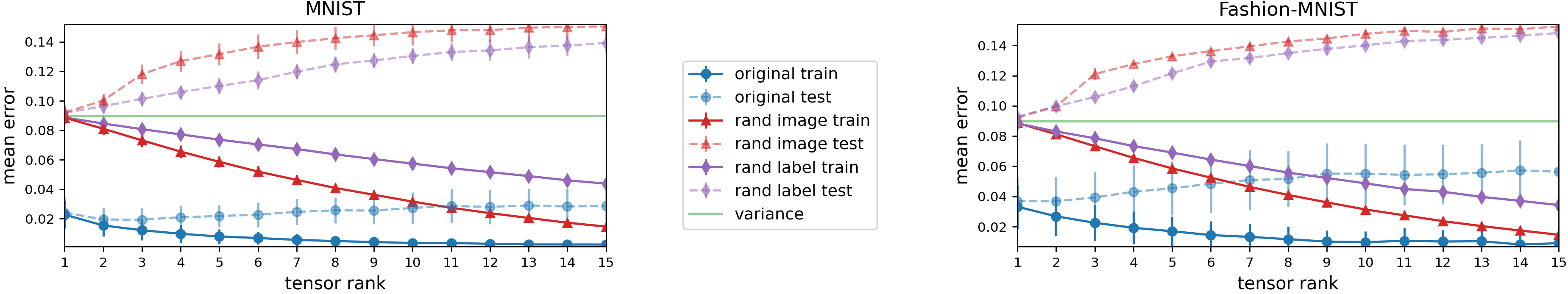}
	\end{center}
	\vspace{-2.75mm}
	\caption{
		Evaluation of tensor rank as measure of complexity~---~standard datasets can be fit accurately with predictors of extremely low tensor rank (far beneath what is required by random datasets), suggesting it may capture the essence of natural data.
		Left and right plots show results of fitting MNIST and Fashion-MNIST datasets, respectively, with predictors of increasing tensor rank.
		Original datasets are compared against two random variants: one generated by replacing images with noise (``rand image''), and the other via shuffling labels (``rand label'').
		As described in the text (Subsection~\ref{sec:experiments:tensor_rank_complexity}), for simplicity of presentation, each dataset was mapped into multiple (ten) one-vs-all prediction tasks (label~$1$ for active category, $0$~for the rest), with fit measured via mean squared error.
		Separately for each one-vs-all prediction task and each value $k \in \{ 1 , \ldots , 15 \}$ for the tensor rank, we applied an approximate numerical method (see Subappendix~\ref{app:experiments:details:natural_data} for details) to find the predictor of tensor rank~$k$ (or less) with which the mean squared error over training examples is minimal.
		We report this mean squared error, as well as that obtained by the predictor on the test set (to mitigate impact of outliers, large squared errors over test samples were clipped~---~see Subappendix~\ref{app:experiments:details:natural_data} for details).
		Plots show, for each value of~$k$, mean (as marker) and standard deviation (as error bar) of these errors taken over the different one-vs-all prediction tasks.
		Notice that the original datasets are fit accurately (low train error) by predictors of tensor rank as low as one, whereas random datasets are not (with tensor rank one, residuals of their fit are close to trivial, \ie~to the variance of the label).
		This suggests that tensor rank as a measure of complexity for predictors has potential to capture the essence of natural data.
		Notice also that, as expected, accurate fit with low tensor rank coincides with accurate prediction on test set, \ie~with generalization.
		For further details, as well as an experiment showing that linear predictors are incapable of accurately fitting the datasets, see Appendix~\ref{app:experiments}.
	}
	\label{fig:mnist_fmnist_rank}
\end{figure*}

Both MNIST and Fashion-MNIST comprise $28$-by-$28$ grayscale images, with each pixel taking one of $256$ possible values.
Tensors associated with predictors are thus of order~$784$, with dimension $256$ in each mode (axis).\note{
In practice, when associating predictors with tensors, it is often beneficial to modify the representation of the input (\cf~\citet{cohen2016expressive}).
For example, in the context under discussion, rather than having the discrete input variables hold pixel intensities, they may correspond to small image patches, where each patch is represented by the index of a centroid it is closest to, with centroids determined via clustering applied to all patches across all images in the dataset.
For simplicity, we did not transform representations in our experiments, and simply operated over raw image pixels.
}
A general rank one tensor can then be expressed as an outer product between $784$ vectors of dimension $256$ each, and accordingly has roughly $784 \cdot 256$ degrees of freedom.
This significantly exceeds the number of training examples in the datasets ($60000$), hence it is no surprise that we could easily fit them, as well as their random variants, with a predictor whose tensor rank is one.
To account for the comparatively small training sets, and render their fit more challenging, we quantized pixels to hold one of two values, \ie~we reduced images from grayscale to black and white.
Following the quantization, tensors associated with predictors have dimension two in each mode, and the number of degrees of freedom in a general rank one tensor is roughly $784 \cdot 2$~---~well below the number of training examples.
We may thus expect to see a difference between the tensor ranks needed for fitting original datasets and those required by the random ones.
This is confirmed by Figure~\ref{fig:mnist_fmnist_rank}, displaying the results of the experiment. 

Figure~\ref{fig:mnist_fmnist_rank} shows that with predictors of low tensor rank, MNIST and Fashion-MNIST can be fit much more accurately than the random datasets.
Moreover, as one would presume, accurate fit with low tensor rank coincides with accurate prediction on unseen data (test set), \ie~with generalization.
Combined with the rest of our results, we interpret this finding as an indication that tensor rank may shed light on both implicit regularization of neural networks, and the properties of real-world data translating this implicit regularization to generalization.

%% file: sec_related.tex
\section{Related Work} \label{sec:related}

Theoretical analysis of implicit regularization induced by gradient-based optimization in deep learning is a highly active area of research.
Works along this line typically focus on simplified settings, delivering results such as:
characterizations of dynamical or statistical aspects of learning \cite{du2018algorithmic,gidel2019implicit,arora2019implicit,brutzkus2020inductive,gissin2020implicit,chou2020gradient};
solutions for test error when data distribution is known \cite{advani2017high,goldt2019dynamics,lampinen2019analytic};
and
proofs of complexity measures being implicitly minimized in certain situations, either exactly or approximately.\note{
Recent results of \citet{vardi2020implicit} imply that under certain conditions, implicit minimization of a complexity measure must be approximate (cannot be exact).
}
The latter type of results is perhaps the most common, covering complexity measures based on:
frequency content of input-output mapping \cite{rahaman2018spectral,xu2018understanding};
curvature of training objective \cite{mulayoff2020unique};
and
norm or margin of weights or input-output mapping \cite{soudry2018implicit,gunasekar2018characterizing,gunasekar2018implicit,jacot2018neural,ji2019implicit,mei2019mean,wu2019implicit,nacson2019convergence,ji2019gradient,oymak2019overparameterized,ali2020implicit,woodworth2020kernel,chizat2020implicit,yun2021unifying}.
An additional complexity measure, arguably the most extensively~studied,~is~matrix~rank.

Rank minimization in matrix completion (or sensing) is a classic problem in science and engineering (\cf~\citet{davenport2016overview}).
It relates to deep learning when solved via linear neural network, \ie~through matrix factorization.
The literature on matrix factorization for rank minimization is far too broad to cover here~---~we refer to~\citet{chi2019nonconvex} for a recent review.
Notable works proving rank minimization via matrix factorization trained by gradient descent with no explicit regularization are \citet{tu2016low,ma2018implicit,li2018algorithmic}.
\citet{gunasekar2017implicit} conjectured that this implicit regularization is equivalent to norm minimization, but the recent studies \citet{arora2019implicit,razin2020implicit,li2021towards} argue otherwise, and instead adopt a dynamical view, ultimately establishing that (under certain conditions) the implicit regularization in matrix factorization is performing a greedy low rank search.
These studies are relevant to ours in the sense that we generalize some of their results to tensor factorization.
As typical in transitioning from matrices to tensors (see \citet{hillar2013most}), this generalization entails significant challenges necessitating use of fundamentally different techniques.

Recovery of low (tensor) rank tensors from incomplete observations via tensor factorizations is a setting of growing interest (\cf~\citet{acar2011scalable,narita2012tensor,anandkumar2014tensor,jain2014provable,yokota2016smooth,karlsson2016parallel,xia2017polynomial,zhou2017tensor,cai2019nonconvex} and the survey \citet{song2019tensor}).\note{
It stands in contrast to inferring representations for fully observed low (tensor) rank tensors via tensor factorizations (\cf~\citet{wang2020lazy})~---~a setting where implicit regularization (as conventionally defined in deep learning) is not applicable.
}
However, the experiments of~\citet{razin2020implicit} comprise the only evidence we are aware of for successful recovery under gradient-based optimization with no explicit regularization (in particular without imposing low tensor rank through a factorization).\note{
	In a work parallel to ours, \citet{milanesi2021implicit}~provides further empirical evidence for such implicit regularization.
}
The current paper provides the first theoretical support for this implicit regularization.

We note that the equivalence between tensor factorizations and different types of neural networks has been studied extensively, primarily in the context of expressive power (see, \eg, \citet{cohen2016expressive,cohen2016convolutional,sharir2016tensorial,cohen2017inductive,cohen2017analysis,sharir2018expressive,levine2018deep,cohen2018boosting,levine2018benefits,balda2018tensor,khrulkov2018expressive,levine2019quantum,khrulkov2019generalized,levine2020limits}).
Connections between tensor analysis and generalization in deep learning have also been made (\cf~\citet{li2020understanding}), but to the best of our knowledge, the notion of quantifying the complexity of predictors through their tensor rank (supported empirically in Subsection~\ref{sec:experiments:tensor_rank_complexity}) is novel to this work.

%% file: sec_conclusion.tex
\section{Conclusion} \label{sec:conclusion}

In this paper we provided the first theoretical analysis of implicit regularization in tensor factorization.
To circumvent the notorious difficulty of tensor problems (see \citet{hillar2013most}), we adopted a dynamical systems perspective, and characterized the evolution that gradient descent (with small learning rate and near-zero initialization) induces on the components of a factorization.
The characterization suggests a form of greedy low tensor rank search, rigorously proven under certain conditions.
Experiments demonstrated said phenomena.

A major challenge in mathematically explaining generalization in deep learning is to define measures for predictor complexity that are both quantitative (\ie~admit quantitative generalization bounds) and capture the essence of ``natural'' data (types of data on which neural networks generalize in practice), in the sense of it being fittable with low complexity.
Motivated by the fact that tensor factorization is equivalent to a certain non-linear neural network, and by our analysis implying that the implicit regularization of this network minimizes tensor rank, we empirically explored the potential of the latter to serve as a measure of predictor complexity.
We found that it is possible to fit standard image recognition datasets (MNIST and Fashion-MNIST) with predictors of extremely low tensor rank (far beneath what is required for fitting random data), suggesting that it indeed captures aspects of natural data.

The neural network to which tensor factorization is equivalent entails multiplicative non-linearity.
It was shown in~\citet{cohen2016convolutional} that more prevalent non-linearities, for example rectified linear unit (ReLU), can be accounted for by considering \emph{generalized tensor factorizations}.
Studying the implicit regularization in generalized tensor factorizations (both empirically and theoretically) is regarded as a promising direction for future work.

There are two drawbacks to tensor factorization when applied to high-dimensional prediction problems.
The first is technical, and relates to numerical stability~---~an order~$N$ tensor factorization involves products of $N$ numbers, thus is susceptible to arithmetic underflow or overflow if $N$ is large. 
Care should be taken to avoid this pitfall, for example by performing computations in log domain (as done in \citet{cohen2014simnets,cohen2016deep,sharir2016tensorial}).
The second limitation is more fundamental, arising from the fact that tensor rank~---~the complexity measure implicitly minimized~---~is oblivious to the ordering of tensor modes (axes).
This means that the implicit regularization does not take into account how predictor inputs are arranged (\eg, in the context of image recognition, it does not take into account spatial relationships between pixels).
A potentially promising path for overcoming this limitation is introduction of \emph{hierarchy} into the tensor factorization, equivalent to adding depth to the corresponding neural network (\cf~\citet{cohen2016expressive}).
It may then be the case that a \emph{hierarchical tensor rank} (see \citet{hackbusch2012tensor}), which does account for mode ordering, will be implicitly minimized.
We hypothesize that hierarchical tensor ranks may be key to explaining both implicit regularization of contemporary deep neural networks, and the properties of real-world data translating this implicit regularization to generalization.

%% file: ack.tex
This work was supported by a Google Research Scholar Award, a Google Research Gift, the Yandex Initiative in Machine Learning, Len Blavatnik and the Blavatnik Family Foundation, and Amnon and Anat Shashua.

%% file: app_sensing.tex
\section{Extension to Tensor Sensing} \label{app:sensing}

Our theoretical analyses (Sections \ref{sec:dynamic} and~\ref{sec:rank}) are presented in the context of tensor completion, but readily extend to the more general task of \emph{tensor sensing}~---~reconstruction of an unknown tensor from linear measurements (projections).
In this appendix we outline the extension.
Empirical demonstrations for tensor sensing are given in Subappendix~\ref{app:experiments:further} (Figure~\ref{fig:ts_mse_ord4}).

For a ground truth tensor $\W^* \in \R^{d_1, \ldots, d_N}$ and measurement tensors $\{ \A_i \in \R^{d_1, \ldots, d_N } \}_{i = 1}^m$, the goal in tensor sensing is to reconstruct $\W^*$ based on $\{ \inprodnoflex{ \A_i }{ \W^* } \in \R \}_{i = 1}^{m}$, where $\inprod{\, \cdot \,}{\, \cdot \,}$ represents the standard inner product.
Similarly to tensor completion (\cf~Equation~\eqref{eq:tc_loss}), a standard loss function for the task is:
\[
\L_{s} ( \W ) = \frac{1}{m} \sum\nolimits_{i = 1}^m \ell \left ( \inprod{ \A_i }{ \W }  - \inprod{ \A_i }{ \W^* }  \right )
\text{\,,}
\]
where $\L_{s} : \R^{d_1, \ldots, d_N} \to \R_{\geq 0}$, and $\ell : \R \to \R_{\geq 0}$ is differentiable and locally smooth.
Note that tensor completion is a special case, in which the measurement tensors hold $1$ at a single entry and $0$ elsewhere.

Beginning with Section~\ref{sec:dynamic}, its results (in particular Lemma~\ref{lem:balancedness_conservation_body}, Theorem~\ref{thm:dyn_fac_comp_norm_unbal} and Corollary~\ref{cor:dyn_fac_comp_norm_balanced}) hold (and are proven in Subappendix~\ref{app:proofs}) for any differentiable and locally smooth~$\L (\cdot)$, thus they apply as is to tensor sensing.
Turning to Section~\ref{sec:rank}, the extension of Theorem~\ref{thm:approx_rank_1} and Corollary~\ref{corollary:converge_rank_1} to tensor sensing (with Huber loss) is straightforward.
Proofs rely on the specifics of tensor completion only in the preliminary Lemmas~\ref{lem:huber_loss_const_grad_near_zero},~\ref{lem:huber_loss_smooth} and~\ref{lem:huber_cp_objective_is_smooth_over_bounded_domain} (Subappendix~\ref{app:proofs:approx_rank_1:prelim_lemmas}), for which analogous lemmas may readily be established.
Thus, up to slight changes in constants if $\max_{i = 1, \ldots, m} \normnoflex{ \A_i} > 1$, the results carry over.

\subsection{Stronger Results Under Restricted Isometry Property}
\label{app:sensing:rip}

In the classic setting of \emph{matrix sensing} (tensor sensing with order~$N = 2$), a commonly studied condition on the measurement matrices is the \emph{restricted isometry property}.
This condition allows for efficient recovery when the ground truth matrix has low rank, and holds with high probability when the entries of the measurement matrices are drawn independently from a zero-mean sub-Gaussian distribution (\cf~\citet{recht2010guaranteed}).
The notion of restricted isometry property extends from matrix to tensor sensing (\ie~from order $N = 2$ to arbitrary $N \in \N_{\geq 2}$)~---~see~\citet{rauhut2017low,ibrahim2020recoverability}.
When it applies, the tensor sensing analogues of Theorem~\ref{thm:approx_rank_1} and Corollary~\ref{corollary:converge_rank_1} can be strengthened as described below.

\medskip

In the context of tensor sensing, the restricted isometry property is defined as follows.
\begin{definition}
\label{def:rip}
We say that the measurement tensors $\{ \A_i \in \R^{d_1, \ldots, d_N } \}_{n = 1}^m$ satisfy \emph{$r$-restricted isometry property} (\emph{$r$-RIP}) with parameter $\delta \in [0, 1)$ if:
\[
(1 - \delta) \norm{ \W }^2 \leq \sum\nolimits_{i = 1}^m \inprod{ \A_i }{ \W }^2 \leq (1 + \delta) \norm{ \W }^2
\text{\,,}
\]
for all $\W \in \R^{d_1, \ldots, d_N}$ of tensor rank $r$ or less.
\end{definition}
By~\citet{ibrahim2020recoverability}, given $m \in \OO ( \log (N) \cdot \sum_{n = 1}^N d_n )$ measurement tensors with entries drawn independently from a zero-mean sub-Gaussian distribution, $1$-RIP holds with high probability.
In this case, we may strengthen the tensor sensing analogue of Theorem~\ref{thm:approx_rank_1}, such that it ensures that arbitrarily small initialization leads tensor factorization to follow a rank one trajectory for an arbitrary amount of time, regardless of the distance traveled.
That is, with the notations of Theorem~\ref{thm:approx_rank_1}, for any time duration $T > 0$ and degree of approximation $\epsilon \in ( 0 , 1 )$, if initialization is sufficiently small, $\overline{W}_e (t)$~is within $\epsilon$~distance from a balanced rank one trajectory emanating from~$\S$ at least until time $t \geq T$.
To see it is so, notice that since the loss function during gradient flow is monotonically non-increasing, $\sum_{i = 1}^m \inprodnoflex{ \A_i }{ \W_1 (t) }^2$ is bounded through time for any rank one trajectory~$\W_1 (t)$.
In turn, since the measurement tensors satisfy $1$-RIP, all such trajectories emanating from $\S$ are confined to a ball of radius $D > 0$ about the origin, for some $D > 0$.
By the tensor sensing analogue of Theorem~\ref{thm:approx_rank_1}, sufficiently small initialization ensures that there exists $\W_1 (t)$~---~a balanced rank one trajectory emanating from $\S$~---~such that $\overline{\W}_e (t)$ is within $\epsilon$~distance from it at least until $t \geq T$ or $\normnoflex{\overline{\W}_e (t)} \geq D + 1$.
However, we know that $\normnoflex{\W_1 (t)} \leq D$, and so $\overline{\W}_e (t)$ cannot reach norm of $D + 1$ before time~$T$, as that would entail a contradiction~---~$\normnoflex{\W_1 (t)} > D$.
As a consequence of the above, in the tensor sensing analogue of Corollary~\ref{corollary:converge_rank_1}, when $1$-RIP is satisfied we need not assume all balanced rank one trajectories emanating from~$\S$ are jointly bounded.

%% file: app_experiments.tex
\begin{figure*}[h!]
	\vspace{-3mm}
	\begin{center}
		\subfloat{
			\includegraphics[width=0.23\textwidth]{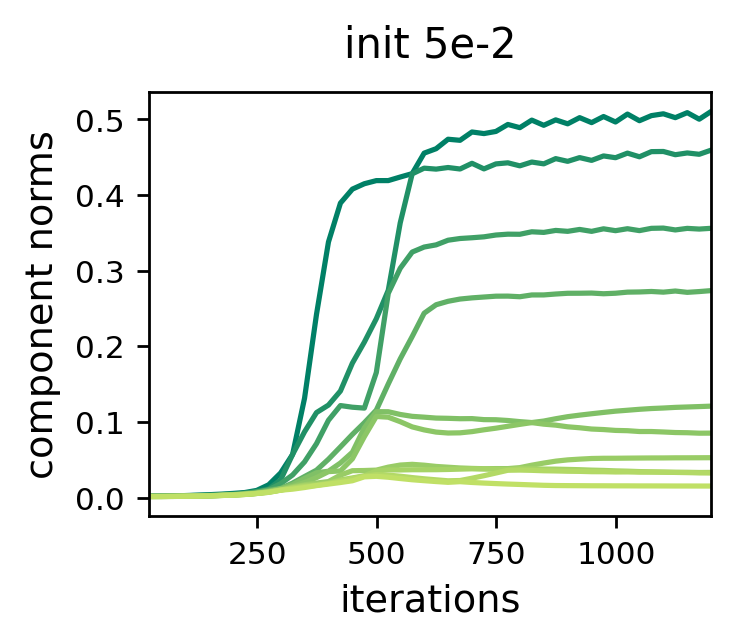}
		}
		\subfloat{
			\includegraphics[width=0.23\textwidth]{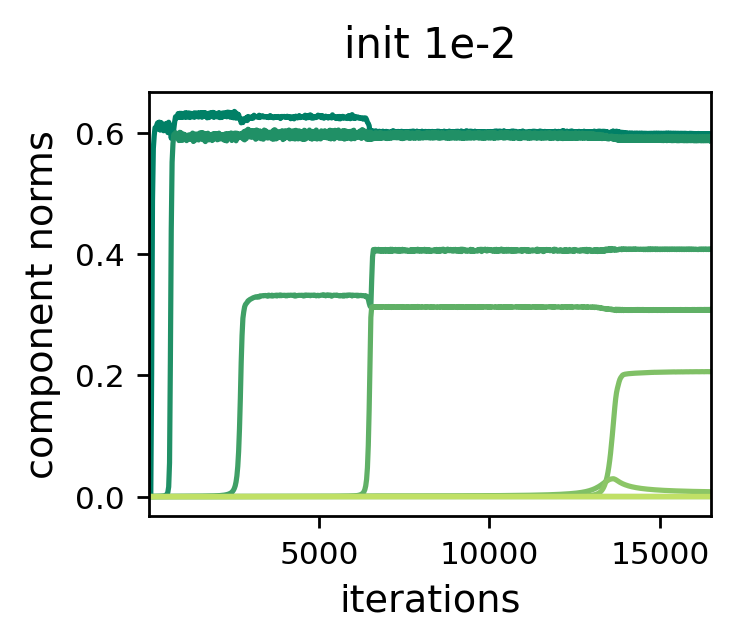}
		}
		\subfloat{
			\includegraphics[width=0.23\textwidth]{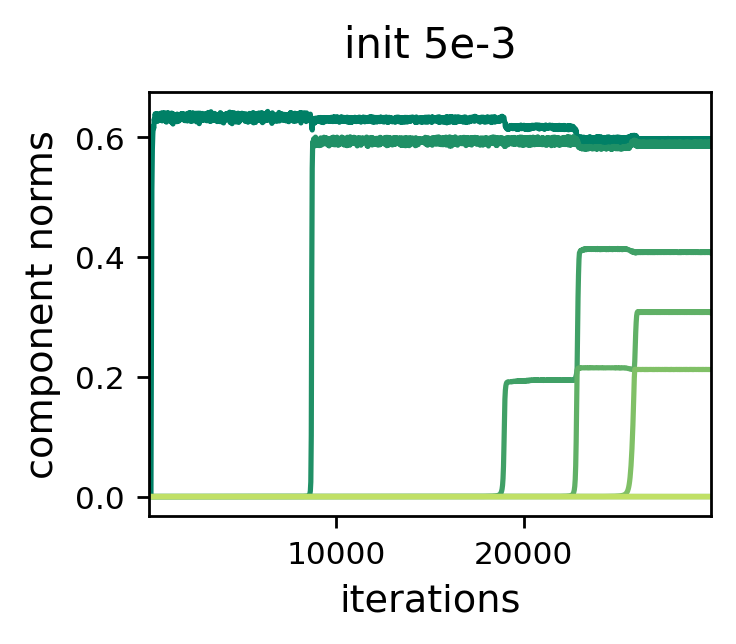}
		}
		\subfloat{
			\includegraphics[width=0.23\textwidth]{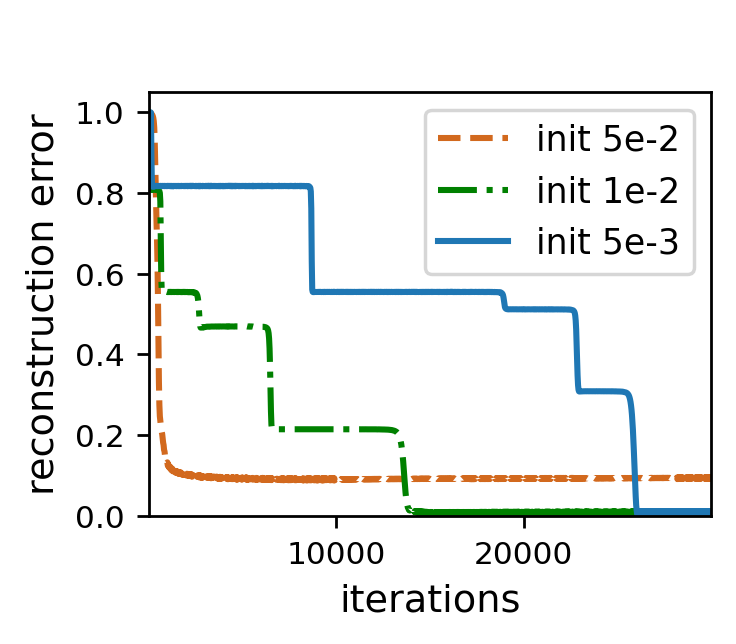}
		}
	\end{center}
	\vspace{-3mm}
	\caption{
		Dynamics of gradient descent over tensor factorization (with Huber loss) --- incremental learning of components yields low tensor rank solutions.
		This figure is identical to Figure~\ref{fig:tc_mse_ord4}, except that the minimized objective (Equation~\eqref{eq:tc_loss}) is based on Huber loss~($\ell_h ( \cdot )$ from Equation~\eqref{eq:huber_loss}) instead of $\ell_2$~loss.
		In accordance with Assumption~\ref{assump:delta_h}, the transition point~$\delta_h$ was set to $5 \cdot 10^{-7}$~---~smaller than the absolute value of observed entries (though larger $\delta_h$ led to similar results).
		For further details see caption of Figure~\ref{fig:tc_mse_ord4}, as well as Subappendix~\ref{app:experiments:details:synth_ts}.
	}
	\vspace{-1mm}
	\label{fig:tc_huber_ord4}
\end{figure*}

\begin{figure*}[h!]
	\vspace{-3mm}
	\begin{center}
		\subfloat{
			\includegraphics[width=0.23\textwidth]{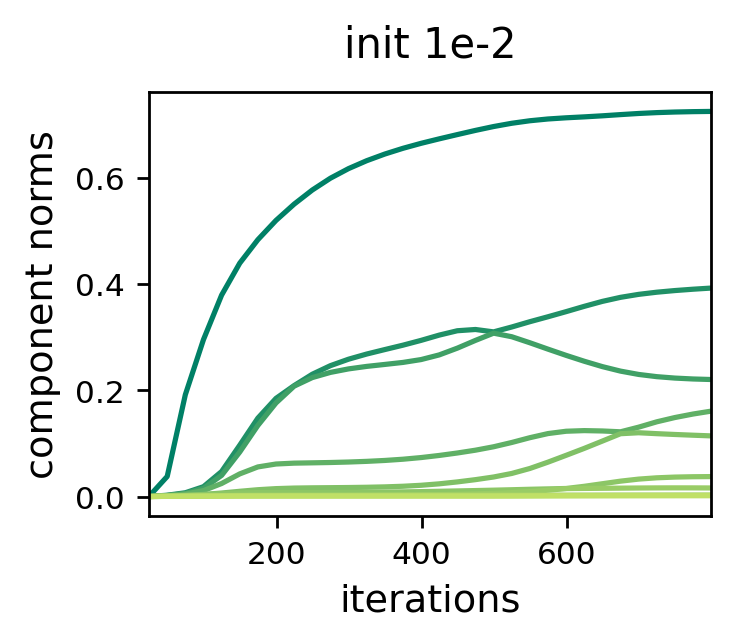}
		}
		\subfloat{
			\includegraphics[width=0.23\textwidth]{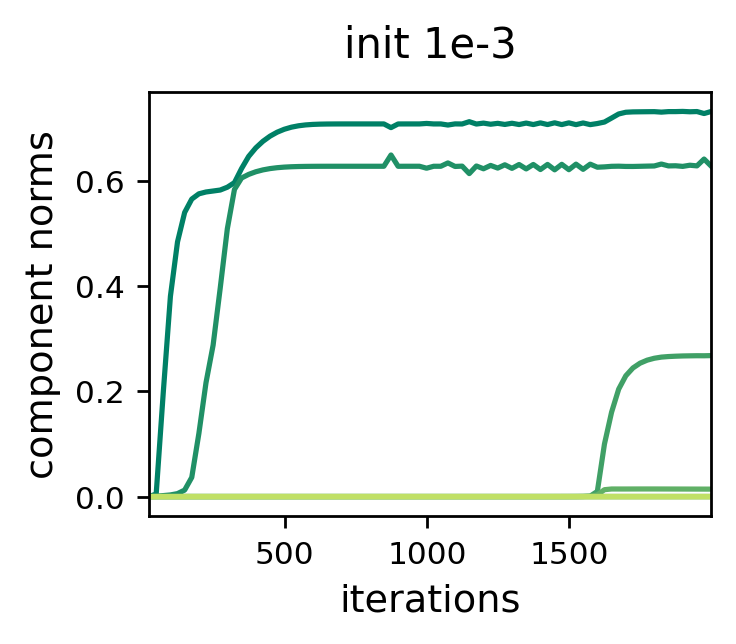}
		}
		\subfloat{
			\includegraphics[width=0.23\textwidth]{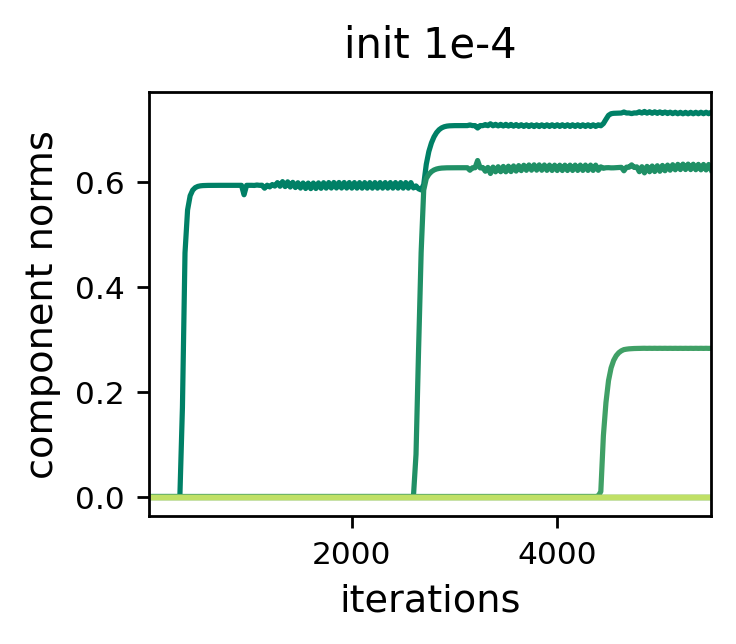}
		}
		\subfloat{
			\includegraphics[width=0.23\textwidth]{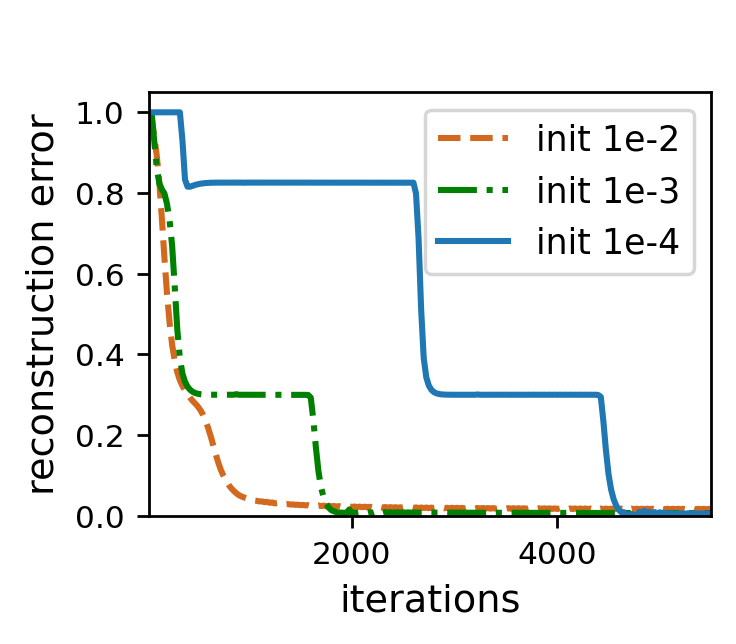}
		}
	\end{center}
	\vspace{-3mm}
	\caption{
		Dynamics of gradient descent over (order~$3$) tensor factorization --- incremental learning of components yields low tensor rank solutions.
		This figure is identical to Figure~\ref{fig:tc_mse_ord4}, except that: \emph{(i)} the ground truth tensor is of (tensor) rank $3$ with size $10$-by-$10$-by-$10$ (order $3$), completed based on $300$ observed entries (smaller sample sizes led to solutions with tensor rank lower than that of the ground truth tensor); and \emph{(ii)} the employed tensor factorization consists of $R = 100$ components (large enough to express any tensor).
		For further details see caption of Figure~\ref{fig:tc_mse_ord4}, as well as Subappendix~\ref{app:experiments:details:synth_ts}.
	}
	\vspace{-1mm}
	\label{fig:tc_mse_ord3}
\end{figure*}

\begin{figure*}[h!]
	\vspace{-3mm}
	\begin{center}
		\subfloat{
			\includegraphics[width=0.23\textwidth]{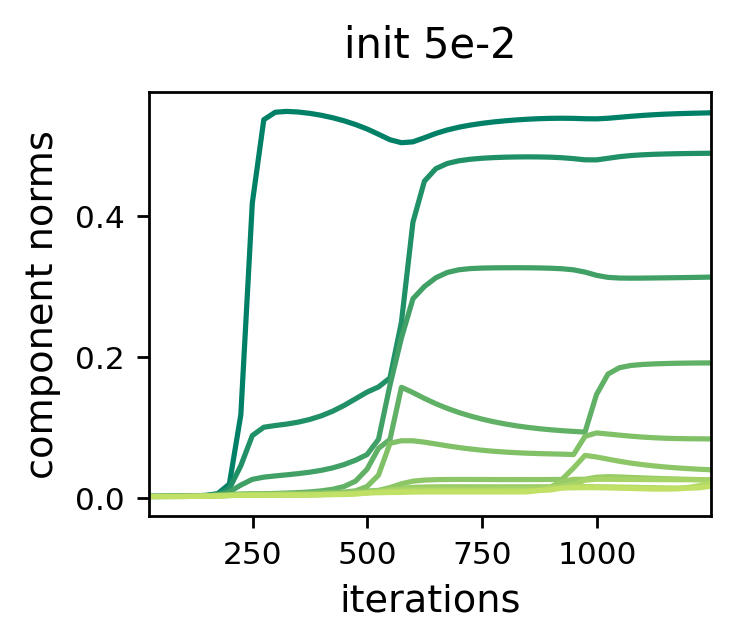}
		}
		\subfloat{
			\includegraphics[width=0.23\textwidth]{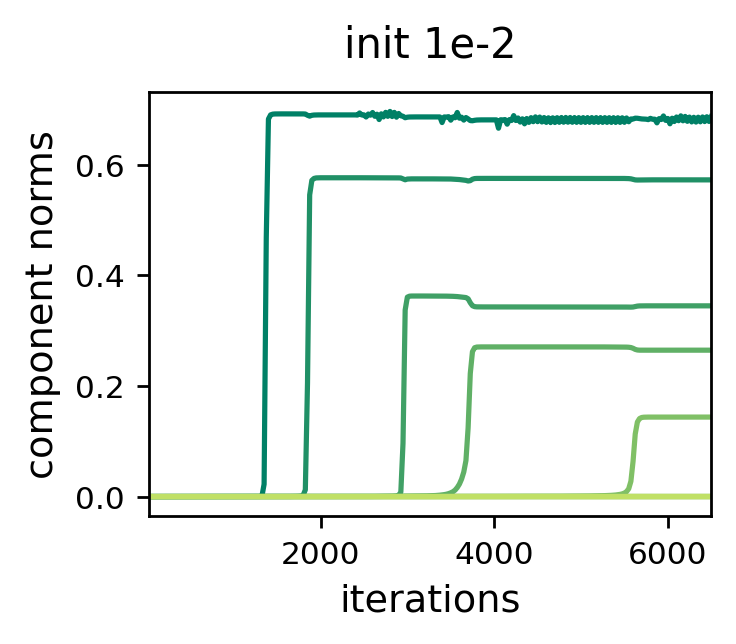}
		}
		\subfloat{
			\includegraphics[width=0.23\textwidth]{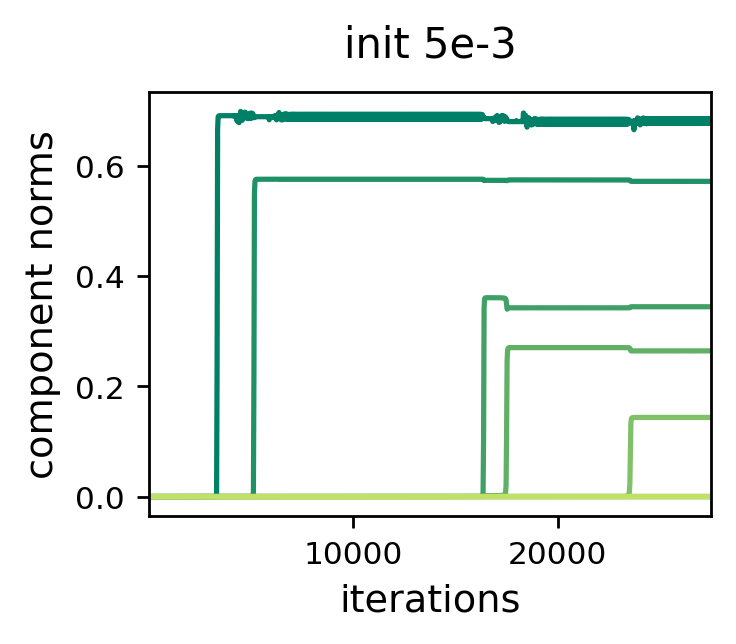}
		}
		\subfloat{
			\includegraphics[width=0.23\textwidth]{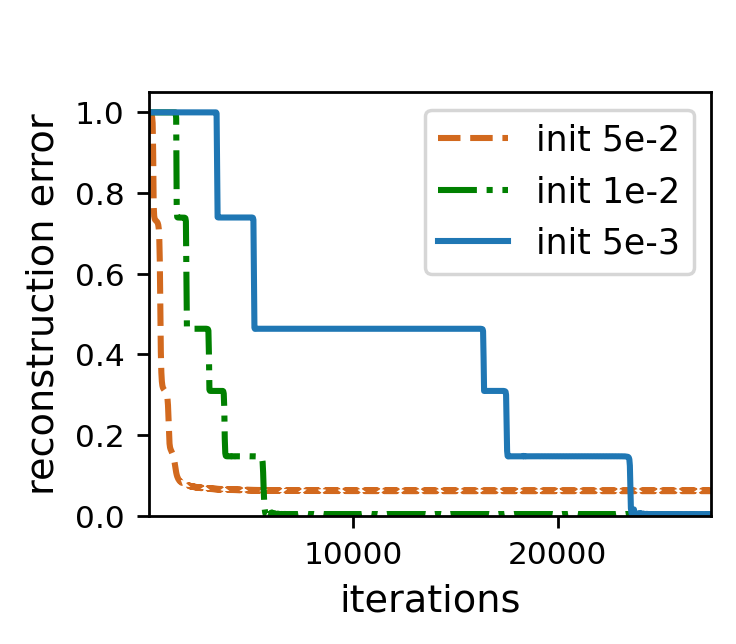}
		}
	\end{center}
	\vspace{-3mm}
	\caption{
		Dynamics of gradient descent over tensor factorization (on tensor sensing task) --- incremental learning of components yields low tensor rank solutions.
		This figure is identical to Figure~\ref{fig:tc_mse_ord4}, except that reconstruction of the ground truth tensor is based on $2000$ linear measurements (instead of $2000$ randomly chosen entries), \ie~on $\{ \inprodnoflex{ \A_i }{ \W^* } \}_{i = 1}^{2000}$, where $\W^* \in \R^{d_1, \ldots, d_N}$ is the ground truth tensor and \smash{$\A_1, \ldots, \A_{2000} \in \R^{d_1, \ldots, d_N}$} are measurement tensors sampled independently from a zero-mean Gaussian distribution (see Appendix~\ref{app:sensing} for a description of the tensor sensing task).
		For further details see caption of Figure~\ref{fig:tc_mse_ord4}, as well as Subappendix~\ref{app:experiments:details:synth_ts}.
	}
	\vspace{-1mm}
	\label{fig:ts_mse_ord4}
\end{figure*}

\begin{table}[t]
	\caption{
		Linear predictors are incapable of accurately fitting the datasets in the experiment reported by Figure~\ref{fig:mnist_fmnist_rank}.
		Table presents mean squared errors (over train and test sets) attained by fitting linear predictors to the one-vs-all prediction tasks induced by MNIST and Fashion-MNIST datasets, as well as their random variants (in compliance with Figure~\ref{fig:mnist_fmnist_rank}, to mitigate impact of outliers, large squared errors over test samples were clipped~---~see Subappendix~\ref{app:experiments:details:natural_data} for details).
		For each dataset, mean and standard deviation of train and test errors, taken over the different one-vs-all prediction tasks, are reported.
		Notice that all errors are not far from~$0.09$~---~the variance of the label~---~which is trivial to achieve.		
		For further details see caption of Figure~\ref{fig:mnist_fmnist_rank}, as well as Subappendix~\ref{app:experiments:details:natural_data}.
	}
	\label{table:mnist_fmnist_linear_errors}
	\begin{center}
		\begin{small}
			\begin{sc}
				\fontsize{8}{1}
				\begin{tabular}{lcccc}
					\toprule
					& \multicolumn{2}{c}{MNIST} & \multicolumn{2}{c}{Fashion-MNIST}\\
					& Train& Test & Train& Test \\
					\midrule
					Original & $3.90\cdot{}10^{-2}$ $\pm$ $8.37\cdot{}10^{-3}$ & $3.92\cdot{}10^{-2}$ $\pm$ $8.04\cdot{}10^{-2}$ & $4.09\cdot{}10^{-2}$ $\pm$ $1.50\cdot{}10^{-2}$ & $4.24\cdot{}10^{-2}$ $\pm$ $1.58\cdot{}10^{-2}$ \\
					Rand Image & $8.88\cdot{}10^{-2}$ $\pm$ $4.24\cdot{}10^{-3}$ & $9.11\cdot{}10^{-2}$ $\pm$ $4.80\cdot{}10^{-3}$ & $8.88\cdot{}10^{-2}$ $\pm$ $3.11\cdot{}10^{-5}$ & $9.12\cdot{}10^{-2}$ $\pm$ $2.07\cdot{}10^{-4}$ \\
					Rand Label & $8.89\cdot{}10^{-2}$ $\pm$ $4.22\cdot{}10^{-3}$ & $9.09\cdot{}10^{-2}$ $\pm$ $4.77\cdot{}10^{-3}$ & $8.88\cdot{}10^{-2}$ $\pm$ $7.46\cdot{}10^{-5}$ & $9.11\cdot{}10^{-2}$ $\pm$ $2.23\cdot{}10^{-4}$ \\
					\bottomrule
				\end{tabular}
			\end{sc}
		\end{small}
	\end{center}
	\vskip -0.1in
\end{table}

\section{Further Experiments and Implementation Details} \label{app:experiments}

\subsection{Further Experiments} \label{app:experiments:further}

Figures~\ref{fig:tc_huber_ord4},~\ref{fig:tc_mse_ord3} and~\ref{fig:ts_mse_ord4} supplement Figure~\ref{fig:tc_mse_ord4} from Subsection~\ref{sec:experiments:dyn} by including, respectively: \emph{(i)} Huber loss (Equation~\eqref{eq:huber_loss}) instead of $\ell_2$~loss; \emph{(ii)} ground truth tensors of different orders and (tensor) ranks; and \emph{(iii)} tensor sensing (see Appendix~\ref{app:sensing}).
Table~\ref{table:mnist_fmnist_linear_errors} supplements Figure~\ref{fig:mnist_fmnist_rank}, reporting mean squared errors of linear predictors fitted to the different datasets.

\subsection{Implementation Details} \label{app:experiments:details}

Below are implementation details omitted from our experimental reports (Section~\ref{sec:experiments} and Subappendix~\ref{app:experiments:further}).
Source code for reproducing our results and figures\ifdefined\CAMREADY
~can be found at \url{https://github.com/noamrazin/imp_reg_in_tf} (based on the PyTorch framework~\cite{paszke2017automatic}).
\else
, based on the PyTorch framework~\cite{paszke2017automatic}, is attached as supplementary material and will be made publicly available.
\fi

\subsubsection{Dynamics of Learning (Figures~\ref{fig:tc_mse_ord4},~\ref{fig:tc_huber_ord4}, ~\ref{fig:tc_mse_ord3} and~\ref{fig:ts_mse_ord4})} \label{app:experiments:details:synth_ts}

The number of components $R$ was set to ensure an unconstrained search space, \ie~to $10^2$ and $10^3$ for tensor sizes $10$-by-$10$-by-$10$ and $10$-by-$10$-by-$10$-by-$10$ respectively.\footnote{
	For any $d_1 , \ldots , d_N \in \N$, setting $R = ( \Pi_{n = 1}^N d_n ) / \max \{ d_n \}_{n = 1}^N$ suffices for expressing all tensors in $\R^{d_1 , \ldots , d_N}$ (\cf~\citet{hackbusch2012tensor}).
}
Gradient descent was initialized randomly by sampling each weight independently from a zero-mean Gaussian distribution, and was run until the loss reached a value lower than $10^{-8}$ or~$10^6$ iterations elapsed.
For each figure, experiments were carried out with standard deviation of initialization varying over $\{ 0.05, 0.01, 0.005, 0.001, 0.0005, 0.0001, 0.00005 \}$.
Reported are representative runs illustrating the different types of dynamics encountered.
To facilitate more efficient experimentation, we employed an adaptive learning rate scheme, where at each iteration a base learning rate is divided by the square root of an exponential moving average of squared gradient norms.
That is, with base learning rate $\eta = 10^{-2}$ and weighted average coefficient $\beta = 0.99$, at iteration~$t$ the learning rate was set to $\eta_t = \eta / (\sqrt{\gamma_t / (1 - \beta^t)} + 10^{-6})$, where \smash{$\gamma_t = \beta \cdot \gamma_{t-1} + (1 - \beta) \cdot \sum \hspace{0mm}_{r = 1}^{R} \hspace{0mm}_{n = 1}^{N} \normnoflex{ \nicefrac{\partial}{\partial \w_r^{n}} \phi ( \{ \w^{ n }_r ( t ) \}_{r = 1}^R\hspace{0mm}_{n = 1}^N ) }^2$} and $\gamma_0 = 0$.
Note that only the learning rate (step size) is affected by this scheme, not the direction of movement.
When compared to optimization with a fixed (small) learning rate, no significant difference in the dynamics was observed, while run times were significantly shorter.

Generating a ground truth rank~$R^*$ tensor $\W^* \in \R^{d_1, \ldots, d_N}$ was done by computing $\W^* = \sum\nolimits_{r = 1}^{R^*} \w^{*  1 }_r \tenp \cdots \tenp \w^{*  N }_r$, with $\{ \w^{* n }_r \in \R^{d_n} \}_{r = 1}^{R^*} \hspace{0mm}_{n = 1}^{N}$ drawn independently from the standard normal distribution.
For convenience, the ground truth tensor was normalized to be of unit Frobenius norm.
In tensor completion experiments (Figures~\ref{fig:tc_mse_ord4},~\ref{fig:tc_huber_ord4} and~\ref{fig:tc_mse_ord3}), the subset of observed entries was chosen uniformly at random.
For tensor sensing (Figure~\ref{fig:ts_mse_ord4}), we sampled the entries of all measurement tensors independently from a zero-mean Gaussian distribution with standard deviation $10^{-2}$ (ensures measurement tensors have expected square Frobenius norm of $1$).

\subsubsection{Tensor Rank as Measure of Complexity (Figure~\ref{fig:mnist_fmnist_rank} and Table~\ref{table:mnist_fmnist_linear_errors})} \label{app:experiments:details:natural_data}

For both MNIST and Fashion-MNIST datasets, we quantized pixels to hold either $0$ or~$1$ by rounding grayscale values to the nearest integer.
Random input datasets were created by replacing all pixels in all images with random values ($0$ or~$1$) drawn independently from the uniform distribution.
Random label datasets were generated by shuffling labels according to a random permutation, separately for train and test sets.

Given a prediction task, fitting the corresponding tensor completion problem with a predictor of tensor rank~$k$ (or less) was done by minimizing the mean squared error over a $k$-component tensor factorization.
Stochastic gradient descent, using the Adam optimizer~\cite{kingma2014adam} with learning rate $5 \cdot 10^{-4}$, default $\beta_1, \beta_2$ coefficients, and a batch size of $5000$, was run until the loss reached a value lower than $10^{-8}$ or $10^4$ iterations elapsed.
For numerical stability, factorization weights were initialized near one.
Namely, their initial values were sampled independently from a Gaussian distribution with mean one and standard deviation~$10^{-3}$.
To accelerate convergence, label values ($0$ or~$1$) were scaled up by two during optimization (thereby ensuring symmetry about initialization), with predictions of resulting models scaled down by the same factor during evaluation.
Results reported in Table~\ref{table:mnist_fmnist_linear_errors} were obtained using the ridge regression implementation of scikit-learn~\cite{scikit-learn} with $\alpha = 0.5$ (setting $\alpha = 0$, \ie~using unregularized linear regression, led to numerical issues due to bad conditioning of the data).
Lastly, to mitigate impact of outliers, in both Figure~\ref{fig:mnist_fmnist_rank} and Table~\ref{table:mnist_fmnist_linear_errors} squared errors over test samples were clipped at one, \ie~taken to be the minimum between one and the calculated error.

%% file: app_proofs.tex
\section{Deferred Proofs}
\label{app:proofs}

\subsection{Notations}
\label{app:proofs:notations}

For $N \in \N$, let $[N] := \{ 1, \ldots, N \}$.
We use $\inprod{ \cdot }{ \cdot}$ to denote the standard Euclidean (Frobenius) inner product between two vectors, matrices, or tensors, and $\norm{ \cdot }$ to denote the norm induced by it.
Furthermore, we denote the outer and Kronecker products by $\tenp$ and $\kronp$, respectively.
For a tensor $\W \in \R^{d_1 ,  \ldots , d_N}$ and $n \in [N]$, we let $\mat{ \W }_{n}$ be the mode-$n$ matricization of $\W$, \ie~its arrangement as a matrix where the rows correspond to the $n$'th mode and the columns correspond to all other modes (see Subsection~2.4 in~\citet{kolda2009tensor}).

\subsection{Useful Lemmas}
\label{app:proofs:useful_lemmas}

\subsubsection{Technical}
\label{app:proofs:useful_lemmas:technical}

Following are several technical lemmas, which are used throughout the proofs.

\begin{lemma}
	\label{lem:inp_with_tenp_to_mat_kronp}
	For any $\W \in \R^{d_1 ,  \ldots , d_N}$ and $\{ \w^{n} \in \R^{d_n} \}_{n = 1}^N$, where $d_1, \ldots, d_N \in \N$, it holds that:
	\[
	\inprod{ \W }{ \tenp_{n' = 1}^N \w^{n'} } = \inprod{ \mat{ \W }_n \cdot \kronp_{n' \neq n} \w^{n'} }{ \w^{n} } \quad , ~n = 1, \ldots, N
	\text{\,.}
	\]
\end{lemma}

\begin{proof}
	To simplify presentation, we prove the equality for $n = 1$.
	For $n = 2, \ldots, N$, an analogous computation yields the desired result.
	By opening up the inner product and applying straightforward computations, we conclude:
	\[
	\begin{split}
		\inprod{ \W }{ \tenp_{n' = 1}^N \w^{n'} } & = \sum_{i_1 = 1}^{d_1} \ldots \sum_{i_N = 1}^{d_N} [ \W ]_{i_1, \ldots, i_N}  \cdot \prod_{n' = 1}^N [ \w^{n'}]_{i_{n'}} \\
		& = \sum_{i_1 = 1}^{d_1} [ \w^{1}]_{i_1} \sum_{i_2 = 1}^{d_2} \ldots \sum_{i_N = 1}^{d_N} [ \W ]_{i_1, \ldots, i_N}  \cdot \prod_{n' = 2}^N [ \w^{n'}]_{i_{n'}} \\
		& = \inprod{ \mat{ \W }_1 \cdot \kronp_{n' = 2}^N \w^{n'} }{ \w^{1} }
		\text{\,.}
	\end{split}
	\]
\end{proof}

\begin{lemma}
	\label{lem:outer_prod_distance_bound}
	For any $\{ \aaa^{n} \in \R^{d_n} \}_{n = 1}^N, \{ \bb^{n} \in \R^{d_n} \}_{n = 1}^N$, where $d_1, \ldots, d_N \in \N$, it holds that:
	\[
	\norm{ \tenp_{n = 1}^N \aaa^{n} - \tenp_{n = 1}^N \bb^{n} } \leq \sum_{n = 1}^N \norm{ \aaa^{n} - \bb^{n} } \cdot \prod_{n' \neq n} \max \left \{ \normnoflex{ \aaa^{n'} }, \normnoflex{ \bb^{n'} } \right \}
	\text{\,.}
	\]
\end{lemma}

\begin{proof}
	The proof is by induction over $N \in \N$.
	For $N = 1$, the claim is trivial.
	Assuming it holds for $N - 1 \geq 1$, we show that it holds for $N$ as well:
	\[
	\begin{split}
		\norm{ \tenp_{n = 1}^N \aaa^{n} - \tenp_{n = 1}^N \bb^{n} } & =  \norm{ \tenp_{n = 1}^N \aaa^{n} - \left ( \tenp_{n = 1}^{N - 1} \aaa^{n} \right ) \tenp \bb^{N} +  \left ( \tenp_{n = 1}^{N - 1} \aaa^{n} \right ) \tenp \bb^{N} - \tenp_{n = 1}^N \bb^{n} } \\
		& \leq \norm{ \aaa^{N} - \bb^{N} } \cdot \norm{ \tenp_{n = 1}^{N - 1} \aaa^{n} } + \norm{ \tenp_{n = 1}^{N - 1} \aaa^{n} - \tenp_{n = 1}^{N - 1} \bb^{n} } \cdot \norm{ \bb^{N} } \\
		& \leq  \norm{ \aaa^{N} - \bb^{N} } \cdot \prod_{n = 1}^{N - 1} \max \left \{ \norm{ \aaa^{n} }, \norm{ \bb^{n} } \right \} \\
		& \hspace{5mm} + \norm{ \tenp_{n = 1}^{N - 1} \aaa^{n} - \tenp_{n = 1}^{N - 1} \bb^{n} } \cdot \max \left \{ \norm{ \aaa^{N} }, \norm{ \bb^{N} } \right \}
		\text{\,.}
	\end{split}
	\]
	The proof concludes by the inductive assumption for $N - 1$.
\end{proof}

\begin{lemma}
	\label{lem:param_dist_to_end_to_end_dist}
	Let $B_{\norm{\cdot}}, B_{dist} > 0$ and $\{ \aaa_{r}^{n} \in \R^{d_n} \}_{r = 1}^R\hspace{0mm}_{n = 1}^N, \{ \bb_{r}^{n} \in \R^{d_n} \}_{r = 1}^R\hspace{0mm}_{n = 1}^N$, where $d_1, \ldots, d_N \in \N$, such that $\max \{ \normnoflex{ \aaa_r^{n} }, \normnoflex{ \bb_r^{n} } \}_{r = 1}^R\hspace{0mm}_{n = 1}^N \leq B_{\norm{\cdot}}$ and $( \sum_{r = 1}^R \sum_{n = 1}^N \normnoflex{ \aaa_r^{n} - \bb_r^{n} }^2 )^{1 / 2} \leq B_{dist}$.
	Then:
	\[
	\norm{ \sum_{r = 1}^R \tenp_{n = 1}^N \aaa_r^{n} - \sum_{r = 1}^R \tenp_{n = 1}^N \bb_r^{n} } \leq \sqrt{R N} B_{\norm{\cdot}}^{N - 1} B_{dist}
	\text{\,.}
	\]
\end{lemma}

\begin{proof}
	Applying the triangle inequality and Lemma~\ref{lem:outer_prod_distance_bound}, we have that:
	\[
	\begin{split}
		\norm{ \sum_{r = 1}^R \tenp_{n = 1}^N \aaa_r^{n} - \sum_{r = 1}^R \tenp_{n = 1}^N \bb_r^{n} } & \leq \sum_{r = 1}^R \norm{ \tenp_{n = 1}^N \aaa_r^{n} - \tenp_{n = 1}^N \bb_r^{n} } \\
		& \leq \sum_{r = 1}^R \sum_{n = 1}^N \norm{ \aaa_r^{n} - \bb_r^{n} } \cdot \prod_{n' \neq n} \max \left \{ \normnoflex{ \aaa_r^{n'} }, \normnoflex{ \bb_r^{n'} } \right \} \\
		& \leq B_{\norm{\cdot}}^{N - 1} \sum_{r = 1}^R \sum_{n = 1}^N \norm{ \aaa_r^{n} - \bb_r^{n} }
		\text{\,.}
	\end{split}
	\]
	The desired result readily follows from the fact that $\normnoflex{ \x }_1 \leq \sqrt{d} \cdot \normnoflex{ \x }$ for any $\x \in \R^d$:
	\[
	\begin{split}
		\norm{ \sum_{r = 1}^R \tenp_{n = 1}^N \aaa_r^{n} - \sum_{r = 1}^R \tenp_{n = 1}^N \bb_r^{n} } 
		& \leq B_{\norm{\cdot}}^{N - 1} \sum_{r = 1}^R \sum_{n = 1}^N \norm{ \aaa_r^{n} - \bb_r^{n} } \\
		& \leq B_{\norm{\cdot}}^{N - 1} \sqrt{R N} \left ( \sum_{r = 1}^R \sum_{n = 1}^N \norm{ \aaa_r^{n} - \bb_r^{n} }^2 \right )^{1 / 2} \\
		& \leq \sqrt{R N} B_{\norm{\cdot}}^{N - 1} B_{dist}
		\text{\,.}
	\end{split}
	\]
\end{proof}

\begin{lemma}
	\label{lem:ivp_no_sign_change}
	Let $f: [0, T_2) \to \R$ and $g: [0, T_1) \to \R$ be continuous functions, where $T_1 < T_2$.
	Suppose that $g(t)$ is bounded, $f(0) > 0$, and:
	\be
	\frac{d}{dt} f(t) = f(t)^p \cdot g(t) \quad, ~t \in [0, T_1)
	\text{\,,}
	\label{eq:ivp_no_sign_change_diff_eq}
	\ee
	for $1 < p \in \R$.
	Then, $f(t) > 0$ for all $t \in [0, T_1]$.
\end{lemma}

\begin{proof}
	Consider the initial value problem induced by Equation~\eqref{eq:ivp_no_sign_change_diff_eq} over the interval $[0, T_1)$, with an initial value of $f(0)$.
	One can verify by differentiation that it is solved by:
	\[
	h(t) = \left ( f (0)^{1 - p} - (p - 1) \int_{t' = 0}^t g(t') dt' \right )^{- \frac{1}{p - 1}}
	\text{\,.}
	\]
	Since the problem has a unique solution (see, \eg, Theorem~2.2 in~\citet{teschl2012ordinary}), it follows that for any $t \in [0, T_1)$:\footnote{
		A technical subtlety is that, in principle, $h( \cdot )$ may asymptote at some $\widebar{T}_1 \in [0, T_1)$.
		However, since the initial value problem has a unique solution, $f(t) = h(t)$ until that time.
		This means $h( \cdot )$ cannot asymptote before $T_1$ as that would contradict continuity of $f (\cdot)$ over $[0, T_2)$.
	}
	\[
		f(t) = h(t) = \left ( f (0)^{1 - p} - (p - 1) \int_{t' = 0}^t g(t') dt' \right )^{- \frac{1}{p - 1}} \geq \left ( f (0)^{1 - p} + (p - 1) \int_{t' = 0}^{t} \abs{ g(t') } dt'  \right )^{- \frac{1}{p - 1}} 
		\text{\,.}
	\]
	Recall that $g(t)$ is bounded.
	Hence, from the inequality above and continuity of $f(\cdot)$ we conclude:
	\[
	f(t) \geq \left ( f (0)^{1 - p} + (p - 1) \cdot \sup\nolimits_{t' \in [0, T_1)} \abs{g(t')} \cdot T_1 \right )^{- \frac{1}{p - 1}} > 0 ~~,~t \in [0, T_1]
	\text{\,.}
	\] 
\end{proof}

\begin{lemma}
	\label{lem:gf_smooth_dist_bound}
	Let $\theta, \theta' : [0, T] \to \R^d$, where $T > 0$, be two curves born from gradient flow over a continuously differentiable function $f: \R^{d} \to \R$:
	\beas
	\theta ( 0 ) = \theta_0 \in \R^d ~ &,& ~ \tfrac{d}{dt} \theta ( t ) = - \nabla f ( \theta ( t ) ) ~~ , ~ t \in [ 0 , T ] 
	\text{\,,}
	\\ [1mm]
	\theta' ( 0 ) = \theta' _0 \in \R^d ~ &,& ~ \tfrac{d}{dt} \theta' ( t ) = - \nabla f ( \theta' ( t ) ) ~~ , ~ t \in [ 0 , T ]
	\text{\,.}
	\eeas
	Let $D > 0$, and suppose that $f( \cdot )$ is $\beta$-smooth over $\D_{D + 1}$ for some $\beta \geq 0$,\footnote{
		That is, for any $\theta_1, \theta_2 \in \D_{D + 1}$ it holds that $\norm{ \nabla f (\theta_1) - \nabla f (\theta_2) } \leq \beta \cdot \norm{ \theta_1 - \theta_2}$.
	} where $\D_{D + 1} := \{ \theta \in \R^d : \normnoflex{ \theta } \leq D + 1 \}$.
	Then, if $\normnoflex{ \theta (0) - \theta' (0) } < \exp ( -\beta \cdot T )$, it holds that:
	\be
	\norm{ \theta (t) - \theta' (t) } \leq \norm{ \theta (0) - \theta' (0) } \cdot \exp \left ( \beta \cdot t \right )
	\text{\,}
	\label{eq:gf_smooth_dist_bound}
	\ee
	at least until $t \geq T$ or $\normnoflex{ \theta' (t) } \geq D$.
	That is, Equation~\eqref{eq:gf_smooth_dist_bound} holds for all $t \in [0, \min \{ T, T_D \}]$, where $T_D := \inf \{ t \geq 0 : \normnoflex{ \theta' (t) } \geq D \}$.

\end{lemma}

\begin{proof}
	If $\normnoflex{  \theta' (0) } \geq D$, the claim trivially holds.
	Suppose $\normnoflex{ \theta' (0) } < D$, and notice that in this case $\normnoflex{ \theta (0) } < \normnoflex{ \theta' (0) } + \exp ( - \beta \cdot T) < D + 1$.
	We examine the initial time at which $\normnoflex{  \theta' (t) } \geq D$ or $\normnoflex{ \theta (t) } \geq D + 1$.
	That is, let $\widebar{T}_{D } := \inf \left \{ t \in [0, T ] :  \normnoflex{  \theta' (t) } \geq D \text{ or }  \normnoflex{ \theta (t) } \geq D + 1 \right \}$, where we take $\widebar{T}_{D} := T$ if the set is empty.
	Since both $\norm{  \theta' (t) }$ and $\norm{ \theta (t) }$ are continuous in $t$, it must be that $\widebar{T} > 0$.
	Furthermore, $\norm{  \theta' (t) } \leq D$ and $\norm{ \theta (t) } \leq D + 1$ for all $t \in [0, \widebar{T}_{D}]$.
	
	Now, define the function $g: [0, T] \to \R_{\geq 0}$ by $g(t) := \norm{ \theta(t) -  \theta' (t) }^2$.
	For any $t \in [0, \widebar{T}_{D}]$ it holds that:
	\[
	\begin{split}
		\frac{d}{dt} g(t) & = 2 \inprod{ \theta (t) -  \theta' (t) }{ \tfrac{d}{dt} \theta (t) - \tfrac{d}{dt}  \theta' (t) } \\
		& = -2 \inprod{ \theta (t) -  \theta' (t) }{ \nabla f ( \theta (t) ) - \nabla f ( \theta' (t) )  }
		\text{\,.}
	\end{split}
	\]
	By the Cauchy-Schwartz inequality and $\beta$-smoothness of $f (\cdot)$ over $\D_{D + 1}$ we have:
	\be
	\begin{split}
		\frac{d}{dt} g(t) & \leq 2 \beta \cdot \norm{ \theta (t) -  \theta' (t) }^2 = 2 \beta \cdot g(t) 
		\text{\,.}
	\end{split}
	\label{eq:g_time_deriv_smoothness_bound}
	\ee
	Thus, Gronwall's inequality leads to $g(t) \leq g(0) \cdot \exp ( 2 \beta \cdot t )$.
	Taking the square root of both sides then establishes Equation~\eqref{eq:gf_smooth_dist_bound} for all $t \in [0 , \widebar{T}_{D}]$.

	If $\widebar{T}_D = T$, the proof concludes since Equation~\eqref{eq:gf_smooth_dist_bound} holds over $[0, T]$.
	Otherwise, if $\widebar{T}_D < T$, then either $\normnoflex{ \theta' (\widebar{T}_D ) } = D$ or $\normnoflex{ \theta ( \widebar{T}_D ) } = D + 1$.
	It suffices to show that in both cases $T_D \leq \widebar{T}_D$.
	In case $\normnoflex{  \theta' ( \widebar{T}_D ) } = D$, the definition of $T_D$ implies $T_D = \widebar{T}_D$.
	On the other hand, suppose $\normnoflex{ \theta ( \widebar{T}_D ) } = D + 1$.
	Since $\norm{ \theta (0) -  \theta' (0) } < \exp ( -\beta \cdot T )$, the fact that Equation~\eqref{eq:gf_smooth_dist_bound} holds for $\widebar{T}_D$ gives $\norm{ \theta (\widebar{T}_D) -  \theta' (\widebar{T}_D) } \leq 1$.
	Therefore, it must be that $\norm{  \theta' ( \widebar{T}_D ) } \geq D$, and so $T_D \leq \widebar{T}_D$, completing the proof.
\end{proof}

\subsubsection{Tensor Factorization}
\label{app:proofs:useful_lemmas:tf}

Suppose that we minimize the objective $\phi(\cdot)$ (Equations~\eqref{eq:cp_objective} and~\eqref{eq:end_tensor}) via gradient flow over an $R$-component tensor factorization (Equation~\eqref{eq:cp_gf}), where we allow the loss $\L (\cdot)$ in Equation~\eqref{eq:cp_objective} to be any differentiable and locally smooth function.
Under this setting, the following lemmas establish several results which will be of use when proving the main theorems.

\begin{lemma}
	\label{lem:cp_gradient}
	For any $\{ \w_{r}^{n}  \in \R^{d_n} \}_{r = 1}^R\hspace{0mm}_{n = 1}^N$:
	\[
	\frac{\partial}{\partial \w_r^{n}} \phi \left (\{ \w_{r'}^{n'} \}_{r' = 1}^R\hspace{0mm}_{n' = 1}^N \right ) = \mat{ \nabla \L \left ( \W_e \right ) }_{n} \cdot \kronp_{n' \neq n} \w_r^{n'} \quad , ~r = 1, \ldots, R ~ , ~ n = 1, \ldots, N
	\text{\,,}
	\]
	where $\W_e$ denotes the end tensor (Equation~\eqref{eq:end_tensor}) induced by $\{ \w_{r}^{n} \}_{r = 1}^R\hspace{0mm}_{n = 1}^N$.
\end{lemma}

\begin{proof}
	For $r \in [R], n \in [N]$, we treat $\{ \w_{r'}^{n'} \}_{(r', n') \neq (r, n)}$ as fixed, and with slight abuse of notation consider:
	\[
	\phi_{r, n} \left ( \w_r^{n} \right ) := \phi \left (\{ \w_{r'}^{n'} \}_{r' = 1}^R\hspace{0mm}_{n' = 1}^N \right )
	\text{\,.}
	\]
	For $\Delta \in \R^{d_n}$, from the first order Taylor approximation of $\L (\cdot)$ we have that:
	\[
	\begin{split}
		\phi_{r, n} \left ( \w_r^{n} + \Delta \right ) & = \L \left ( \W_e + \left ( \tenp_{n' = 1}^{n - 1} \w_r^{n'} \right ) \tenp \Delta \tenp \left ( \tenp_{n' = n + 1}^N \w_r^{n'} \right ) \right ) \\
		& = \L \left (  \W_e \right ) + \inprod{ \nabla \L \left ( \W_e \right ) }{ \left ( \tenp_{n' = 1}^{n - 1} \w_r^{n'} \right ) \tenp \Delta \tenp \left ( \tenp_{n' = n + 1}^N \w_r^{n'} \right )  } + o \left ( \norm{ \Delta } \right )
		\text{\,.}
	\end{split}
	\]
	Since $\L \left ( \W_e \right ) = \phi_{r, n} ( \w_r^{n} )$, by applying Lemma~\ref{lem:inp_with_tenp_to_mat_kronp} we arrive at:
	\[
	\phi_{r, n} \left ( \w_r^{n} + \Delta \right ) = \phi_{r, n} \left ( \w_r^{n} \right ) + \inprod{ \mat{ \nabla \L ( \W_e ) }_{n} \cdot \kronp_{n' \neq n} \w_r^{n'} }{ \Delta } + o(\norm{\Delta})
	\text{\,.}
	\]
	Uniqueness of the linear approximation of $\phi_{r, n} ( \cdot )$ at $\w_r^{n}$ then implies:
	\[
	\frac{\partial}{\partial \w_r^{n}} \phi \left ( \{ \w_{r'}^{n'} \}_{r' = 1}^R\hspace{0mm}_{n' = 1}^N \right ) = \frac{d}{d \w_r^{n}} \phi_{r, n} \left ( \w_r^{n} \right ) = \mat{ \nabla \L \left ( \W_e \right ) }_{n} \cdot \kronp_{n' \neq n} \w_r^{n'}
	\text{\,.}
	\]
\end{proof}

\begin{lemma}
	\label{lem:dyn_parameter_vector_sq_norm}
	For any $r \in [R]$ and $n \in [N]$:
	\[
	\frac{d}{dt}  \normnoflex{ \w_r^{n} (t) }^2 = - 2 \inprod{ \nabla \L \left ( \W_e (t) \right ) }{ \tenp_{n' = 1}^N \w_r^{n'} (t) } 
	\text{\,.}
	\]
\end{lemma}

\begin{proof}
	Fix $r \in [R]$ and $n \in [N]$.
	Differentiating $\normnoflex{ \w_r^{n} (t) }^2$ with respect to time, we have:
	\[
	\frac{d}{dt} \normnoflex{ \w_r^{n} (t) }^2 = 2 \inprod{  \w_r^{n} (t) }{ \tfrac{d}{dt} \w_r^{n} (t) } = - 2 \inprod{  \w_r^{n} (t) }{ \frac{\partial}{\partial \w_r^{n}} \phi \left (\{ \w_{r'}^{n'} (t) \}_{r' = 1}^R\hspace{0mm}_{n' = 1}^N \right ) }
	\text{\,.}
	\]
	Applying Lemmas~\ref{lem:cp_gradient} and~\ref{lem:inp_with_tenp_to_mat_kronp} completes the proof.
\end{proof}

\begin{lemma}[Lemma~\ref{lem:balancedness_conservation_body} restated]
	\label{lem:balancedness_conservation}
	For all $r \in [R]$ and $n, \bar{n} \in [N]$:
	\[
	\norm{ \w_r^{n} (t) }^2 - \norm{ \w_r^{\bar{n}} (t) }^2 = \norm{ \w_r^{n} (0) }^2 - \norm{ \w_r^{\bar{n}} (0) }^2 \quad , ~t \geq 0
	\text{\,.}
	\]
\end{lemma}

\begin{proof}[Proof of Lemma~\ref{lem:balancedness_conservation}]
	For any $r \in [R]$ and $n, \bar{n} \in [N]$, by Lemma~\ref{lem:dyn_parameter_vector_sq_norm} it holds that:
	\[
	\frac{d}{dt}  \normnoflex{ \w_r^{n} (t) }^2 = - 2 \inprod{ \nabla \L \left ( \W_e (t) \right ) }{ \tenp_{n' = 1}^N \w_r^{n'} (t) }  = \frac{d}{dt}  \normnoflex{ \w_r^{\bar{n}} (t) }^2
	\text{\,.}
	\]
	Integrating both sides with respect to time gives:
	\[
	\norm{ \w_r^{n} (t) }^2 - \norm{ \w_r^{n} (0) }^2 = \norm{ \w_r^{\bar{n}} (t) }^2 - \norm{ \w_r^{\bar{n}} (0) }^2
	\text{\,.}
	\]
	Rearranging the equality above establishes the desired result.
\end{proof}

\begin{lemma}
	\label{lem:width_R_equivalent_to_larger_width_with_zero_init}
	Let $\widetilde{R} > R$, and define:
	\be
	\widetilde{ \w }_r^n (t) := \begin{cases}
		\w_r^n (t)	& , r \in \{ 1, \ldots, R \} \\
		0 \in \R^{d_n}	& , r \in \{ R + 1, \ldots, \widetilde{R} \}
	\end{cases}
	\quad ,~t \geq 0 ~,~n = 1, \ldots, N
	\text{\,.}
	\label{eq:width_R_equivalent_to_larger_width_with_zero_init}
	\ee
	Then, $\{ \widetilde{ \w }_r^n (t) \}_{r = 1}^{\widetilde{R}}\hspace{0mm}_{n = 1}^N$ follow a gradient flow path of an $\widetilde{R}$-component factorization.
\end{lemma}

\begin{proof}
	We verify that $\{ \widetilde{ \w }_r^n (t) \}_{r = 1}^{\widetilde{R}}\hspace{0mm}_{n = 1}^N$ satisfy the differential equations governing gradient flow.
	Fix $n \in [N]$.
	For any $r \in [R]$ and $t \geq 0$ we have:
	\[
	\begin{split}
		\frac{d}{dt} \widetilde{\w}_{r}^n (t) & = \frac{d}{dt} \w_{r}^n (t) = - \frac{\partial}{ \partial \w_{r}^{n} } \phi \left (\{ \w_{r'}^{n'} (t) \}_{r' = 1}^{ R }\hspace{0mm}_{n' = 1}^N \right )
		\text{\,.}
	\end{split}
	\]
	Noticing that $\W_e (t) =\sum_{r' = 1}^{R} \tenp_{n' = 1}^N \w_{r'}^{n'} (t) = \sum_{r' = 1}^{\widetilde{R}} \tenp_{n' = 1}^N \widetilde{\w}_{r'}^{n'} (t) = \widetilde{ \W }_e (t)$, and invoking Lemma~\ref{lem:cp_gradient}, we may write:
	\[
	\begin{split}
		\frac{d}{dt} \widetilde{\w}_{r}^n (t) & = - \matflex{ \nabla \L \left ( \W_e (t) \right ) }_{n} \cdot \kronp_{n' \neq n} \w_{r}^{n'} (t) \\
		& = - \matflex{ \nabla \L \left ( \widetilde{\W}_e (t) \right ) }_{n} \cdot \kronp_{n' \neq n} \widetilde{\w}_{r}^{n'} (t)  \\
		& = - \frac{\partial}{ \partial \widetilde{\w}_{r}^{n} } \phi \left (\{ \widetilde{\w}_{r'}^{n'} (t) \}_{r' = 1}^{ \widetilde{R} }\hspace{0mm}_{n' = 1}^N \right )
		\text{\,.}
	\end{split}
	\]
	On the other hand, for any $r \in \{ R + 1, \ldots, \widetilde{R} \}$, recalling that $\widetilde{\w}_r^n (t)$ is identically zero:
	\[
	\begin{split}
		\frac{d}{dt} \widetilde{\w}_{r}^n (t) = 0 = - \matflex{ \nabla \L \left ( \widetilde{\W}_e (t) \right ) }_{n} \cdot \kronp_{n' \neq n} \widetilde{\w}_{r}^{n'} (t) = - \frac{\partial}{ \partial \widetilde{\w}_{r}^{n} } \phi \left (\{ \widetilde{\w}_{r'}^{n'} (t) \}_{r' = 1}^{ \widetilde{R} }\hspace{0mm}_{n' = 1}^N \right )
		\text{\,,}
	\end{split}
	\]
	for all $t \geq 0$, completing the proof.
\end{proof}

\begin{lemma}
	\label{lem:balanced_param_vector_norm_no_sign_change}
	For any $r \in [R]$:
	\begin{itemize}
		\item If $\normnoflex{ \w_r^{1} (0) } = \cdots = \normnoflex{ \w_r^{N} (0) } = 0$, then:
		\be
		\norm{ \w_r^{1} (t) } = \cdots = \norm{ \w_r^{N} (t) } = 0 \quad ,~t \geq 0
		\text{\,.}
		\label{eq:balanced_param_vector_norm_no_sign_change_stay_zero}
		\ee
		\item On the other hand, if $\normnoflex{ \w_r^{1} (0) } = \cdots = \normnoflex{ \w_r^{N} (0) } > 0$, then:
		\be
		\norm{ \w_r^{1} (t) } = \cdots = \norm{ \w_r^{N} (t) } > 0 \quad,~t \geq 0
		\text{\,.}
		\label{eq:balanced_param_vector_norm_no_sign_change_stay_nonzero}
		\ee
	\end{itemize}
\end{lemma}

\begin{proof}
	The proof is divided into two separate parts, establishing Equations~\eqref{eq:balanced_param_vector_norm_no_sign_change_stay_zero} and~\eqref{eq:balanced_param_vector_norm_no_sign_change_stay_nonzero} under their respective conditions.
	
	\paragraph*{Proof of Equation~\eqref{eq:balanced_param_vector_norm_no_sign_change_stay_zero}  (if $\normnoflex{ \w_r^{1} (0) } = \cdots = \normnoflex{ \w_r^{N} (0) } = 0$):}
	
	To simplify presentation, we assume without loss of generality that $r = R$.
	Consider the following initial value problem induced by gradient flow over $\phi (\cdot)$:
	\be
	\begin{split}
		& \widetilde{\w}_{\bar{r}}^{n} ( 0 ) = \w_{\bar{r}}^{n} ( 0 ) \quad ,  ~ \bar{r} = 1, \ldots, R ~,~ n = 1, \ldots, N \text{\,,}  \\
		& \frac{d}{dt} \widetilde{\w}_{\bar{r}}^{n} (t) = - \frac{\partial}{\partial \widetilde{\w}_{\bar{r}}^{n}} \phi \left (\{ \widetilde{\w}_{r'}^{n'} (t) \}_{r' = 1}^{R}\hspace{0mm}_{n' = 1}^N \right ) 	\quad , ~t \geq 0 ~,~ \bar{r} = 1, \ldots, R ~,~ n = 1, \ldots, N  \text{\,.}
	\end{split}
	\label{eq:gf_w_tilde_ivp}
	\ee
	By definition, $\{\w_{\bar{r}}^{n} (t) \}_{\bar{r} = 1}^{R}\hspace{0mm}_{n = 1}^N$ is a solution to the initial value problem above.
	Since it has a unique solution (see, \eg, Theorem~2.2 in~\citet{teschl2012ordinary}), we need only show that there exist $\{ \widetilde{\w}_{\bar{r}}^{n} (t) \}_{\bar{r} = 1}^{R}\hspace{0mm}_{n = 1}^N$ satisfying Equation~\eqref{eq:gf_w_tilde_ivp} such that $\widetilde{\w}_R^{1} (t) = \cdots = \widetilde{\w}_R^{N} (t) = 0$ for all $t \geq 0$.
	
	If $R = 1$, \ie~the factorization consists of a single component, by Lemma~\ref{lem:cp_gradient}:
	\[
	- \frac{\partial}{\partial \widetilde{\w}_1^{n}} \phi \left (\{ \widetilde{\w}_{1}^{n'} \}_{n' = 1}^N \right ) = - \matflex{ \nabla \L \left ( \tenp_{n' = 1}^{N} \widetilde{\w}_1^{n'} \right ) }_{n} \cdot \kronp_{n' \neq n} \widetilde{\w}_1^{n'} \quad , ~n = 1 ,\ldots, N
	\text{\,,}
	\]
	for any $\widetilde{\w}_1^{1} \in \R^{d_1}, \ldots, \widetilde{\w}_1^{N} \in \R^{d_N}$.
	Hence, $\widetilde{\w}_1^{1} (t) = \cdots = \widetilde{\w}_1^{N} (t) = 0$ for all $t \geq 0$ form a solution to the initial value problem in Equation~\eqref{eq:gf_w_tilde_ivp}.
	To see it is so, notice that the initial conditions are met, and:
	\[
	\frac{d}{dt} \widetilde{\w}_1^{n} (t) = 0 = - \frac{\partial}{\partial \widetilde{\w}_1^{n}} \phi \left (\{ \widetilde{\w}_{1}^{n'} (t) \}_{n' = 1}^N \right ) \quad , ~t \geq 0 ~,~ n = 1 ,\ldots, N
	\text{\,.}
	\]
	
	If $R > 1$, with slight abuse of notation we denote by $\phi ( \{ \widetilde{\w}_{\bar{r}}^{n} \}_{\bar{r} = 1}^{R - 1}\hspace{0mm}_{n = 1}^N ) := \L ( \sum_{\bar{r} = 1}^{R - 1} \tenp_{n = 1}^N \widetilde{\w}_{\bar{r}}^{n} )$ the objective over an $(R - 1)$-component tensor factorization.
	Let $\{ \widetilde{\w}_{\bar{r}}^{n} (t) \}_{\bar{r} = 1}^{R - 1}\hspace{0mm}_{n = 1}^N$ be curves obtained by running gradient flow on this objective, initialized such that:
	\[
	\widetilde{\w}_{\bar{r}}^{n} (0) := \w_{\bar{r}}^{n} (0) \quad , ~\bar{r} = 1, \ldots, R - 1 ~,~ n = 1, \ldots, N
	\text{\,.}
	\]
	Additionally, define $\widetilde{\w}_{R}^{1} (t) = \cdots = \widetilde{\w}_{R}^{N} (t) = 0$ for all $t \geq 0$. 
	According to Lemma~\ref{lem:width_R_equivalent_to_larger_width_with_zero_init},  $\{ \widetilde{\w}_{\bar{r}}^{n} (t) \}_{\bar{r} = 1}^{R}\hspace{0mm}_{n = 1}^N$ form a valid solution to the original gradient flow over an $R$-component factorization, \ie~satisfy Equation~\eqref{eq:gf_w_tilde_ivp}.
	Thus, uniqueness of the solution implies $\w_R^{1} (t) = \cdots = \w_R^{N} (t) = 0$ for all $t \geq 0$, completing the proof for Equation~\eqref{eq:balanced_param_vector_norm_no_sign_change_stay_zero}.

	\paragraph*{Proof of Equation~\eqref{eq:balanced_param_vector_norm_no_sign_change_stay_nonzero}  (if $\normnoflex{ \w_r^{1} (0) } = \cdots = \normnoflex{ \w_r^{N} (0) } > 0$):}
	
	From Lemma~\ref{lem:balancedness_conservation_body} it follows that $\normnoflex{ \w_r^{1} ( t ) } = \cdots = \normnoflex{ \w_r^{N} ( t ) }$ for any $t \geq 0$.
	Hence, it suffices to show that $\normnoflex{ \w_r^{1} (t) }$ stays positive.
	Assume by way of contradiction that there exists $\bar{t} > 0$ for which $\normnoflex{\w_r^{1} (\bar{t}\,)} = 0$.
	Define:
	\[
	t_0 := \inf \left \{ t \geq 0: \normnoflex{ \w_r^{1} ( t ) } = 0 \right \}
	\text{\,,}
	\]
	the initial time at which $\normnoflex{  \w_r^{1} (t) }$ meets zero.
	Due to the fact that $\normnoflex{ \w_r^{1} (t) }$ is continuous in $t$, $\normnoflex{ \w_r^{1} (t_0) } = 0$ and $t_0 > 0$.
	Furthermore, $\normnoflex{ \w_r^{1} (t) } > 0$ for all $t \in [0, t_0)$.
	We may therefore differentiate $\normnoflex{\w_r^{1} (t) }$ with respect to time over the interval $[0, t_0)$ as follows:
	\[
	\begin{split}
		\frac{d}{dt} \norm{ \w_r^{1} (t) } & =  \left ( \tfrac{d}{dt} \norm{ \w_r^{1} (t) }^2 \right ) \cdot 2^{-1}  \norm{ \w_r^{1} (t) }^{-1} \\
		& = \norm{ \w_r^{1} (t) }^{-1}  \inprod{ - \nabla \L \left ( \W_e (t) \right ) }{ \tenp_{n = 1}^N \w_r^{n} (t) } \\
		& = \norm{ \w_r^{1} (t) }^{N - 1}  \inprod{ - \nabla \L \left ( \W_e (t) \right ) }{ \tenp_{n = 1}^N \widehat{\w}_r^{n} (t) }
		\text{\,,}
	\end{split}
	\]
	where in the second transition we made use of Lemma~\ref{lem:dyn_parameter_vector_sq_norm}, and $\widehat{\w}_r^{n} (t) := \w_r^{n} (t) / \normnoflex{ \w_r^{n} (t) }$ for $n = 1, \ldots, N$.
	Define $g(t) := \inprodnoflex{ - \nabla \L \left ( \W_e (t) \right ) }{ \tenp_{n = 1}^N \widehat{\w}_r^{n} (t) }$.
	Since $\nabla \L ( \W_e (t) )$ is continuous with respect to time, $g(t)$ is bounded over $[0, t_0]$ and continuous over $[0, t_0)$.
	Thus, invoking Lemma~\ref{lem:ivp_no_sign_change} with $g(t)$, $T_1 := t_0$ and $f(t) := \normnoflex{ \w_r^1 (t) }$, we get that $\normnoflex{ \w_r^{1} (t) } > 0$ for all $t \in [0, t_0]$, in contradiction to $\normnoflex{ \w_r^{1} (t_0) } = 0$.
	This means that $\normnoflex{ \w_r^{1} (t) } > 0$ for all $t \geq 0$, concluding the proof for Equation~\eqref{eq:balanced_param_vector_norm_no_sign_change_stay_nonzero}.
\end{proof}

\subsection{Proof of Theorem~\ref{thm:dyn_fac_comp_norm_unbal}}
\label{app:proofs:dyn_fac_comp_norm_unbal}

Fix $r \in [R]$ and $t \geq 0$.
Since $\normnoflex{ \tenp_{n = 1}^N \w_r^{n} (t) } = \prod_{n = 1}^N \normnoflex { \w_r^{n} (t) }$, the product rule gives:
\[
\frac{d}{dt} \norm{ \tenp_{n = 1}^N \w_r^{n} (t) } = \sum_{n = 1}^N \frac{d}{dt} \norm{ \w_r^{n} (t) } \cdot \prod_{n' \neq n} \normnoflex{ \w_r^{n'} (t) }
\text{\,.}
\]
Notice that for any $n \in [N]$ we have $\normnoflex{ \w_r^{n} (t) } > 0$, as otherwise $\normnoflex{ \tenp_{n' = 1}^N \w_r^{n'} (t) }$ must be zero.
Thus, applying Lemma~\ref{lem:dyn_parameter_vector_sq_norm} we get $\frac{d}{dt} \normnoflex{ \w_r^{n} (t) } = \frac{1}{2} \normnoflex{ \w_r^{n} (t) }^{-1} \frac{d}{dt} \normnoflex{ \w_r^n (t) }^2 = \normnoflex{ \w_r^{n} (t) }^{-1} \inprodnoflex{ - \nabla \L \left ( \W_e (t) \right ) }{ \tenp_{n' = 1}^N \w_r^{n'} (t) }$.
Combined with the equation above, we arrive at:
\be
\begin{split}
	\frac{d}{dt} \norm{ \tenp_{n = 1}^N \w_r^{n} (t) }  & = \sum_{n = 1}^N \norm{ \w_r^{n} (t) }^{-1} \inprod{ - \nabla \L \left ( \W_e (t) \right ) }{ \tenp_{n' = 1}^N \w_r^{n'} (t) }   \cdot \prod_{n' \neq n} \normnoflex{ \w_r^{n'} (t) } \\
	& =  \inprod{ - \nabla \L \left ( \W_e (t) \right ) }{ \tenp_{n' = 1}^N \widehat{\w}_r^{n'} (t) }  \cdot \sum_{n = 1}^N \prod_{n' \neq n} \normnoflex{ \w_r^{n'} (t) }^2 
	\text{\,.}
\end{split}
\label{eq:unbal_comp_time_derive_intermid}
\ee
By Lemma~\ref{lem:balancedness_conservation_body}, the differences between squared norms of vectors in the same component are constant through time.
In particular, the unbalancedness magnitude (Definition~\ref{def:unbalancedness_magnitude}) is conserved during gradient flow, implying that for any $n \in [N]$:
\be
\norm{ \w_r^n (t) }^2 \leq \min_{n' \in [N]} \normnoflex{ \w_r^{n'} (t) }^2 + \epsilon \leq \norm{ \tenp_{n' = 1}^N \w_r^{n'} (t) }^{\frac{2}{N}} + \epsilon
\text{\,.}
\label{eq:sq_norm_unbal_upper_bound}
\ee
Now, suppose that $\gamma_r (t) := \inprodnoflex{ - \nabla \L ( \W_e (t) ) }{ \tenp_{n = 1}^N \widehat{\w}_r^{n} (t) } \geq 0$. 
Going back to Equation~\eqref{eq:unbal_comp_time_derive_intermid}, applying the inequality in Equation~\eqref{eq:sq_norm_unbal_upper_bound} for each $\normnoflex{ \w_r^{n'} (t) }^2$ yields the desired upper bound from Equation~\eqref{eq:dyn_fac_comp_norm_unbal_pos}.
On the other hand, multiplying and dividing each summand in Equation~\eqref{eq:unbal_comp_time_derive_intermid} by the corresponding $\normnoflex{ \w_r^n (t) }^2$, we may equivalently write:
\[
\begin{split}
	\frac{d}{dt} \norm{ \tenp_{n = 1}^N \w_r^{n} (t) } & =  \inprod{ - \nabla \L \left ( \W_e (t) \right ) }{ \tenp_{n' = 1}^N \widehat{\w}_r^{n'} (t) }  \cdot \sum_{n = 1}^N \normnoflex{ \w_r^n (t) }^{-2} \prod_{n' = 1}^N \normnoflex{ \w_r^{n'} (t) }^2 \\
	& = \inprod{ - \nabla \L \left ( \W_e (t) \right ) }{ \tenp_{n' = 1}^N \widehat{\w}_r^{n'} (t) } \norm{ \tenp_{n = 1}^N \w_r^n (t) }^2 \cdot \sum_{n = 1}^N \normnoflex{ \w_r^n (t) }^{-2}
	\text{\,.}
\end{split}
\]
Noticing that Equation~\eqref{eq:sq_norm_unbal_upper_bound} implies $\normnoflex{ \w_r^{n} (t) }^{-2} \geq ( \normnoflex{ \tenp_{n' = 1}^N \w_r^{n'} (t) }^{\frac{2}{N}} + \epsilon )^{-1}$, the lower bound from Equation~\eqref{eq:dyn_fac_comp_norm_unbal_pos} readily follows.

If $\gamma_r (t) < 0$, Equation~\eqref{eq:dyn_fac_comp_norm_unbal_neg} is established by following the same computations, up to differences in the direction of inequalities due to the negativity of $\gamma_r (t)$.
\qed

\subsection{Proof of Corollary~\ref{cor:dyn_fac_comp_norm_balanced}}
\label{app:proofs:dyn_fac_comp_norm_balanced}

Fix $r \in [R]$ and $t \geq 0$.
The lower and upper bounds in Theorem~\ref{thm:dyn_fac_comp_norm_unbal} are equal to $\smash{ N \gamma_r (t) \cdot \norm{ \tenp_{n = 1}^N \w_r^{n} (t) }^{2 - 2 / N} }$ for unbalancedness magnitude $\epsilon = 0$.
Therefore, if $\normnoflex{ \tenp_{n = 1}^N \w_r^{n} (t) } > 0$, Equation~\eqref{eq:dyn_fac_comp_norm} immediately follows from Theorem~\ref{thm:dyn_fac_comp_norm_unbal}.

If $\normnoflex{ \tenp_{n = 1}^N \w_r^{n} (t) } = 0$, we claim that necessarily $\normnoflex{ \tenp_{n = 1}^N \w_r^{n} (t') }  = 0$ for all $t' \geq 0$, in which case both sides of Equation~\eqref{eq:dyn_fac_comp_norm} are zero.
Indeed, since the unbalancedness magnitude is zero at initialization and $\normnoflex{ \tenp_{n = 1}^N \w_r^{n} (t) } = \prod_{n = 1}^N \normnoflex { \w_r^{n} (t) }$, by Lemma~\ref{lem:balanced_param_vector_norm_no_sign_change} we know that either $\normnoflex{ \tenp_{n = 1}^N \w_r^{n} (t') } = 0$  for all $t' \geq 0$, or $\normnoflex{ \tenp_{n = 1}^N \w_r^{n} (t') } > 0$ for all $t' \geq 0$.
Hence, given that $\normnoflex{ \tenp_{n = 1}^N \w_r^{n} (t) } = 0$, the norm of the component must be identically zero through time.
\qed

\subsection{Proof of Theorem~\ref{thm:approx_rank_1}}
\label{app:proofs:approx_rank_1}

For conciseness, we consider the case where the number of components $R \geq 2$.
For $R = 1$, existence of a time $T_0 > 0$ at which $\W_e (T_0) \in \S$ follows by analogous steps, disregarding parts pertaining to factorization components $2, \ldots, R$.
Furthermore, proximity to a balanced rank one trajectory becomes trivial as, by Assumption~\ref{assump:a_balance} and Lemma~\ref{lem:balancedness_conservation_body}, $\W_e (t)$ is in itself such a trajectory.

Assume without loss of generality that Assumption~\ref{assump:a_lead_comp} holds for $\bar{r} = 1$.

Before delving into the proof details, let us introduce some notation and specify the exact requirement on the initialization scale $\alpha$.
We let $\L_h : \R^{d_1, \ldots, d_N} \to \R_{\geq 0}$ be the tensor completion objective induced by the Huber loss (Equation~\eqref{eq:tc_loss} with $\ell_h ( \cdot )$ in place of $\ell (\cdot)$), and $\phi_h ( \cdot )$ be the corresponding tensor factorization objective (Equation~\eqref{eq:cp_objective} with $\L_h (\cdot)$ in place of $\L (\cdot)$). 
For reference sphere radius $\rho \in ( 0 , \min_{(i_1, \ldots, i_N) \in \Omega} \abs{ y_{i_1, \ldots, i_N} } - \delta_h )$, distance from origin $D > 0$, time duration $T > 0$, and degree of approximation $\epsilon \in (0, 1)$, let:
\be
\begin{split}
& \normnoflex{ \aaa_r } := \normnoflex{ \aaa_r^1} = \cdots = \normnoflex{ \aaa_r^N } \quad , ~r = 1, \ldots, R
\text{\,,} \\
&  A := \max\nolimits_{r \in [R]} \norm{ \aaa_r } \text{\,,} \\ 
& A_{-1} := \max\nolimits_{r \in \{ 2, \ldots, R\}} \norm{ \aaa_r } \text{\,,} \\
& \widetilde{D} := \sqrt{N} \left ( \max \{ D, \rho \} + 1 \right )^{\frac{1}{N} } \text{\,,} \\
& \beta := R N \left ( (\widetilde{D} + 1)^{2(N - 1)} + \delta_h (\widetilde{D} + 1)^{N - 2} \right ) \text{\,,} \\
& \hat{\epsilon} <  \min \left \{ 2^{- \frac{N}{2}} R^{-N} N^{-N} (\widetilde{D} + 1)^{N - N^2} \cdot \exp ( -N \beta T) \cdot \epsilon^N ~,~ \rho (R - 1)^{-1} \right \} \text{\,,} \\
& \tilde{\epsilon} := \min \left \{ \hat{\epsilon} ~,~ ( R - 1 )^{-1} \left ( \rho - \left [ \rho^{\frac{1}{N}} - (R - 1)^{\frac{1}{N}} \cdot \hat{\epsilon}^{\frac{1}{N}} \right ]^{N} \right ) \right \}
\text{\,.}
\end{split}
\label{eq:approx_rank_1_consts}
\ee
With the constants above in place, for the results of the theorem to hold it suffices to require that:
\be
\alpha < \min \left \{  R^{- \frac{1}{N}} A^{-1} \rho^{\frac{1}{N}} ~,~ \left ( A_{-1}^{2 - N} - \norm{ \aaa_1 }^{2 - N}  \tfrac{ \norm{ \nabla \L (0) } }{ \inprod{ -\nabla \L (0) }{ \tenp_{n = 1}^N \widehat{\aaa}_{1}^n } } \right )^{\frac{1}{N - 2}} \cdot \tilde{\epsilon}^{\frac{1}{N}} \right \}
\text{\,.}
\label{eq:alpha_size_requirement}
\ee

\medskip

The proof is sectioned into three parts.
We begin with several preliminary lemmas in Subappendix~\ref{app:proofs:approx_rank_1:prelim_lemmas}.
Then, Subappendix~\ref{app:proofs:approx_rank_1:const_grad} establishes the existence of a time $T_0 > 0$ at which $ \W_e ( t )$ initially reaches the reference sphere $\S$, \ie~$\normnoflex{ \W_e (T_0) } = \rho$, while $ \normnoflex{ \tenp_{n = 1}^N \w_2^n (T_0) }, \ldots, \normnoflex{ \tenp_{n = 1}^N \w_R^n (T_0) }$ are still $\OO ( \alpha^N )$.
Consequently, as shown in Subappendix~\ref{app:proofs:approx_rank_1:trajectory}, at that time the weight vectors of the $R$-component tensor factorization are close to weight vectors corresponding to a balanced rank one trajectory emanating from $\S$, denoted $\W_1 (t)$.
The proof concludes by showing that this implies the time-shifted trajectory $\overline{\W}_e ( t )$ is within $\epsilon$ distance from $\W_1 (t)$ at least until $t \geq T$ or $\normnoflex{ \overline{\W}_e ( t ) } \geq D$.

\subsubsection{Preliminary Lemmas}
\label{app:proofs:approx_rank_1:prelim_lemmas}

\begin{lemma}
	\label{lem:huber_loss_const_grad_near_zero}
	Let $\W \in \R^{d_1 ,  \ldots , d_N}$ be such that $\norm{ \W } \leq \rho$, where $\rho \in (0, \min_{(i_1, \ldots, i_N) \in \Omega} \abs{ y_{i_1, \ldots, i_N} } - \delta_h )$.
	Then:
	\[
	\nabla \L_h ( \W ) = \frac{\delta_h}{ \abs{\Omega} } \sum\nolimits_{(i_1, \ldots, i_N) \in \Omega} \sign( - y_{i_1, \ldots, i_N} ) \cdot \mathcal{E}_{i_1,...,i_N}
	\text{\,,}
	\]
	where $\mathcal{E}_{i_1,...,i_N} \in \R^{d_1, \ldots, d_N}$ holds $1$ in its $(i_1, \ldots, i_N)$'th entry and $0$ elsewhere.
\end{lemma}

\begin{proof}
	Fix $I := (i_1, \ldots, i_N) \in \Omega$, and let $\ell_h' (\cdot)$ denote the derivative of $\ell_h ( \cdot )$.
	If $y_I > 0$, we have that $[ \W ]_I - y_I \leq \norm{ \W } - y_I \leq \min_{(i_1, \ldots, i_N) \in \Omega} \abs{ y_{i_1, \ldots, i_N} } - \delta_h - y_I \leq - \delta_h$.
	Therefore, $\ell_h' ( [ \W ]_I - y_I ) = - \delta_h = \sign( - y_I ) \delta_h$.
	Similarly, if $y_I < 0$, we have that $[ \W ]_I - y_I \geq \delta_h$ and $\ell_h' ( [ \W ]_I - y_I ) = \delta_h = \sign( - y_I ) \delta_h$.
	Note that $y_I$ cannot be exactly zero as, by Assumption~\ref{assump:delta_h}, $\min_{(i_1, \ldots, i_N) \in \Omega} \abs{ y_{i_1, \ldots, i_N} } > \delta_h > 0$.
	The proof concludes by the chain rule:
	\[
	\begin{split}
	\nabla \L_h ( \W ) & = \frac{1}{ \abs{\Omega} } \sum\nolimits_{I \in \Omega} \ell_h' ( [ \W ]_{I} - y_{I}) \cdot \mathcal{E}_{I} \\
	& = \frac{\delta_h}{ \abs{\Omega} } \sum\nolimits_{I \in \Omega} \sign( - y_{I} ) \cdot \mathcal{E}_{I}
	\text{\,.}
	\end{split}
	\]
\end{proof}

\begin{lemma}
	\label{lem:huber_loss_smooth}
	The function $\L_h ( 
	\cdot )$ is $1$-smooth, \ie~for any $\W_1, \W_2 \in \R^{d_1 ,  \ldots , d_N}$:
	\[
	\norm{ \nabla \L_h ( \W_1 ) - \nabla \L_h ( \W_2 ) } \leq \norm{ \W_1 - \W_2 }
	\text{\,.}
	\]
\end{lemma}

\begin{proof}
	Let $\W_1, \W_2 \in \R^{d_1 ,  \ldots , d_N}$.
	Denote by $\ell_h' (\cdot)$ the derivative of $\ell_h (\cdot)$, \ie:
	\[
	\ell_h' (z) = \begin{cases}
		- \delta_h	& , z < - \delta_h \\
		z	& , \abs{ z } \leq \delta_h \\
		\delta_h	& , z > \delta_h , 
	\end{cases}
	\text{\,.}
	\]
	The result readily follows from the triangle inequality and the fact that $\ell_h' (\cdot)$ is $1$-Lipschitz:
	\[
	\begin{split}
		\norm{ \nabla \L_h ( \W_1 ) - \nabla \L_h ( \W_2 ) } & = \norm{ \frac{1}{ \abs{\Omega}} \sum_{I \in \Omega} \left [ \ell_h' ( [ \W_1 ]_I - y_I ) \cdot \mathcal{E}_{I} - \ell_h' ([ \W_2 ]_I - y_I ) \cdot \mathcal{E}_{I} \right ] } \\
		& \leq \frac{1}{ \abs{\Omega}} \sum_{I \in \Omega} \abs{ \ell_h' ( [ \W_1 ]_I - y_I ) - \ell_h' ([ \W_2 ]_I - y_I ) } \\
		& \leq \frac{1}{ \abs{\Omega}} \sum_{I \in \Omega} \abs{ [ \W_1 ]_I - [ \W_2 ]_I }  \\
		& \leq \norm{ \W_1 - \W_2 }
		\text{\,,}
	\end{split}
	\]	
	where $\mathcal{E}_{I} \in \R^{d_1, \ldots, d_N}$ holds $1$ in its $I$'th entry and $0$ elsewhere, for $I = (i_1, \ldots, i_N) \in \Omega$.
\end{proof}

\begin{lemma}
	\label{lem:huber_cp_objective_is_smooth_over_bounded_domain}
	Let $G \geq 0$, and denote $\D_{G} :=  \{  \{ \w_{r}^{n} \in \R^{d_n} \}_{r = 1}^R\hspace{0mm}_{n = 1}^N : ( \sum_{r = 1}^R \sum_{n = 1}^N \normnoflex{ \w_r^n }^2 )^{1/2} \leq G  \}$.
	Then, the objective $\phi_h (\cdot)$ is $R N (G^{2(N - 1)} + \delta_h G^{N - 2})$-smooth over $\D_G$, \ie:
	\[
	\norm{ \nabla \phi_h \left ( \{ \w_{r}^{n} \}_{r = 1}^R\hspace{0mm}_{n = 1}^N \right ) - \nabla \phi_h \left( \{ \widetilde{\w}_{r}^{n} \}_{r = 1}^R\hspace{0mm}_{n = 1}^N \right ) } \leq R N ( G^{2(N - 1)} + \delta_h G^{N - 2}) \cdot \sqrt{ \sum\nolimits_{r = 1}^R \sum\nolimits_{n = 1}^N \norm{ \w_{r}^{n} -  \widetilde{\w}_{r}^{n} }^2 }
	\text{\,,}
	\]
	for any $ \{ \w_{r}^{n} \}_{r = 1}^R\hspace{0mm}_{n = 1}^N,  \{ \widetilde{\w}_{r}^{n} \}_{r = 1}^R\hspace{0mm}_{n = 1}^N \in \D_G$.
\end{lemma}

\begin{proof}
	Let $ \{ \w_{r}^{n} \}_{r = 1}^R\hspace{0mm}_{n = 1}^N,  \{ \widetilde{\w}_{r}^{n} \}_{r = 1}^R\hspace{0mm}_{n = 1}^N \in \D_G$.
	By Lemma~\ref{lem:cp_gradient} we may write:
	\be
	\begin{split}
		& \norm{ \nabla \phi_h \left ( \{ \w_{r}^{n} \}_{r = 1}^R\hspace{0mm}_{n = 1}^N \right ) - \nabla \phi_h \left( \{ \widetilde{\w}_{r}^{n} \}_{r = 1}^R\hspace{0mm}_{n = 1}^N \right ) }^2  \\
		& = \sum_{r = 1}^R \sum_{n = 1}^N \norm{ \mat{ \nabla \L_h \left ( \W_e \right ) }_{n} \cdot \kronp_{n' \neq n} \w_r^{n'} - \mat{ \nabla \L_h \big ( \widetilde{\W}_e \big ) }_{n} \cdot \kronp_{n' \neq n} \widetilde{\w}_r^{n'}  }^2
		\text{\,,}
	\end{split}
	\label{eq:cp_huber_loss_sq_grad_dist}
	\ee
	where $\W_e$ and $\widetilde{\W}_e$ are the end tensors (Equation~\eqref{eq:end_tensor}) of $\{ \w_r^n \}_{r = 1}^R\hspace{0mm}_{n = 1}^N$ and $\{ \widetilde{\w}_r^n \}_{r = 1}^R\hspace{0mm}_{n = 1}^N$, respectively.
	We turn to bound the square root of each term in the sum.
	Fix $r \in [R], n \in [N]$.
	By the triangle inequality and sub-multiplicativity of the Frobenius norm, we have:
	\[
	\begin{split}
		\norm{ \mat{ \nabla \L_h \left ( \W_e \right ) }_{n} \cdot \kronp_{n' \neq n} \w_r^{n'} - \mat{ \nabla \L_h  \big ( \widetilde{\W}_e \big ) }_{n} \cdot \kronp_{n' \neq n} \widetilde{\w}_r^{n'}  } & \leq \underbrace{ \norm{ \mat{ \nabla \L_h \left ( \W_e \right ) }_{n} - \mat{ \nabla \L_h  \big ( \widetilde{\W}_e \big ) }_{n} } }_{(I)} \cdot \underbrace{ \norm{ \kronp_{n' \neq n} \w_r^{n'} } }_{ (II) } \\
		& \hspace{5mm} + \underbrace{ \norm{ \mat{ \nabla \L_h \big ( \widetilde{\W}_e \big ) }_{n} } }_{(III)} \cdot \underbrace{ \norm{ \kronp_{n' \neq n} \w_r^{n'} - \kronp_{n' \neq n} \widetilde{\w}_r^{n'} } }_{(IV)}
		\text{\,.}
	\end{split}
	\]
	Below, we derive upper bounds for $(I), (II), (III)$ and $(IV)$ separately.
	Starting with $(I)$, by Lemma~\ref{lem:huber_loss_smooth}, the triangle inequality and Lemma~\ref{lem:outer_prod_distance_bound}, it follows that:
	\[
	\begin{split}
		(I) & = \norm{ \nabla \L_h \left ( \W_e \right ) - \nabla \L_h  \big ( \widetilde{\W}_e \big ) } \\
		& \le \norm{ \W_e - \widetilde{\W}_e } \\
		& \leq \sum_{r' = 1}^R \norm{ \tenp_{n' = 1}^N \w_{r'}^{n'} - \tenp_{n' = 1}^N \widetilde{\w}_{r'}^{n'} } \\
		& \leq G^{N - 1} \sum_{r' = 1}^R \sum_{n' = 1}^N  \norm{ \w_{r'}^{n'} - \widetilde{\w}_{r'}^{n'} }
		\text{\,.}
	\end{split}
	\]
	Moving on to $(II)$, we have that $\normnoflex{ \kronp_{n' \neq n} \w_r^{n'} } = \prod_{n' \neq n} \normnoflex{ \w_r^{n'} } \leq G^{N - 1}$.
	For $(III)$, the triangle inequality and the fact that $\ell_h' ( \cdot )$, the derivative of $\ell_h (\cdot)$, is bounded (in absolute value) by $\delta_h$ yield:
	\[
	\begin{split}
		(III) & = \norm{ \frac{1}{ \abs{\Omega} } \sum_{I \in \Omega} \ell_h' \left ( [ \widetilde{\W}_e ]_I - y_I \right ) \cdot \mathcal{E}_{I} } \leq \delta_h
		\text{\,,}
	\end{split}
	\]
	where $\mathcal{E}_{I} \in \R^{d_1, \ldots, d_N}$ holds $1$ in its $I$'th entry and $0$ elsewhere, for $I = (i_1, \ldots, i_N) \in \Omega$.
	Lastly, since $\normnoflex{ \kronp_{n' \neq n} \w_r^{n'} - \kronp_{n' \neq n} \widetilde{\w}_r^{n'} } = \normnoflex{ \tenp_{n' \neq n} \w_r^{n'} - \tenp_{n' \neq n} \widetilde{\w}_r^{n'} }$, by Lemma~\ref{lem:outer_prod_distance_bound} we have that:
	\[
	(IV) \leq G^{N - 2} \sum_{n' \neq n} \norm{ \w_r^{n'} - \widetilde{\w}_r^{n'} } \leq G^{N - 2} \sum_{n' =1}^N \norm{ \w_r^{n'} - \widetilde{\w}_r^{n'} }
	\text{\,.}
	\]
	Putting it all together, we arrive at the following bound:
	\[
	\begin{split}
		& \norm{ \mat{ \nabla \L_h \left ( \W_e \right ) }_{n} \cdot \kronp_{n' \neq n} \w_r^{n'} - \mat{ \nabla \L_h  \big ( \widetilde{\W}_e \big ) }_{n} \cdot \kronp_{n' \neq n} \widetilde{\w}_r^{n'}  } \\
		& \leq G^{2(N - 1)} \sum_{r' = 1}^R \sum_{n' = 1}^N  \norm{ \w_{r'}^{n'} - \widetilde{\w}_{r'}^{n'} } + \delta_h G^{N - 2} \sum_{n' =1}^N \norm{ \w_r^{n'} - \widetilde{\w}_r^{n'} } \\
		& \leq ( G^{2(N - 1)} + \delta_h G^{N - 2}) \sum_{r' = 1}^R \sum_{n' = 1}^N  \norm{ \w_{r'}^{n'} - \widetilde{\w}_{r'}^{n'} } 
		\text{\,.}
	\end{split}
	\]
	Applying the bound above to Equation~\eqref{eq:cp_huber_loss_sq_grad_dist}, for all $r \in [R] , n \in [N]$, leads to:
	\[
	\begin{split}
		& \norm{ \nabla \phi_h \left ( \{ \w_{r}^{n} \}_{r = 1}^R\hspace{0mm}_{n = 1}^N \right ) - \nabla \phi_h \left( \{ \widetilde{\w}_{r}^{n} \}_{r = 1}^R\hspace{0mm}_{n = 1}^N \right ) }^2  \\
		& \leq R N ( G^{2(N - 1)} + \delta_h G^{N - 2})^2 \left ( \sum\nolimits_{r = 1}^R \sum\nolimits_{n = 1}^N  \norm{ \w_{r}^{n} - \widetilde{\w}_{r}^{n} } \right )^2 \\
		& \leq R^2 N^2 ( G^{2(N - 1)} + \delta_h G^{N - 2})^2 \sum\nolimits_{r = 1}^R \sum\nolimits_{n = 1}^N  \norm{ \w_{r}^{n} - \widetilde{\w}_{r}^{n} }^2
		\text{\,,}
	\end{split}
	\]
	where the last transition is by the fact that $\normnoflex{ \x }_1 \leq \sqrt{d} \cdot \normnoflex{ \x }$ for any $\x \in \R^d$.
	Taking the square root of both sides concludes the proof.
\end{proof}

\begin{lemma}
	\label{lem:const_grad_inner_prod_monotonicity_helper_lemma}
	Let $t' > 0$ and $r \in [R]$.
	Denote $\gamma_r (t) := \inprodnoflex{ - \nabla \L_h ( \W_e (t) ) }{ \tenp_{n = 1}^N \widehat{\w}_r^{n} (t) }$, where $\widehat{\w}_r^{n} (t) := \w_r^{n} (t) / \normnoflex{ \w_r^{n} (t) }$ if $\w_r^{n} (t) \neq 0$, and $\widehat{\w}_r^{n} (t) := 0$ otherwise, for $n = 1, \ldots, N$. Suppose that $\nabla \L_h ( \W_e (t)) = \nabla \L_h (0)$ for all $t \in [0, t')$.
	Then, $\gamma_r (t)$ is monotonically non-decreasing over the interval $[0, t')$.
\end{lemma}

\begin{proof}
	In the following, unless explicitly stated otherwise, $t$ is to be considered in the time interval $[0, t')$.
	
	Recall that by Assumption~\ref{assump:a_balance} we have that $\normnoflex{ \w_r^{1} (0) } = \cdots = \normnoflex{ \w_r^{N} (0) }$.
	If $\normnoflex{ \w_r^{1} (0) } = \cdots = \normnoflex{ \w_r^{N} (0) } = 0$, then according to Lemma~\ref{lem:balanced_param_vector_norm_no_sign_change} $\normnoflex{ \w_r^{1} (t) } = \cdots = \normnoflex{ \w_r^{N} (t) } = 0$ for all $t \geq 0$.
	In this case $\gamma_r (t) = 0$ over $[0, t')$, and is therefore non-decreasing.
	
	Otherwise, if $\normnoflex{ \w_r^{1} (0) } = \cdots = \normnoflex{ \w_r^{N} (0) } > 0$, from Lemma~\ref{lem:balanced_param_vector_norm_no_sign_change}  we get that $\normnoflex{ \w_r^{1} (t) } = \cdots = \normnoflex{ \w_r^{N} (t) }> 0$ for all $t \geq 0$.
	Thus:
	\[
	\begin{split}
	\gamma_r (t) & = \norm{ \tenp_{n = 1}^N \w_r^{n} (t) }^{-1} \inprodnoflex{ - \nabla \L_h ( \W_e (t) ) }{ \tenp_{n = 1}^N \w_r^{n} (t) } \\
	& = \norm{ \tenp_{n = 1}^N \w_r^{n} (t) }^{-1} \inprodnoflex{ - \nabla \L_h ( 0 ) }{ \tenp_{n = 1}^N \w_r^{n} (t) }
	\text{\,,}
	\end{split}
	\] 
	where the second transition is due to $\nabla \L_h ( \W_e (t)) = \nabla \L_h (0)$.
	Differentiating with respect to time, we have that:
	\be
	\begin{split}
		\frac{d}{dt} \gamma_r (t) & = - \underbrace{\frac{d}{dt} \left [ \norm{ \tenp_{n = 1}^N \w_r^{n} (t) } \right ] \cdot \norm{ \tenp_{n = 1}^N \w_r^{n} (t) }^{-2} \inprod{ - \nabla \L_h ( 0 ) }{ \tenp_{n = 1}^N \w_r^{n} (t) }}_{(I)} \\
		& \hspace{5mm} + \norm{ \tenp_{n = 1}^N \w_r^{n} (t) }^{-1} \underbrace{ \inprod{ - \nabla \L_h ( 0 ) }{ \tfrac{d}{dt} \tenp_{n = 1}^N \w_r^{n} (t) } }_{ (II) }
		\text{\,.}
	\end{split}
	\label{eq:inner_prod_time_deriv_two_parts}
	\ee
	We now treat $(I)$ and $(II)$ separately.
	Plugging the expression for $\frac{d}{dt} \normnoflex{ \tenp_{n = 1}^N \w_r^{n} (t) }$ from Corollary~\ref{cor:dyn_fac_comp_norm_balanced} into $(I)$, and recalling that $\nabla \L_h ( \W_e (t) ) = \nabla \L_h (0)$, leads to:
	\[
	\begin{split}
		(I) = N \norm{ \tenp_{n = 1}^N \w_r^{n} (t) }^{-1 - 2 / N} \inprod{ - \nabla \L_h ( 0 ) }{ \tenp_{n = 1}^N \w_r^{n} (t) }^2
		\text{\,.}
	\end{split}
	\]
	Due to the fact that $\normnoflex{ \tenp_{n = 1}^N \w_r^{n} (t) }^{-2/N} = \normnoflex{ \w_r^{1} (t) }^{-2} = \cdots = \normnoflex{ \w_r^{N} (t) }^{-2}$, we may equivalently write:
	\be
	\begin{split}
		(I) = \norm{ \tenp_{n = 1}^N \w_r^{n} (t) }^{-1} \sum_{n = 1}^N \norm{ \w_r^{n} (t) }^{-2} \inprod{ - \nabla \L_h ( 0 ) }{ \tenp_{n' = 1}^N \w_r^{n'} (t) }^2
		\text{\,.}
	\end{split}
	\label{eq:I_phrase_in_mon_inner_prod_lemma}
	\ee
	For any $n \in [N]$, by Lemma~\ref{lem:dyn_parameter_vector_sq_norm} we know that $\frac{d}{dt} \normnoflex{ \w_r^{n} (t) }^2 = - 2 \inprodnoflex{ \nabla \L_h \left ( 0 \right ) }{ \tenp_{n' = 1}^N \w_r^{n'} (t) }$, which implies $\frac{d}{dt} \normnoflex{ \w_r^{n} (t) } = \normnoflex{ \w_r^{n} (t) }^{-1} \inprodnoflex{ - \nabla \L_h \left ( 0 \right ) }{ \tenp_{n' = 1}^N \w_r^{n'} (t) }$.
	Going back to Equation~\eqref{eq:I_phrase_in_mon_inner_prod_lemma}, we can see that:
	\[
	(I) = \norm{ \tenp_{n = 1}^N \w_r^{n} (t) }^{-1} \sum_{n = 1}^N \left ( \tfrac{d}{dt} \norm{ \w_r^{n} (t) } \right )^2
	\text{\,.}
	\]
	Turning our attention to $(II)$, by Lemmas~\ref{lem:inp_with_tenp_to_mat_kronp} and~\ref{lem:cp_gradient} it follows that:
	\[
	\begin{split}
		(II) & =  \sum_{n = 1}^N \inprod{ - \nabla \L_h ( 0 ) }{ \left ( \tenp_{n' = 1}^{n - 1} \w_r^{n'} (t) \right ) \tenp \tfrac{d}{dt} \w_r^{n} (t) \tenp \left ( \tenp_{n' = n + 1}^N \w_r^{n'} (t) \right )  } \\
		& = \sum_{n = 1}^N \inprod{ \mat{ - \nabla \L_h \left ( 0 \right ) }_{n} \cdot \kronp_{n' \neq n} \w_r^{n'} (t)}{ \tfrac{d}{dt} \w_r^{n} (t) } \\
		& = \sum_{n = 1}^N \norm{  \tfrac{d}{dt} \w_r^{n} (t) }^2
		\text{\,.}
	\end{split}
	\]
	Plugging the expressions we derived for $(I)$ and $(II)$ into Equation~\eqref{eq:inner_prod_time_deriv_two_parts} yields:
	\be
	\begin{split}
		& \frac{d}{dt} \gamma_r (t) = \norm{ \tenp_{n = 1}^N \w_r^{n} (t) }^{-1} \cdot \sum_{n = 1}^N  \left [ \norm{  \tfrac{d}{dt} \w_r^{n} (t) }^2 -  \left ( \tfrac{d}{dt} \norm{ \w_r^{n} (t) } \right )^2 \right ]
		\text{\,.}
	\end{split}
	\label{eq:inner_prod_time_deriv_final}
	\ee
	Notice that for any $n \in [N]$:
	\[
	\begin{split}
		\norm{ \tfrac{d}{dt} \w_r^{n} (t) }^2 & \geq \norm{ \Pi_{\w_r^{n} (t)} \left ( \tfrac{d}{dt} \w_r^{n} (t) \right ) }^2 \\
		& = \norm{ \inprodnoflex{ \tfrac{d}{dt} \w_r^{n} (t) }{ \w_r^{n} (t) } \frac{ \w_r^{n} (t) }{  \normnoflex{ \w_r^{n} (t) }^2 } }^2 \\
		& = \left ( \norm{ \w_r^{n} (t) }^{-1} \inprod{ \tfrac{d}{dt} \w_r^{n} (t) }{ \w_r^{n} (t) } \right )^2 \\
		& =  \left ( \tfrac{d}{dt} \norm{ \w_r^{n} (t) } \right )^2
		\text{\,,}
	\end{split}
	\]
	where $\Pi_{\w_r^{n} (t)} (\cdot)$ denotes the orthogonal projection onto the subspace spanned by $\w_r^{n} (t)$.
	The right hand side in Equation~\eqref{eq:inner_prod_time_deriv_final} is therefore non-negative, \ie~$\frac{d}{dt} \gamma_r (t) \geq 0$,
	concluding the proof.
\end{proof}

\subsubsection{Stage I: End Tensor Reaches Reference Sphere}
\label{app:proofs:approx_rank_1:const_grad}

\begin{proposition}
\label{prop:end_tensor_reaches_ref_sphere}
The end tensor initially reaches reference sphere $\S$ (Equation~\eqref{eq:ref_sphere}) at some time $T_0 >0$, and:
\begin{align}
	& \norm{ \tenp_{n = 1}^N \w_r^{n} (t) } \leq \tilde{\epsilon} \quad, ~t \in [0, T_0] ~,~ r = 2, \ldots, R 
	\text{\,,}	
	\label{eq:approx_rank_1_stage_I_t_0_comp_bounds_upper_bound} \\[0.5ex]
	& \abs{\norm{ \tenp_{n = 1}^N \w_1^{n} (T_0) } - \rho} \leq (R - 1) \cdot \tilde{\epsilon} \text{\,,} 
	\label{eq:approx_rank_1_stage_I_t_0_first_comp_bound} 
\end{align}
where $\tilde{\epsilon}$ is as defined in Equation~\eqref{eq:approx_rank_1_consts}.
\end{proposition}

\medskip

Towards proving Proposition~\ref{prop:end_tensor_reaches_ref_sphere}, we establish the following key lemma.
\begin{lemma}
	\label{lem:const_grad_comp_norm_bounds_helper_lemma}
	Let $t' \leq  \frac{\alpha^{2 - N} \normnoflex{ \aaa_1 }^{2 - N} (N - 2)^{-1}}{ \inprodnoflex{ - \nabla \L_h (0) }{ \tenp_{n = 1}^N \widehat{\aaa}_{ 1 }^{n} } }$, and suppose that $\nabla \L_h ( \W_e (t)) = \nabla \L_h (0)$ for all $t \in [0, t')$.
	Then:
	\begin{align}
		& \norm{ \tenp_{n = 1}^N \w_1^{n} (t) } \geq \left ( \alpha^{2-N} \normnoflex{ \aaa_1 }^{2 - N}  - (N - 2) \inprod{ - \nabla \L_h (0) }{ \tenp_{n = 1}^N \widehat{\aaa}_{ 1 }^{n} } \cdot t \right )^{- \frac{N}{N-2}} ~~,~ t \in [0, t') \text{\,,}
		\label{eq:comp_norm_lower_bound_with_time_helper_lemma} \\[0.5ex]
		& \norm{ \tenp_{n = 1}^N \w_r^{n} (t) } \leq \left ( \alpha^{2-N} \normnoflex{ \aaa_r }^{2 - N}  - (N - 2) \norm{ \nabla \L_h ( 0 ) } \cdot t \right )^{- \frac{N}{N-2}} ~~,~ t \in [0, t') ~,~ r = 2, \ldots, R \text{\,.}
		\label{eq:comp_norms_upper_bound_with_time_helper_lemma}
	\end{align}
	In particular:
	\be
	\norm{ \tenp_{n = 1}^N \w_r^{n} (t) } \leq \alpha^N \left ( \normnoflex{ \aaa_r }^{2 - N}  - \normnoflex{ \aaa_1 }^{2 - N} \tfrac{ \norm{ \nabla \L_h ( 0 ) } }{  \inprod{ - \nabla \L_h (0) }{ \tenp_{n = 1}^N \widehat{\aaa}_{ 1 }^{n} } } \right )^{- \frac{N}{N-2}} ~~,~ t \in [0, t') ~,~ r = 2, \ldots, R \text{\,.}
	\label{eq:comp_norms_upper_bound_helper_lemma}
	\ee
\end{lemma}

\begin{proof}
For simplicity of notation we denote $\gamma_r (t) := \inprodnoflex{ - \nabla \L_h ( \W_e (t) ) }{ \tenp_{n = 1}^N \widehat{\w}_r^{n} (t) }$, where $\widehat{\w}_r^{n} (t) := \w_r^{n} (t) / \normnoflex{ \w_r^{n} (t) }$ if $\w_r^{n} (t) \neq 0$, and $\widehat{\w}_r^{n} (t) := 0$ otherwise, for $r = 1, \ldots, R, ~n = 1, \ldots, N$.
In the following, unless explicitly stated otherwise, $t$ is to be considered in the time interval $[0, t')$.

Since $\{ \aaa_r^n \}_{r = 1}^R\hspace{0mm}_{n = 1}^N$ have unbalancedness magnitude zero (Assumption~\ref{assump:a_balance}) so do $\{ \w_r^n (0) \}_{r = 1}^R\hspace{0mm}_{n = 1}^N$ (recall $\w_r^n (0) = \alpha \cdot \aaa_r^n$ for $r = 1, \ldots, R, ~n = 1, \ldots, N$).
According to Corollary~\ref{cor:dyn_fac_comp_norm_balanced} the evolution of a component's norm is given by:
\be
\frac{d}{dt} \norm{ \tenp_{n = 1}^N \w_r^{n} (t) } = N \gamma_r (t) \cdot \norm{ \tenp_{n = 1}^N \w_r^{n} (t) }^{2 - \frac{2}{N}} \quad , ~r = 1, \ldots, R
\text{\,.}
\label{eq:dyn_fac_comp_norm_const_grad}
\ee

\paragraph*{Proof of Equation~\eqref{eq:comp_norm_lower_bound_with_time_helper_lemma}  (lower bound for $\normnoflex{ \tenp_{n = 1}^N \w_1^{n} (t) }$):}

By Lemma~\ref{lem:const_grad_inner_prod_monotonicity_helper_lemma}, $\gamma_1 (t)$ is monotonically non-decreasing.
Thus, from Equation~\eqref{eq:dyn_fac_comp_norm_const_grad} we have:
\be
\frac{d}{dt} \norm{ \tenp_{n = 1}^N \w_1^{n} (t) } \geq N \gamma_1 (0) \cdot \norm{ \tenp_{n = 1}^N \w_1^{n} (t) }^{2 -\frac{2}{N}}
\text{\,.}
\label{eq:comp_1_norm_deriv_lower_bound}
\ee
Assumption~\ref{assump:a_lead_comp} (second line in Equation~\eqref{eq:assump_components_sep_at_init}) necessarily means that $\w_1^{n} (0) = \alpha \cdot \aaa_1^{n} \neq 0$ for all $n \in [N]$.
Recalling that the unbalancedness magnitude is zero at initialization, from Lemma~\ref{lem:balanced_param_vector_norm_no_sign_change} we get that $\normnoflex{ \w_1^{1} (t) } = \cdots = \normnoflex{ \w_1^{N} (t) } > 0$, and so $\normnoflex{ \tenp_{n = 1}^N \w_1^{n} (t) }^{2 - 2 / N} > 0$, for all $t \in [0, t')$.
Therefore, we may divide both sides of Equation~\eqref{eq:comp_1_norm_deriv_lower_bound} by $\normnoflex{ \tenp_{n = 1}^N \w_1^{n} (t) }^{2 - 2 / N}$.
Doing so, and integrating with respect to time, leads to:
\be
\begin{split}
	& \int_{\hat{t} = 0}^t \left [ \norm{ \tenp_{n = 1}^N \w_1^{n} (\hat{t}) }^{2 / N - 2} \frac{d}{d \hat{t}} \norm{ \tenp_{n = 1}^N \w_1^{n} (\hat{t}) } \right ] d \hat{t} \geq N \gamma_1 (0) \cdot t \\
	\implies & \frac{N}{2 - N} \left ( \norm{ \tenp_{n = 1}^N \w_1^{n} (t) }^{2 / N - 1} - \norm{ \tenp_{n = 1}^N \w_1^{n} (0) }^{2 / N - 1} \right ) \geq N \gamma_1 (0) \cdot t \\
	\implies & \norm{ \tenp_{n = 1}^N \w_1^{n} (t) }^{2 / N - 1} \leq  \norm{ \tenp_{n = 1}^N \w_1^{n} (0) }^{2 / N - 1} - (N - 2) \gamma_1 (0) \cdot t
	\text{\,.}
\end{split}
\label{eq:comp_1_integration_lower_bound}
\ee
Notice that $\gamma_1 (0) =  \inprodnoflex{ - \nabla \L_h ( \W_e (0) ) }{ \tenp_{n = 1}^N \widehat{\w}_1^{n} (0) } = \inprodnoflex{ - \nabla \L_h ( 0 ) }{ \tenp_{n = 1}^N \widehat{\aaa}_1^{n} }$. 
Since $\normnoflex{ \tenp_{n = 1}^N \w_1^{n} (0) } = \prod_{n = 1}^N \normnoflex{ \w_1^{n} (0) } = \alpha^N \normnoflex{ \aaa_1 }^N$ and $t < t' \leq \alpha^{2 - N} \normnoflex{ \aaa_1 }^{2 - N} (N - 2)^{-1} \gamma_1 (0)^{-1}$, we can see that:
\[
\norm{ \tenp_{n = 1}^N \w_1^{n} (0) }^{2 / N - 1} - (N - 2) \gamma_1 (0) \cdot t = \alpha^{2 - N} \norm{ \aaa_1 }^{2 - N} - (N - 2) \gamma_1 (0) \cdot t > 0
\text{\,.}
\]
Therefore, Equation~\eqref{eq:comp_norm_lower_bound_with_time_helper_lemma} readily follows by rearranging the last inequality in Equation~\eqref{eq:comp_1_integration_lower_bound}:
\[
\begin{split}
	\norm{ \tenp_{n = 1}^N \w_1^{n} (t) } \geq \left (\alpha^{2 - N} \norm{ \aaa_1}^{2 - N} - (N - 2) \gamma_1 (0) \cdot t  \right )^{ - \frac{N}{N - 2} }
	\text{\,.}
\end{split}
\]

\paragraph*{Proof of Equations~\eqref{eq:comp_norms_upper_bound_with_time_helper_lemma} and~\eqref{eq:comp_norms_upper_bound_helper_lemma}  (upper bounds for $\normnoflex{ \tenp_{n = 1}^N \w_r^{n} (t) }$):}

Fix some $r \in \{ 2, \ldots, R \}$.
First, we deal with the case where $\normnoflex{ \w_r^{1} (0) } = \cdots = \normnoflex{ \w_r^{N} (0) } = 0$.
If it is so, by Lemma~\ref{lem:balanced_param_vector_norm_no_sign_change} we have that $\normnoflex{ \w_r^{1} (t) } = \cdots = \normnoflex{ \w_r^{N} (t) } = 0$ for all $t \in [0, t')$.
Hence, $\normnoflex{ \tenp_{n = 1}^N \w_r^{n} (t) } = 0$ for all $t \in [0, t')$, \ie~Equations~\eqref{eq:comp_norms_upper_bound_with_time_helper_lemma} and~\eqref{eq:comp_norms_upper_bound_helper_lemma} trivially hold.

Now we move to the case where $\normnoflex{ \w_r^{1} (0) } = \cdots = \normnoflex{ \w_r^{N} (0) } > 0$.
From Lemma~\ref{lem:balanced_param_vector_norm_no_sign_change} we know that $\normnoflex{ \w_r^{1} (t) } = \cdots = \normnoflex{ \w_r^{N} (t) } > 0$ for all $t \in [0, t')$.
Since $\nabla \L_h ( \W_e (t)) = \nabla \L_h (0)$, by the Cauchy-Schwartz inequality we then have:
\[
\gamma_r (t) = \inprod{ - \nabla \L_h ( 0 ) }{ \tenp_{n = 1}^N \widehat{\w}_r^{n} (t) } \leq \norm{ \nabla \L_h (0) } \norm{ \tenp_{n = 1}^N \widehat{\w}_r^{n} (t) } = \norm{ \nabla \L_h (0) }
\text{\,.}
\]
Combined with Equation~\eqref{eq:dyn_fac_comp_norm_const_grad}, we arrive at the following upper bound:
\[
\frac{d}{dt} \norm{ \tenp_{n = 1}^N \w_r^{n} (t) } \leq N \norm{ \nabla \L_h (0) } \cdot \norm{ \tenp_{n = 1}^N \w_r^{n} (t) }^{2 - \frac{2}{N}}
\text{\,.}
\]
Dividing both sides of the inequality by $\normnoflex{ \tenp_{n = 1}^N \w_r^{n} (t) }^{2 - 2 / N}$ (is positive since $\normnoflex{ \w_r^{1} (t) } = \cdots = \normnoflex{ \w_r^{N} (t) } > 0$), and integrating with respect to time, yields:
\[
\begin{split}
	& \int_{\hat{t} = 0}^t \left [ \norm{ \tenp_{n = 1}^N \w_r^{n} (\hat{t}) }^{2 / N - 2} \frac{d}{d \hat{t}} \norm{ \tenp_{n = 1}^N \w_r^{n} (\hat{t}) } \right ] d \hat{t} \leq N \norm{ \nabla \L_h (0) } \cdot t \\
	\implies & \frac{N}{2 - N} \left ( \norm{ \tenp_{n = 1}^N \w_r^{n} (t) }^{2 / N - 1} - \norm{ \tenp_{n = 1}^N \w_r^{n} (0) }^{2 / N - 1} \right ) \leq N \norm{ \nabla \L_h (0) } \cdot t
	\text{\,.}
\end{split}
\]
Rearranging the inequality above, and making use of the fact that $\normnoflex{ \tenp_{n = 1}^N \w_r^{n} (0) } = \prod_{n = 1}^N \normnoflex{ \w_r^{n} (0) } = \alpha^N \normnoflex{ \aaa_r }^N$, we arrive at:
\be
\begin{split}
\norm{ \tenp_{n = 1}^N \w_r^{n} (t) }^{2 / N - 1} & \geq \norm{ \tenp_{n = 1}^N \w_r^{n} (0) }^{2 / N - 1} - (N - 2) \norm{ \nabla \L_h (0) } \cdot t \\
& = \alpha^{2 - N} \norm{ \aaa_r}^{2 - N} - (N - 2) \norm{ \nabla \L_h (0) } \cdot t 
\text{\,.}
\end{split}
\label{eq:comp_norm_derivation_intermediate_upper_bound}
\ee
Noticing $\gamma_1 (0) =  \inprodnoflex{ - \nabla \L_h ( \W_e (0) ) }{ \tenp_{n = 1}^N \widehat{\w}_1^{n} (0) } = \inprodnoflex{ - \nabla \L_h ( 0 ) }{ \tenp_{n = 1}^N \widehat{\aaa}_1^{n} }$, by Assumption~\ref{assump:a_lead_comp} we have that $\normnoflex{ \aaa_{1} } > \normnoflex{ \aaa_{r} } \norm{ \nabla \L_h ( 0 ) }^{1 / (N - 2)} \cdot \gamma_1 (0)^{- 1 / (N - 2)}$.
Therefore:
\[
t' \leq \alpha^{2 - N} \normnoflex{ \aaa_1 }^{2 - N} (N - 2)^{-1} \gamma_1 (0)^{-1} < \alpha^{2 - N} \normnoflex{ \aaa_r }^{2 - N} (N - 2)^{-1} \normnoflex{ \nabla \L_h (0)}^{-1}
\text{\,.}
\]
This implies that the right hand side in Equation~\eqref{eq:comp_norm_derivation_intermediate_upper_bound} is positive for all $t \in [0, t')$.
Thus, rearranging Equation~\eqref{eq:comp_norm_derivation_intermediate_upper_bound} establishes
Equation~\eqref{eq:comp_norms_upper_bound_with_time_helper_lemma}:
\[
\norm{ \tenp_{n = 1}^N \w_r^{n} (t) } \leq \left ( \alpha^{2-N} \normnoflex{ \aaa_r }^{2 - N}  - (N - 2) \norm{ \nabla \L_h ( 0 ) } \cdot t \right )^{- \frac{N}{N-2}} 
\text{\,.}
\]
Equation~\eqref{eq:comp_norms_upper_bound_helper_lemma} then directly follows:
\[
\begin{split}
	\norm{ \tenp_{n = 1}^N \w_r^{n} (t) } & \leq \left ( \alpha^{2-N} \normnoflex{ \aaa_r }^{2 - N}  - (N - 2) \norm{ \nabla \L_h ( 0 ) } \cdot t' \right )^{- \frac{N}{N-2}} \\
	& \leq \left ( \alpha^{2-N} \normnoflex{ \aaa_r }^{2 - N}  - \alpha^{2 - N} \normnoflex{ \aaa_1 }^{2 - N} \norm{ \nabla \L_h ( 0 ) } \gamma_1 (0)^{-1}  \right )^{- \frac{N}{N-2}} \\
	& = \alpha^N \left (  \normnoflex{ \aaa_r }^{2 - N} - \normnoflex{ \aaa_1 }^{2 - N} \norm{ \nabla \L_h ( 0 ) } \gamma_1 (0)^{-1} \right )^{- \frac{N}{N-2}} 
	\text{\,.}
\end{split}
\]
\end{proof}

\medskip

\begin{proof}[Proof of Proposition~\ref{prop:end_tensor_reaches_ref_sphere}]
Notice that at initialization $\norm{ \W_e (0) } \leq  \sum_{r = 1}^R \normnoflex{ \tenp_{n = 1}^N \w_r^{n} (0) } \leq R \alpha^N A^N < \rho$.
We can therefore examine the trajectory up until the time at which $\norm{ \W_e (t) } = \rho$, \ie~until it reaches the reference sphere $\S$.
Formally, define:
\[
T_0 := \inf \left \{ t \geq 0 : \W_e (t) \in \S \right \}
\text{\,,}
\]
where by convention $T_0 := \infty$ if the set on the right hand side is empty.
For all $t \in [0, T_0)$, clearly, $\norm{ \W_e (t) } < \rho$, and so by Lemma~\ref{lem:huber_loss_const_grad_near_zero} $\nabla \L_h ( \W_e (t) ) = \nabla \L_h (0)$.
We claim that $T_0$ is finite.
Assume by way of contradiction that $T_0 = \infty$.
For $t' :=  \alpha^{2 - N} \normnoflex{ \aaa_1 }^{2 - N} (N - 2)^{-1} \inprodnoflex{ - \nabla \L_h (0) }{ \tenp_{n = 1}^N \widehat{\aaa}_{1}^{n} }^{-1}$, by Equation~\eqref{eq:comp_norm_lower_bound_with_time_helper_lemma} from Lemma~\ref{lem:const_grad_comp_norm_bounds_helper_lemma} we have that $ \normnoflex{ \tenp_{n = 1}^N \w_1^{n} (t) } $ is lower bounded by a quantity that goes to $\infty$ as $t \to t^{\prime -}$.
On the other hand, by Equation~\eqref{eq:comp_norms_upper_bound_helper_lemma} from Lemma~\ref{lem:const_grad_comp_norm_bounds_helper_lemma}, $\normnoflex{ \tenp_{n = 1}^N \w_2^{n} (t) }, \ldots, \normnoflex{ \tenp_{n = 1}^N \w_R^{n} (t) }$ are bounded over $[0, t')$.
Taken together, there must exist $\hat{t} \in [0, t')$ at which:
\[
\norm{ \W_e  ( \hat{t})} \geq \normnoflex{ \tenp_{n = 1}^N \w_1^{n} ( \hat{t} ) } - \sum_{r = 2}^R \normnoflex{ \tenp_{n = 1}^N \w_r^{n} ( \hat{t} ) } \geq \rho
\text{\,.}
\]
Since $\norm{ \W_e (t) }$ is continuous in $t$, and $\norm{ \W_e (0) } < \rho$, this contradicts our assumption that $T_0 = \infty$.
Hence, $T_0 < \infty$, and in particular $T_0 < t'$.
Notice that continuity of $\norm{ \W_e (t) }$ further implies that $\norm{ \W_e (T_0) } = \rho$, \ie~ $T_0$ is the initial time at which $\W_e (t)$ reaches the reference sphere $\S$.
Applying our assumption on the size of $\alpha$ (Equation~\eqref{eq:alpha_size_requirement}) to Equation~\eqref{eq:comp_norms_upper_bound_helper_lemma} from Lemma~\ref{lem:const_grad_comp_norm_bounds_helper_lemma} establishes Equation~\eqref{eq:approx_rank_1_stage_I_t_0_comp_bounds_upper_bound}.
Equation~\eqref{eq:approx_rank_1_stage_I_t_0_first_comp_bound} then readily follows by the triangle inequality:
\[
\begin{split}
 \abs{ \norm{ \tenp_{n = 1}^N \w_1^{n} (T_0) } - \rho } & =  \abs{ \norm{ \tenp_{n = 1}^N \w_1^{n} (T_0) } - \norm{ \W_e (T_0) } } \\
 & \leq \norm{ \tenp_{n = 1}^N \w_1^{n} (T_0) - \W_e (T_0) } \\
 & = \norm{ \sum\nolimits_{r = 2}^R \tenp_{n = 1}^N \w_r^n (T_0) }  \\
 & \leq (R - 1) \cdot \tilde{\epsilon}
\text{\,.}
\end{split}
\]
\end{proof}

\subsubsection{Stage II: End Tensor Follows Rank One Trajectory}
\label{app:proofs:approx_rank_1:trajectory}

As shown in Proposition~\ref{prop:end_tensor_reaches_ref_sphere} (Subappendix~\ref{app:proofs:approx_rank_1:const_grad}), the end tensor initially reaches reference sphere $\S$ at some time $T_0 > 0$, for which Equations~\eqref{eq:approx_rank_1_stage_I_t_0_comp_bounds_upper_bound} and~\eqref{eq:approx_rank_1_stage_I_t_0_first_comp_bound} hold.
Therefore, the time-shifted trajectory is given by $\overline{\W}_e (t) = \W_e (t + T_0)$ for all $t \geq 0$.
Denote the corresponding time-shifted factorization weight vectors by:
\[
\widebar{\w}_r^n (t) := \w_r^n (t + T_0) \quad , ~ t \geq 0 ~,~ r = 1, \ldots, R ~,~ n = 1, \ldots, N
\text{\,.}
\]
We are now at a position to define the approximating rank one trajectory $\W_1 ( t )$ emanating from~$\S$.
Let $\{ \widetilde{\w}^n (t) \}_{n = 1}^N$ be a curve born from gradient flow when minimizing $\phi_h (\cdot)$ with a one-component tensor factorization, initialized at:
 \[
\widetilde{\w}^n (0) := \frac{ \rho^{1 / N} }{\normnoflex{ \widebar{\w}_1^n (0) }} \cdot \widebar{\w}_{1}^n (0) \quad , ~n = 1, \ldots, N
\text{\,.}
\]
Notice that by definition $\normnoflex{ \widetilde{\w}^1 (0) } = \cdots = \normnoflex{ \widetilde{\w}^N (0) } = \rho^{1 / N}$.
Therefore, $\{ \widetilde{\w}^n (0) \}_{n = 1}^N$ have unbalancedness magnitude zero (Definition~\ref{def:unbalancedness_magnitude}).
Denoting $\W_1 (t) := \tenp_{n = 1}^N \widetilde{\w}^n (t)$, for $t \geq 0$, we can see that $\W_1 (t)$ is a balanced rank one trajectory.
Furthermore, $\normnoflex{ \W_1 (0) } = \normnoflex{ \tenp_{n = 1}^N \widetilde{\w}^n (0) } = \prod_{n = 1}^N \normnoflex{ \widetilde{\w}^n (0) } = \rho$, meaning $\W_1 (0) \in \S$.
It will be convenient to treat $\{ \widetilde{\w}^n (t) \}_{n = 1}^N$ as an $R$-component factorization with components $2, \ldots, R$ being zero.
To this end, denote $\widetilde{\w}_1^n (t) := \widetilde{\w}^n (t)$, and define $\widetilde{\w}_r^n (t) := 0$ for all $t \geq 0$, $r \in \{ 2, \ldots, R \}$ and $n \in [N]$.
Notice that, according to Lemma~\ref{lem:width_R_equivalent_to_larger_width_with_zero_init}, $\{ \widetilde{\w}_r^n (t) \}_{r = 1}^R\hspace{0mm}_{n = 1}^N$ indeed follow a gradient flow path of an $R$-component factorization.

Next, we turn to bound the distance between $\{ \widebar{\w}_r^n (0) \}_{r = 1}^R\hspace{0mm}_{n = 1}^N$ and $\{ \widetilde{\w}_r^n (0) \}_{r = 1}^R\hspace{0mm}_{n = 1}^N$.
From Equation~\eqref{eq:approx_rank_1_stage_I_t_0_comp_bounds_upper_bound} in Proposition~\ref{prop:end_tensor_reaches_ref_sphere}, recalling $\tilde{\epsilon} \leq \hat{\epsilon}$ (by their definition in Equation~\eqref{eq:approx_rank_1_consts}), we obtain:
\be
\normnoflex{ \widebar{\w}_r^n (0)} = \normnoflex{ \w_r^n (T_0)} = \normnoflex{ \tenp_{n' = 1}^N \w_r^{n'} (T_0) }^{ \frac{1}{N} } \leq \tilde{\epsilon}^{ \frac{1}{N} } \leq \hat{\epsilon}^{ \frac{1}{N} } \quad , ~r = 2, \ldots, R ~,~ n = 1, \ldots, N
\text{\,.}
\label{eq:approx_rank_1_t_0_param_vec_bounds_other_compcs}
\ee
As for the first component, for any $n \in [N]$, the fact that $\normnoflex{\widebar{\w}_1^n (0)} = \normnoflex{ \w_1^n (T_0) } =  \normnoflex{ \tenp_{n' = 1}^N \w_1^{n'} (T_0 ) }^{1 / N}$ and Equation~\eqref{eq:approx_rank_1_stage_I_t_0_first_comp_bound} from Proposition~\ref{prop:end_tensor_reaches_ref_sphere} yield the following bound:
\[
\left ( \rho - (R - 1) \cdot \tilde{\epsilon} \right )^{\frac{1}{N}} \leq \norm{\widebar{\w}_1^n (0)} \leq \left ( \rho + (R - 1) \cdot \tilde{\epsilon} \right )^{\frac{1}{N}}
\text{\,.}
\]
On the one hand, since the $\ell_1$ norm is no greater than the $\ell_p$ norm for $p < 1$, we have that $( \rho + (R - 1) \cdot \tilde{\epsilon} )^{1 / N} \leq \rho^{1 / N} + (R - 1)^{1 / N} \cdot \tilde{\epsilon}^{1 / N} \leq \rho^{1 / N} + (R - 1)^{1 / N} \cdot \hat{\epsilon}^{1 / N}$.
On the other hand, since by definition $\tilde{\epsilon} \leq ( R - 1 )^{-1} ( \rho - [ \rho^{1 / N} - (R - 1)^{1 / N} \cdot \hat{\epsilon}^{1 / N} ]^{N} )$, it is straightforward to verify that $( \rho - (R - 1) \cdot \tilde{\epsilon}  )^{1 / N} \geq \rho^{1 / N} - (R - 1)^{1 / N} \cdot \hat{\epsilon}^{1 / N}$.
Put together, while noticing that $\normnoflex{ \widebar{\w}_1^n (0) - \widetilde{\w}_1^n (0) } = \abs{ \normnoflex{ \widebar{\w}_1^n (0) } - \rho^{1 / N} }$, we arrive at:
\be
\norm{ \widebar{\w}_1^n (0) - \widetilde{\w}_1^n (0) } = \abs{ \norm{\widebar{\w}_1^n (0)} - \rho^{ \frac{1}{N} } } \leq (R - 1)^{\frac{1}{N}} \cdot \hat{\epsilon}^{\frac{1}{N}} \quad , ~ n \in [N]
\text{\,.}
\label{eq:approx_rank_1_t_0_param_vec_bounds_first_comp}
\ee
Equations~\eqref{eq:approx_rank_1_t_0_param_vec_bounds_other_compcs} and~\eqref{eq:approx_rank_1_t_0_param_vec_bounds_first_comp} lead to the following bound on the distance between $\{ \widebar{\w}_r^n (0) \}_{r = 1}^R\hspace{0mm}_{n = 1}^N$ and $\{ \widetilde{\w}_r^n (0) \}_{r = 1}^R\hspace{0mm}_{n = 1}^N$:
\[
\begin{split}
	\sum_{r = 1}^R \sum_{n = 1}^N \norm{ \widebar{\w}_r^n (0) - \widetilde{\w}_r^n (0) }^2 & = \sum_{n = 1}^N \norm{ \widebar{\w}_1^n (0) - \widetilde{\w}_1^n (0) }^2  + \sum_{r = 2}^R \sum_{n = 1}^N \norm{ \widebar{\w}_r^n (0)}^2 \\
	& \leq (R - 1)^{\frac{2}{N}} N \cdot \hat{\epsilon}^{\frac{2}{N}} + (R - 1) N \cdot \hat{\epsilon}^{ \frac{2}{N} } \\
	& \leq 2 (R - 1) N \cdot \hat{\epsilon}^{ \frac{2}{N} }
	\text{\,,}
\end{split}
\]
where the last transition is by $(R - 1)^{2 / N} \leq (R - 1)$.
Let $\widetilde{D} := \sqrt{N} \left ( \max \{ D, \rho \} + 1 \right )^{\frac{1}{N}}$ and $\beta := R N ( (\widetilde{D} + 1)^{2(N - 1)} + \delta_h (\widetilde{D} + 1)^{N - 2} )$ (as defined in Equation~\eqref{eq:approx_rank_1_consts}).
According to Lemma~\ref{lem:huber_cp_objective_is_smooth_over_bounded_domain}, the objective $\phi_h (\cdot)$ is $\beta$-smooth over the closed ball of radius $\widetilde{D} + 1$ around the origin.
Furthermore, seeing that $2 (R - 1) N \cdot \hat{\epsilon}^{2 / N} < \exp ( - 2 \beta \cdot T )$ (by the definition of $\hat{\epsilon}$ in Equation~\eqref{eq:approx_rank_1_consts}), we obtain:
\[
\begin{split}
	\sum_{r = 1}^R \sum_{n = 1}^N \norm{ \widebar{\w}_r^n (0) - \widetilde{\w}_r^n (0) }^2 \leq 2 (R - 1) N \cdot \hat{\epsilon}^{ \frac{2}{N} } < \exp ( - 2 \beta \cdot T )
	\text{\,.}
\end{split}
\]
Thus, Lemma~\ref{lem:gf_smooth_dist_bound} implies the following holds at least until $t \geq T$ or $( \sum_{r = 1}^R \sum_{n = 1}^N \normnoflex{ \widetilde{\w} _r^n ( t)}^2 )^{1/2} \geq \widetilde{D}$:
\be
\begin{split}
	\sum_{r = 1}^R \sum_{n = 1}^N \norm{ \widebar{\w}_r^n (t) - \widetilde{\w}_r^n (t) }^2 & \leq \sum_{r = 1}^R \sum_{n = 1}^N \norm{ \widebar{\w}_r^n (0) - \widetilde{\w}_r^n (0) }^2 \cdot \exp \left ( 2 \beta \cdot t \right ) \\
	& \leq 2 (R - 1) N \cdot \hat{\epsilon}^{\frac{2}{N}} \cdot \exp \left ( 2 \beta \cdot t \right )
	\text{\,.}
\end{split}
\label{eq:approx_rank_1_t_0_sq_param_dist_at_end}
\ee

Suppose that $( \sum_{r = 1}^R  \sum_{n = 1}^N \normnoflex{ \widetilde{\w} _r^n ( t )}^2 )^{1/2} < \widetilde{D}$ for all $t \in [0, T]$.
In this case, Equation~\eqref{eq:approx_rank_1_t_0_sq_param_dist_at_end} holds for all $t \in [0, T]$.
Seeing that $2 (R - 1) N \cdot \hat{\epsilon}^{2 / N} \cdot \exp \left ( 2 \beta \cdot T \right ) < 1$, Equation~\eqref{eq:approx_rank_1_t_0_sq_param_dist_at_end} gives $( \sum_{r = 1}^R \sum_{n = 1}^N \normnoflex{ \widebar{\w}_r^n ( t )}^2 )^{1/2} < \widetilde{D} + 1$.
Then, Equation~\eqref{eq:approx_rank_1_t_0_sq_param_dist_at_end}, the fact that $\W_1 (t) = \tenp_{n = 1}^N \widetilde{\w}_1^n (t) = \sum_{r = 1}^R \tenp_{n = 1}^N \widetilde{\w}_r^n (t)$, and Lemma~\ref{lem:param_dist_to_end_to_end_dist} yield:
\[
\norm{ \overline{\W}_e ( t ) - \W_1 (t) } \leq \sqrt{2} R N (\widetilde{D} + 1)^{N - 1} \cdot \exp \left ( \beta \cdot T \right ) \cdot \hat{\epsilon}^{ \frac{1}{N} } \quad , ~ t \in [0, T]
\text{\,.}
\]
Recalling that $\hat{\epsilon} \leq 2^{- \frac{N}{2}} R^{-N} N^{-N} (\widetilde{D} + 1)^{N - N^2} \cdot \exp ( -N \beta T) \cdot \epsilon^N$, we conclude:
\be
\norm{ \overline{\W}_e ( t ) - \W_1 (t) } \leq \epsilon
\text{\,,}
\label{eq:end_tensor_rank_1_epsilon_dist}
\ee
for all $t \in [0, T]$.

It remains to treat the case where $( \sum_{r = 1}^R  \sum_{n = 1}^N \normnoflex{ \widetilde{\w} _r^n ( t )}^2 )^{1/2} \geq \widetilde{D}$ for some $t \in [0, T]$.
Let $t' \in [ 0 , T ]$ be the initial such time (well defined due to continuity of $( \sum_{r = 1}^R  \sum_{n = 1}^N \normnoflex{ \widetilde{\w} _r^n ( t )}^2 )^{1/2}$ with respect to $t$).
The desired result readily follows by showing that: \emph{(i)} Equation~\eqref{eq:end_tensor_rank_1_epsilon_dist} holds for $t \in [0, t']$; and \emph{(ii)} $\normnoflex{ \overline{\W}_e ( t' )} \geq D$.

We start by proving that $\norm{ \W_1 ( t' ) } \geq \max \{ D, \rho \} + 1$ and $t' > 0$.
Recalling $\widetilde{\w}_r^1 (t), \ldots, \widetilde{\w}_r^N (t)$ are identically zero for all $r \in \{ 2, \ldots, R \}$, we have that:
\[
\sum_{n = 1}^N \normnoflex{ \widetilde{\w} _1^n ( t' )}^2  =  \sum_{r = 1}^R  \sum_{n = 1}^N \normnoflex{ \widetilde{\w} _r^n ( t' )}^2 \geq \widetilde{D}^2
\text{\,.}
\]
Since $\normnoflex{ \widetilde{\w}_1^1 ( 0 ) } = \cdots = \normnoflex{ \widetilde{\w}_1^N ( 0 ) }$, Lemma~\ref{lem:balancedness_conservation_body} implies $\normnoflex{ \widetilde{\w}_1^1 ( t' ) } = \cdots = \normnoflex{ \widetilde{\w}_1^N ( t' ) }$.
Thus, for any $n \in [N]$:
\[
N \normnoflex{ \widetilde{\w}_{1}^{n} (t') }^2 = \sum_{n' = 1}^N \normnoflex{ \widetilde{\w}_{1}^{n'} ( t' ) }^2 \geq \widetilde{D}^2
\text{\,,}
\]
which leads to $\normnoflex{ \widetilde{\w}_{1}^{n} ( t' ) } \geq \widetilde{D} N^{-1 / 2}$.
In turn this yields $\normnoflex{ \W_1 (t') } = \norm{ \tenp_{n = 1}^N \widetilde{\w}_{1}^{n} ( t' ) } = \prod_{n = 1}^N \norm{ \widetilde{\w}_{1}^{n} ( t' ) } \geq \widetilde{D}^N N^{- \frac{N}{2}}$.
Plugging in $\widetilde{D} := \sqrt{N} (\max \{ D, \rho \}  + 1)^{\frac{1}{N}}$, we conclude: 
\be
\normnoflex{ \W_1 ( t' ) } \geq \max \{ D, \rho \} + 1
\text{\,.}
\label{eq:rank_1_traj_norm_lower_bound}
\ee
Note that this necessarily means $t' > 0$ as $\W_1 (0) \in \S$, \ie~$\normnoflex{ \W_1 ( 0 ) } = \rho < \max \{ D, \rho \} + 1$.

Now, we focus on the time interval $[0, t')$, over which Equation~\eqref{eq:approx_rank_1_t_0_sq_param_dist_at_end} holds and $( \sum_{r = 1}^R  \sum_{n = 1}^N \normnoflex{ \widetilde{\w} _r^n ( t )}^2 )^{1/2} < \widetilde{D}$.
From the same reasoning as in the case where $( \sum_{r = 1}^R  \sum_{n = 1}^N \normnoflex{ \widetilde{\w} _r^n ( t )}^2 )^{1/2} < \widetilde{D}$ for all $t \in [0, T]$, we obtain that Equation~\eqref{eq:end_tensor_rank_1_epsilon_dist} holds for all $t \in [0, t')$.
Continuity with respect to time then implies $\normnoflex{ \overline{\W}_e ( t' ) - \W_1 (t') } \leq \epsilon < 1$.
Lastly, together with Equation~\eqref{eq:rank_1_traj_norm_lower_bound} this leads to $\normnoflex{ \overline{\W}_e ( t' )} \geq \normnoflex{ \W_1 (t') } - 1 \geq D$.

Overall, we have shown that $\normnoflex{\overline{\W}_e ( t ) - \W_1 (t) } \leq \epsilon$ at least until time $T$ or time $t'$ at which $\normnoflex{\overline{\W}_e ( t' )} \geq D$, establishing the desired result.
\qed

\subsection{Proof of Corollary~\ref{corollary:converge_rank_1}}
\label{app:proofs:converge_rank_1}

For $\epsilon > 0$, there exists a time $T' > 0$ at which all balanced rank one trajectories emanating from $\S$ are within distance $\epsilon / 2$ from $\W^*$.
Moreover, these trajectories are confined to a ball of radius $D$ around the origin, for some $D > 0$.
According to Theorem~\ref{thm:approx_rank_1}, if initialization scale $\alpha$ is sufficiently small, $\normnoflex{ \overline{ \W }_e (t) - \W_1 (t) } \leq \min \{ \epsilon / 2, 1 / 2 \}$ at least until $t \geq T'$ or $\normnoflex{ \overline{ \W }_e (t) } \geq D + 1$, where $\overline{ \W }_e (t)$ is the time-shifted trajectory of $\W_e (t)$, and $ \W_1 (t)$ is a balanced rank one trajectory emanating from $\S$.
We claim that the latter cannot hold, \ie~$\normnoflex{ \overline{ \W }_e (t) } < D + 1$ for all $t \in [0, T']$.
To see it is so, assume by way of contradiction otherwise, and let $t' \in [0, T']$ be the initial time at which $\normnoflex{ \overline{ \W }_e (t') } \geq D + 1$.
Since $\normnoflex{ \overline{ \W }_e (t') - \W_1 (t') } < 1$, we have that $\normnoflex{  \W_1 (t') } > D$, in contradiction to $\W_1 (t)$ being confined to a ball of radius $D$ around the origin.
Thus, $\normnoflex{ \overline{ \W }_e ( T' ) - \W_1 (T') } \leq \epsilon / 2$.
The proof concludes by the triangle inequality:
\[
\norm{ \overline{ \W }_e (T') - \W^* } \leq \norm{ \overline{ \W }_e (T') - \W_1 (T') } + \norm{ \W_1 (T') - \W^* } \leq \epsilon
\text{\,.}
\]
\qed